\documentclass[12pt]{article}
\usepackage{comment}
\usepackage{enumerate}
\usepackage{url} 

\usepackage{authblk}

\def\removenotes{1}

\usepackage{jcd}
\usepackage[numbers]{natbib} 
\usepackage{authblk} 
\usepackage{xspace}
\usepackage{overpic}
\usepackage{tikz}
\usepackage{algorithm,algorithmic}

\bluehyperref

\newcommand{\VC}{\textup{VC}}
\newcommand{\WCQuantile}{\textup{Quantile}^\wc}
\newcommand{\Quantile}{\textup{Quantile}}

\newcommand{\cdfrob}{F^\wc}

\newcommand{\score}{s}  
\newcommand{\scorerv}{S}  
\newcommand{\scoreest}{\score_n}  
\newcommand{\wcoverage}{\textup{WC}}
\newcommand{\wc}{\wcoverage}
\newcommand{\stocuo}{\succeq_{\textup{uo}}}

\newcommand{\eqdist}{\stackrel{\textup{dist}}{=}}

\newcommand{\levelv}{\alpha_{\mathsf{v}}} 
\newcommand{\Pv}{\P_{\mathsf{v}}} 
\newcommand{\Pvemp}{\hat{\P}_{\mathsf{v},k}} 

\newcommand{\empP}{\hat{P}_n}  
\newcommand{\empQ}{\hat{Q}_n}  

\newcommand{\vvar}{v_{\textup{var}}}  
\newcommand{\empPone}{\hat{P}_{n_1}}  
\newcommand{\empPtwo}{\hat{P}_{n_2}}
\newcommand{\empQone}{\hat{Q}_{n_1}}
\newcommand{\empQtwo}{\hat{Q}_{n_2}}
\newcommand{\law}{\mc{L}}

\newcommand{\ltwopx}[1]{\norm{#1}_{L^2(Q_0)}}
\newcommand{\ltwopxs}[1]{\norm{#1}_{L^2(Q_0)}}

\newcommand{\train}{\textrm{train}}
\newcommand{\val}{\textrm{val}}
\newcommand{\test}{\textrm{test}}
\newcommand{\Ptest}{P_{\textup{test}}}
\providecommand{\toto}{\rightrightarrows}

\def\eg{e.g.}

\newcommand{\noise}{\varepsilon}

\DeclareMathOperator*{\esssup}{ess\,sup}
\DeclareMathOperator*{\SF}{SF}
\DeclareMathOperator*{\MC}{M}

\graphicspath{{figs/}}

\newcommand{\blind}{1}

\begin{document}

\def\spacingset#1{\renewcommand{\baselinestretch}%
{#1}\small\normalsize} \spacingset{1}


\if1\blind
{
  \title{\bf Robust Validation: Confident Predictions \\
  Even When Distributions Shift\thanks{
    Research supported by the NSF under CAREER Award CCF-1553086
    and HDR 1934578 (the Stanford Data Science Collaboratory),
    Office of Naval Research YIP Award N00014-19-2288,
    and the Stanford DAWN Consortium.}}
  \author{
  	Maxime Cauchois \\Department of Statistics, Stanford University
    \and Suyash Gupta \\
    Department of Statistics, Stanford University
    \and Alnur Ali \\
    Department of Statistics and Electrical Engineering, Stanford University
    \and  John C. Duchi \\
    Department of Statistics and Electrical Engineering, Stanford University
    }
  \maketitle
} \fi

\if0\blind
{
  \bigskip
  \bigskip
  \bigskip
  \begin{center}
    {\Large \bf Robust Validation: Confident Predictions \\
  Even When Distributions Shift}
\end{center}
  \medskip
} \fi

\bigskip

\begin{abstract}
  While the traditional viewpoint in machine learning and statistics assumes
  training and testing samples come from the same population, practice
  belies this fiction. One strategy---coming from robust statistics and
  optimization---is thus to build a model robust to distributional
  perturbations. In this paper, we take a different approach to describe
  procedures for robust predictive inference, where a model provides
  uncertainty estimates on its predictions rather than point predictions. We
  present a method that produces prediction sets (almost exactly) giving the
  right coverage level for any test distribution in an $f$-divergence ball
  around the training population. The method, based on conformal inference,
  achieves (nearly) valid coverage in finite samples, under only the
  condition that the training data be exchangeable. An essential component
  of our methodology is to estimate the amount of expected future data shift
  and build robustness to it; we develop estimators and prove their
  consistency for protection and validity of uncertainty estimates under
  shifts.  By experimenting on several large-scale benchmark datasets,
  including Recht et al.'s CIFAR-v4 and ImageNet-V2 datasets, we provide
  complementary empirical results that highlight the importance of robust
  predictive validity.
\end{abstract}

\noindent%
{\it Keywords:}  Conformal inference, Confidence sets, Coverage validity, $f$-divergences, Robust statistics
\vfill

\newpage
\spacingset{1.9} 


\section{Introduction}
\label{sec:intro}

The central conceit of statistical machine learning is that data comes from a
population, and that a model fit on a training set and validated on a held-out
validation set will generalize to future data.  Yet this conceit is at
best debatable: indeed, \citet*{RechtRoScSh19} create new test sets
for the central image recognition CIFAR-10 and ImageNet benchmarks, and
they observe
that published accuracies drop by between 3--15\% on CIFAR and more than 11\% on
ImageNet (increases in error rate of 50--100\%), even though the authors follow
the original dataset creation processes. Given this drop in accuracy---even in
carefully reproduced experiments---shift in the data generating distribution is
inevitable, and should be an essential focus, given the growing applications of
machine learning.

To address such distribution shifts and related challenges, a growing
literature advocates fitting predictive models that adapt to changes in the
data generating distribution. For example, researchers suggest reweighting
data to match new test distributions when covariates
shift~\cite{SugiyamaKrMu07}, while work
on distributional robustness~\cite{BertsimasGuKa18,
  DuchiNa21} considers fitting models
that optimize losses under worst-case distribution changes.  Yet the
resulting models often are conservative, appear to sacrifice
accuracy for robustness, and even more, they may not be robust to
natural distribution shifts~\cite{TaoriDaShCaReSc19}. The models also come
with few tools for validating their performance on new data.

Instead of seeking robust models, we instead advocate focusing on models
that provide \emph{validity} in their predictions: a model
should be able to provide some calibrated notion of its confidence, even in
the face of distribution shift. Consequently, in this paper we revisit cross
validation, validity, and conformal inference~\cite{VovkGaSh05} from the
perspective of robustness, advocating for more robust approaches to cross
validation and equipping predictors with valid confidence sets. We present a
method for robust predictive inference under distributional shifts,
borrowing tools both from conformal inference~\cite{VovkGaSh05} and
distributional robustness. Our method can allow valid inferences even when
training and test distributions are distinct, and we provide a (in our view
well-motivated, but still heuristic) methodology to estimate plausible
amounts of shift to which we should be robust.

To formalize, consider a supervised learning problem of predicting
labels $y \in \mc{Y}$ from data $x \in \mc{X}$, where we assume we have a
putative predictive model that outputs scores $\score(x, y)$ measuring
error (so that $\score(x, y) < \score(x, y')$ means that the
model assigns higher likelihood to $y$ than $y'$ given $x$). For example,
for a probabilistic model $p(y \mid x)$, a typical choice is the negative
log likelihood $\score(x, y) = -\log(p(y \mid x))$. For a distribution $Q_0$
on $\mc{X} \times \mc{Y}$,\footnote{We always write
  $Q$ for a probability on $\mc{X} \times \mc{Y}$ and $P$ for the induced
  distribution on $\score(X, Y)$ for $(X, Y) \sim Q$.}  we observe $(X_i,
Y_i)_{i = 1}^n \simiid Q_0$. Future data may come from $Q_0$ or a
distribution $Q$ near---in some appropriate sense, deriving from distribution
shift---to $Q_0$, and we wish to
output valid predictions for future instances $(X, Y) \sim Q$, where $Q$ is
unknown.
The goal of this paper is twofold: first, given a level $\alpha
\in (0, 1)$ and an uncertainty set
$\mc{Q}$ of plausible shifted distributions, we wish to construct
\emph{uniformly valid} confidence set mappings
$\what{C} : \mc{X} \toto \mc{Y}$ of the form
$\what{C}(x) = \{y \in \mc{Y} \mid \score(x, y) \le q\}$ for
a threshold $q$, which provide $1 - \alpha$ coverage, satisfying
\begin{equation}
  \label{eqn:uniform-coverage}
  Q(Y \in \what{C}(X)) \ge 1 - \alpha
  ~~ \mbox{for~all~} Q \in \mc{Q}.
\end{equation}
Second, we propose a methodology for finding a collection $\mc{Q}$ of
plausible shifts, providing convergence theory and a
concomitant empirical validiation on real distribution shift problems. Further, we propose methodology to study sensitivity of coverage under various covariate shifts. This helps the user identify the type of shifts, the coverage is sensitive to as protecting against all possible shifts may lead to very conservative predictive sets.


\subsection{Background: split conformal
  inference under exchangeability}
\label{sec:split-conformal-intro}

To set the stage, we review conformal predictive
inference~\cite{VovkGaSh05, LeiGSRiTiWa18}.
The setting here is a supervised learning problem where we have exchangeable
data $\{ (X_i, Y_i) \}_{i = 1}^{n+1} \subset \mc{X} \times \mc{Y}$, and for
a given confidence level $1-\alpha \in (0, 1)$ we wish to provide a confidence
set $\what{C}(X_{n+1})$ such that $\P(Y_{n + 1} \in \what{C}(X_{n+1})) \ge 1 -
\alpha$. Standard properties of quantiles make such a construction
possible. Indeed, assume that $S_1, \ldots, S_{n+1} \in \R$ are exchangeable
random variables; then, the rank $\rank(S_j)$ of any $S_j$ among
$\{S_i\}_{i=1}^{n+1}$---its position if we sort the values of the $S_i$---is
evidently uniform on $\{1, \ldots, n+1\}$, assuming ties are broken
randomly.  Thus, for probability distributions $P$ on $\R$, defining the
familiar quantile
\begin{equation}
  \Quantile(\beta; P)
  \defeq \inf \big\{ s \in \R : P(S \leq s) \geq \beta \big\},
  \label{eqn:quantile-defn}
\end{equation}
and $\Quantile(\beta; \{S_i\}_{i=1}^n)$ to be the corresponding
empirical quantile on $\{S_i\}_{i=1}^n$, we have
\begin{equation*}
\P\left(S_{n+1} \le \Quantile\left(\left(1 + n^{-1}\right)
(1 - \alpha), \{S_i\}_{i=1}^n\right)\right)  \ge \P( \rank(S_{n+1}) \le  \ceil{(n+1)(1-\alpha)}) \ge 1 - \alpha.
\end{equation*}

Using this idea to provide confidence sets is now
straightforward~\cite{VovkGaSh05, LeiGSRiTiWa18}.  Let $\{ (X_i, Y_i)
\}_{i=1}^n$ be a validation set---we assume here and throughout that we have
already fit a model on training data independent of the validation set
$\{(X_i, Y_i)\}_{i=1}^n$---and assume we have a scoring function $\score :
\mc{X} \times \mc{Y} \to \R$, where a large value of $\score(x, y)$
indicates that the point $(x, y)$ is \emph{non-conforming}. In typical
supervised learning tasks, such a function is easy to construct. Indeed,
assume we have a predictor function $\mu$ (fit on an independent training
set); in the case of regression, $\mu : \mc{X} \to \R$ predicts $\E[Y \mid
  X]$, while for a multiclass classification problem $\mu : \mc{X} \to
\R^k$, and $\mu_y(x)$ is large when the model predicts class $y$ to be
likely given $x$. Then natural nonconformity scores are $\score(x, y) =
|\mu(x) - y|$ for regression and $\score(x, y) = -\mu_y(x)$ for
classification. As long as $\{(X_i, Y_i)\}_{i=1}^{n+1}$ are
exchangeable, if we define $\what{\mc{Q}}_{n,1-\alpha} \defeq
\Quantile\left((1 + n^{-1})(1- \alpha); \{\score(X_i, Y_i)\}_{i=1}^n\right)$,
the confidence set
\begin{equation}
  \label{eqn:confidence-set}
  \what{C}_n(x) \defeq \left\{y \in \mc{Y} \mid \score(x, y) \le
 \what{\mc{Q}}_{n,1-\alpha}
  \right\},
\end{equation}
immediately satisfies
\begin{equation}
 \P(Y_{n + 1} \in \what{C}_n(X_{n+1}))
    = \P\left(\score(X_{n+1}, Y_{n+1}) \le
   \what{\mc{Q}}_{n,1-\alpha}
    \right) \ge 1 - \alpha,
 \label{eqn:exchangeable-coverage}
\end{equation}
whatever the scoring function $\score$ and distribution on $(X_i,
Y_i)$~\cite{VovkGaSh05, LeiGSRiTiWa18}.  The coverage
statement~\eqref{eqn:exchangeable-coverage} depends critically (as we shall
see) on the exchangeability of the samples, failing if even the marginal
distribution over $X$ changes, and it does not imply conditional coverage:
we have no guarantee that $\P(Y \in \what{C}(X) \mid X) \ge 1 - \alpha$.

\subsection{Related work}
\label{sec:related}

The machine learning community has long identified distribution shift as a
challenge, with domain adaptation strategies and covariate shift two
major foci~\cite{SugiyamaKrMu07, Quionero-CandelaSuSc09}, though much of this work
focuses on model estimation and selection strategies, and one often assumes
access to data (or at least likelihood ratios) of data from the new
distribution. We argue that a model should instead provide robust and valid
estimates of its confidence rather than simply predictions that may or may
not be robust. There is a growing body of work on distributionally
robust optimization (DRO), which considers worst-case dataset shifts in
neighborhoods of the training distribution; these have been
important in finance and operations research, where one wishes to guard
against catastrophic losses~\cite{RockafellarUr00,
  BertsimasGuKa18}. In DRO in statistical
learning~\cite{BlanchetMu19, DuchiNa21}, the focus has also been on improving
estimators rather than inferential predictive tasks.  We extend
this distributional robustness to apply in predictive
inference.

\citet{VovkGaSh05}'s conformal inference provides an important tool for
valid predictions.  The growing applications of machine learning and
predictive analytics have renewed interest in predictive validity, and
recent papers attempt to move beyond the standard exchangeability
assumptions upon which conformalization reposes~\cite{TibshiraniBaCaRa19,
  ChernozhukovWuZh18, CauchoisGuAlDu22, DuchiGuJiSu24, Gupta22}, though this typically requires some additional
assumptions for strict validity.  Of particular relevance to our setting is
\citeauthor{TibshiraniBaCaRa19}'s work~\cite{TibshiraniBaCaRa19}, which
considers conformal inference under covariate shift, where the marginal over
$X$ changes while $P(Y \mid X)$ remains fixed. Validity in this setting
requires knowing a likelihood ratio of the shift, which in high dimensions
is challenging.  In addition, as \citet{Jordan19} argues, in typical
practice covariate shifts are no more plausible than other (more general)
shifts, especially in situations with unobserved confounders. For this
reason, we take a more general approach and do not restrict to specific
structured shifts.

In the existing literature on sensitivity analysis in causal inference \citep{Imbens03,VeitchZa20,HsuSm13}, researchers use the sensitivity parameter to gauge the influence of unobserved confounders on treatment allocation and outcomes. One essence is that the odds of receiving treatment, considering both observed covariates and the confounder U, can differ by a factor of some constant $\Gamma$ when juxtaposed against odds based solely on observed covariates, with a value near 1 indicating minimal influence. Mirroring this, we employ f-divergence, especially the expected log-likelihood ratio in KL divergence offset by a factor $\rho$, to understand distribution shifts between training and test distributions, comparable to the odds ratio in causal inference. Our study in Section 4 assesses the intensity of such shifts and hints at calibrating $\rho$, reminiscent of using observed covariates to adjust $\Gamma$ in causal inference.

\subsection{A few motivating examples}
\label{sec:motivation-exp}

Standard validation methodology randomly splits data into
train/validation/test sets, artificially enforcing exchangeabilty).  Thus,
to motivate the challenges in predictive validity even
under simple covariate shifts---we only modify the
distribution of $X$, returning later to more sophisticated real-world
scenarios---we experiment on nine regression datasets from the UCI
repository~\cite{DuaGr17}. We repeat the following 50 times.  We randomly
partition each dataset into disjoint sets $D_\train, D_\val, D_\test$,
each consisting of $1/3$ of the data. We fit a random forest predictor $\mu$
using $D_\train$ and construct conformal intervals of the
form~\eqref{eqn:confidence-set} with $\score(x, y) = |\mu(x)- y|$, so that
$\what{C}_n(x) = \{y \mid |\mu(x) - y| \le \hat{t}\}$ for a threshold
$\hat{t}$ achieving coverage at nominal level $\alpha = .05$ on $D_\val$, as
is standard in split-conformal prediction~\cite{VovkGaSh05}.
We evaluate coverage on tiltings of varying strength on $D_\test$: letting
$v$ be the top eigenvector of the test $X$-covariance $\Sigma_\test$
and $\wb{x}_\test$ be the mean of $X$ over $D_\test$, we reweight $D_\test$
by probabilities proportional to $w(x) = \exp(a v^T (x -
\wb{x}_\test))$ for tilting parameters $a \in \pm \{0, .02, .04, .08, .16,
.32, .64\}$. Essentially, this shift asks the following question: why
would we \emph{not} expect a shift along the principal directions of
variation in $X$ on future data?

\begin{figure}[h!]
  \centering
  \includegraphics[width=0.3\linewidth]{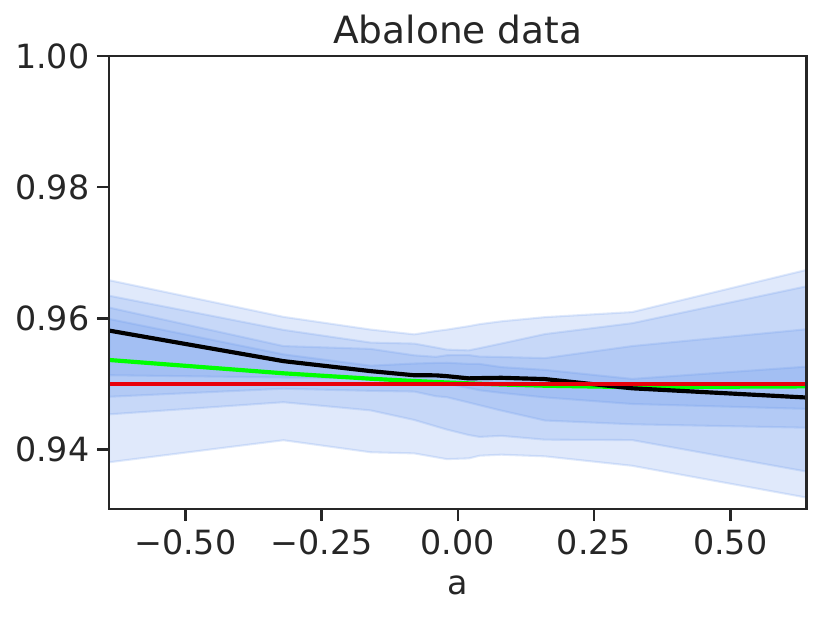}
  \hfill
  \includegraphics[width=0.3\linewidth]{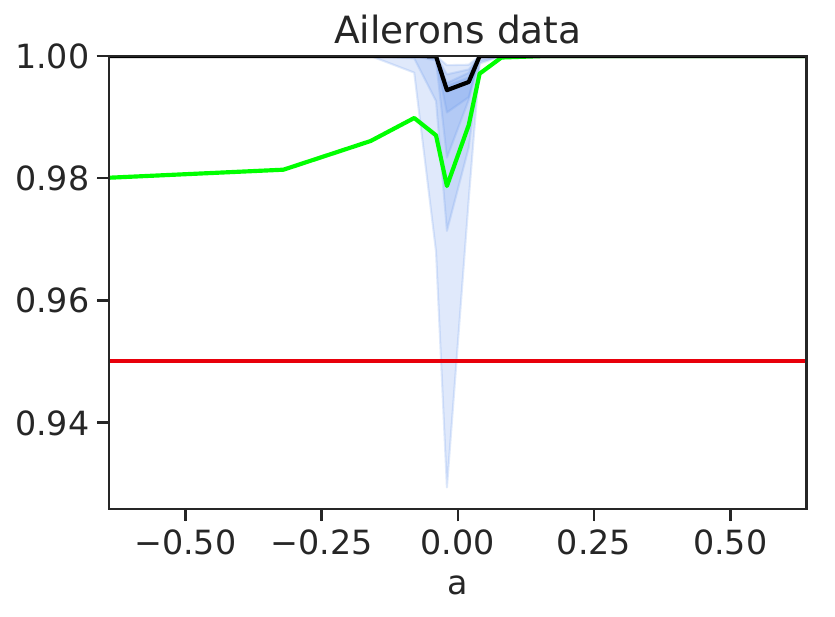}
  \hfill
  \includegraphics[width=0.3\linewidth]{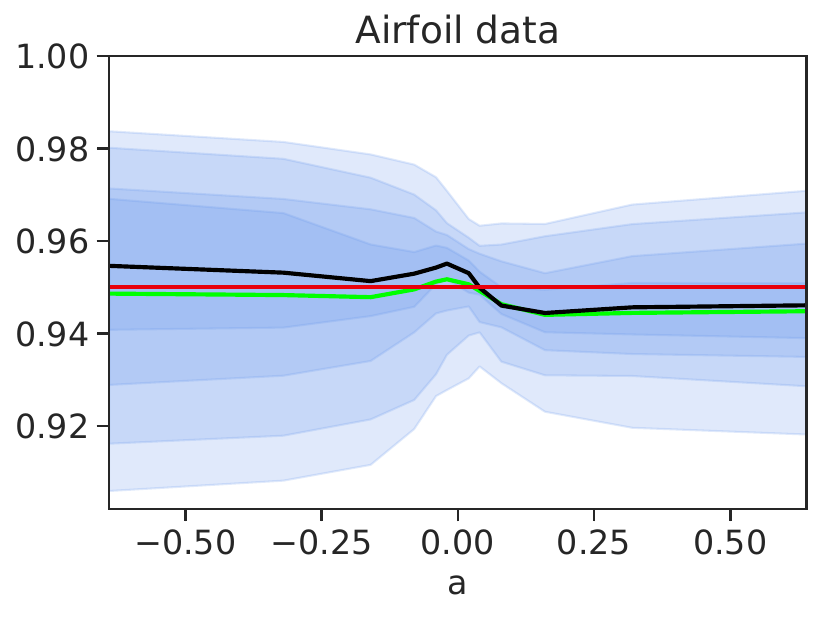}
  \\
  \includegraphics[width=0.3\linewidth]{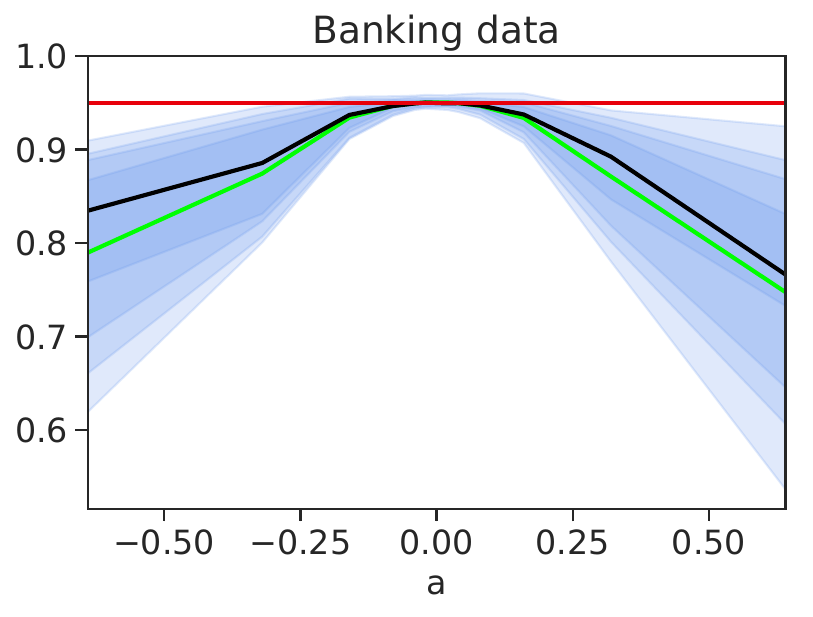}
  \hfill
  \includegraphics[width=0.3\linewidth]{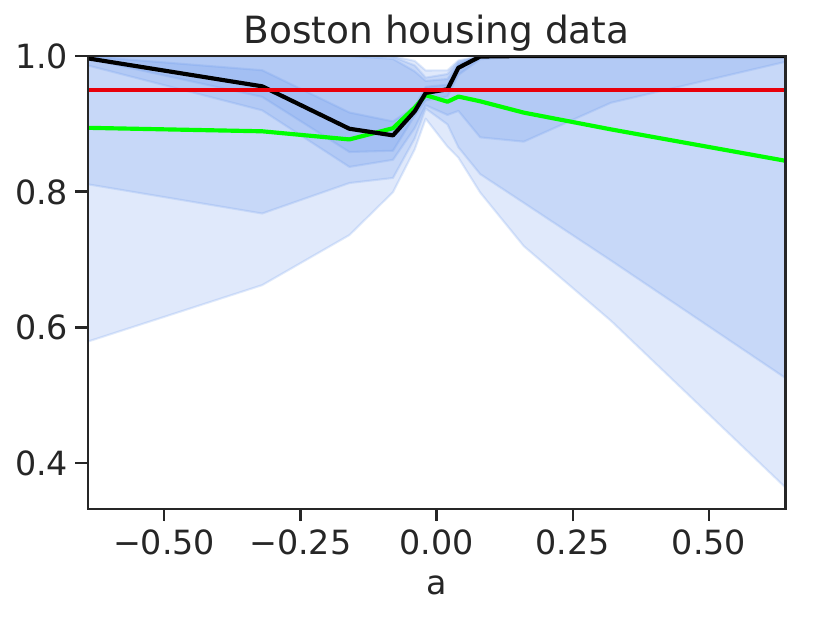}
  \hfill
  \includegraphics[width=0.3\linewidth]{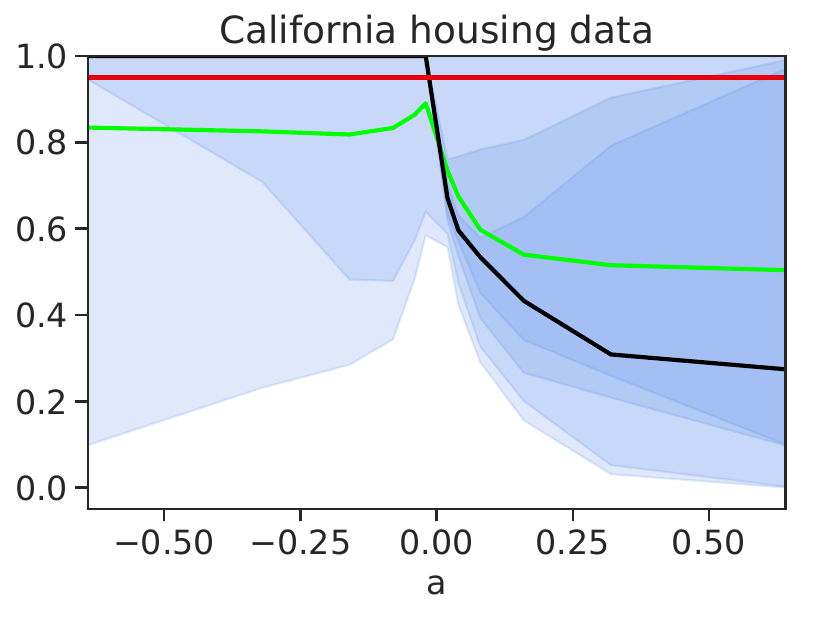}
  \\
  \includegraphics[width=0.3\linewidth]{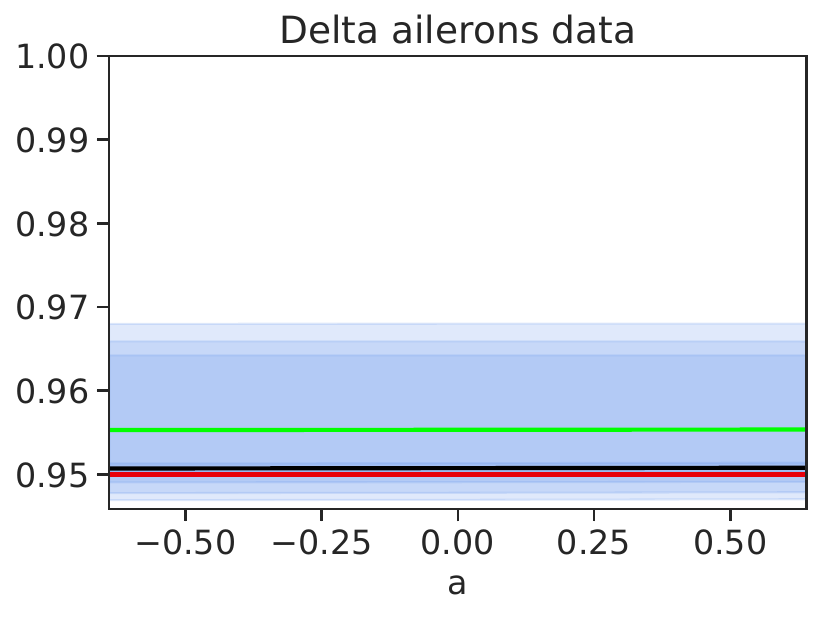}
  \hfill
  \includegraphics[width=0.3\linewidth]{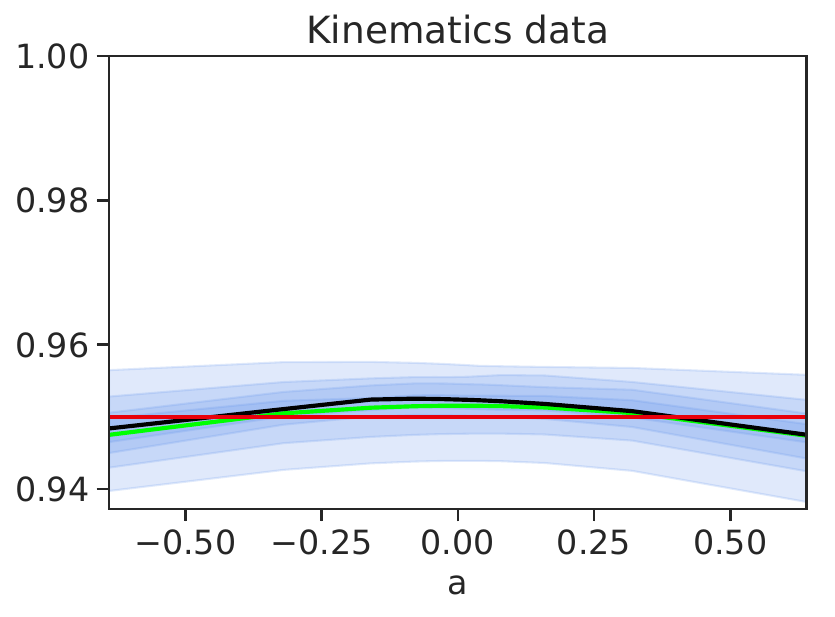}
  \hfill
  \includegraphics[width=0.3\linewidth]{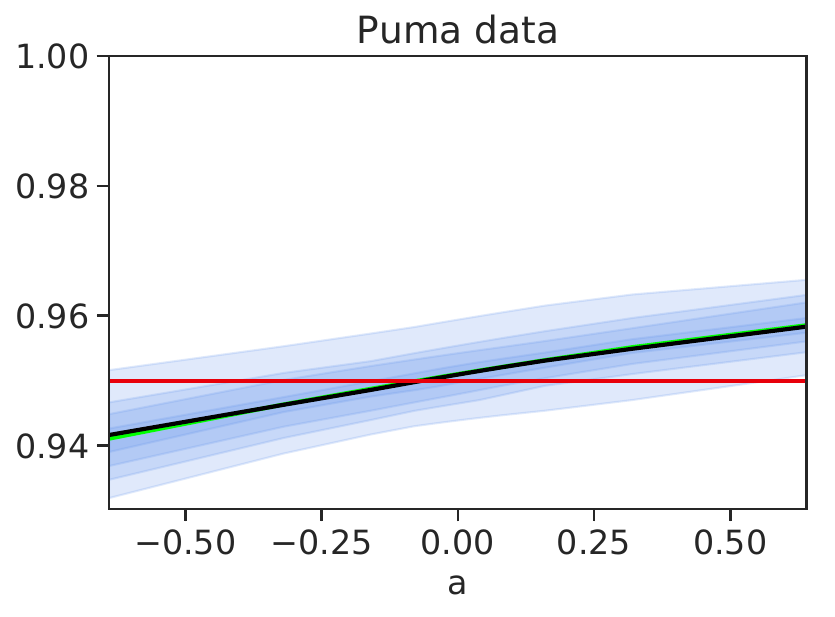}  
  \caption{Empirical coverage for the prediction sets generated by the standard
    conformal methodology across nine regression data sets and 50
    random splits of each data set, with an exponential tilting in $X$ space
    along the first principal component of $X$. The horizontal axis gives the
    value of the tilting parameter $a$; the vertical the coverage level. A green
    line marks the average coverage, a black line marks the median coverage, and
    the horizontal red line marks the nominal coverage $.95$. The blue bands
    show the coverage at deciles over 50 splits.}
  \label{fig:cvgs_only_std}
\end{figure}

Figure~\ref{fig:cvgs_only_std} presents the results: even when the covariate
shifts are small, which corresponds to tilting parameters $a$ with small
magnitude, prediction intervals from the standard conformal methodology
frequently fail to cover (sometimes grossly) the true response values. While
this is but a simple motivation, if we expect some
shift in future data---say along the directions of principal variation in
$X$, as the data itself is already variable along that axis---it seems that
standard validation approaches~\cite{HastieTiFr09} provide too rosy
of a picture of future validity~\cite{RechtRoScSh19}, as they \emph{enforce}
exchangeability by randomly splitting data.


\section{Robust predictive inference}
\label{sec:robustpredinference}

Of course, standard cross validation and conformalization methodology makes
no claims of validity without exchangeability~\cite{VovkGaSh05,
  BarberCaRaTi19a}, so their potential failure even under simple
covariate shifts is not completely surprising.  The
coverage~\eqref{eqn:exchangeable-coverage} relies on the exchangeability
assumption between the training and test data and can quickly
collapse when the test distribution violates that assumption, as
Section~\ref{sec:motivation-exp} shows.  These observations thus call for a
notion of confidence more robust to potential future shifts.


Assume as usual that we have a score function $\score: \mc{X} \times \mc{Y}
\to \R$, and observe data $\{(X_i, Y_i) \}_{i=1}^n$ such that $\{
\scorerv_i\}_{i=1}^n \defeq \{ \score(X_i, Y_i) \}_{i=1}^n \simiid P_0$,
so that $P_0$ is the push-forward of $(X, Y) \sim Q_0$ under $\score(X, Y)$.
For a set $\mc{P}(P_0)$ of potential future \emph{score} distributions on
$\R$, our goal is to achieve
coverage~\eqref{eqn:uniform-coverage} for all distributions
$Q$ on pairs $(X,Y)$ that induce a distribution $P$ on $\score(X, Y)$
such that $P \in \mc{P}(P_0)$, that is,
\begin{align*}
  Q \in \mc{Q}(\score, \mc{P}(P_0)) \defeq \left\{ Q
  ~ \mbox{s.t.~for}~
  (X,Y) \sim Q, \mbox{the score}~ \score(X,Y) \sim P \in \mc{P}(P_0) \right\}. 
\end{align*}
Our focus is exclusively on validating our predictive model, not changing
it, so we follow standard practice~\cite{VovkGaSh05,
  BarberCaRaTi19a} and use confidence sets $\what{C}(x)$ to be of the form
$\what{C}(x) = \{y \in \mc{Y} \mid \score(x,y) \le t \}$ for a threshold $t \in
\R$.  
For such confidence sets, the choice $t \defeq \max_{P \in
  \mc{P}(P_0)} \Quantile(1-\alpha,P) $ is the smallest $q \in \R$ such that
$P(\scorerv \le q) \ge 1-\alpha$ for every distribution $P \in \mc{P}(P_0)$
of the scores.  Our general problem to achieve
coverage~\eqref{eqn:uniform-coverage} with uncertainty set $\mc{Q}(\score,
\mc{P}(P_0))$ thus reduces to the optimization problem
\begin{equation}
\label{eqn:robust-quantile-prob}
\maximize ~ \Quantile( 1-\alpha; \, P ) ~~~
\subjectto ~ P \in \mc{P}(P_0).
\end{equation}
In the next section, we characterize solutions to this problem,
showing in Section~\ref{sec:get-coverage-for-divs} how
to use the characterizations to achieve coverage on future data.

\subsection{Characterizing and computing
  quantiles over $f$-divergence balls}
\label{sec:robust-coverage-div-balls}

It remains to specify a set of distributions $\mc{P}(P_0)$ that makes
problem~\eqref{eqn:robust-quantile-prob} computationally tractable and
statistically meaningful. We thus consider various restrictions on the
likelihood ratio $dP / dP_0$ for $P \in \mc{P}(P_0)$. Following the
distributionally robust optimization literature
(DRO)~\cite{ BlanchetMu19,
  DuchiNa21}, we consider $f$-divergence
balls. Given a closed convex function $f: \R \to \R$
satisfying $f(1)=0$ and $f(t) = +\infty$ for $t<0$,
the $f$-divergence~\cite{Csiszar67} between probability distributions
$P$ and $Q$ on a set $\mc{Z}$ is
\begin{equation*}
  \fdiv{P}{Q} \defeq \int_{z \in \mc{Z}} f\left( \frac{dP(z)}{dQ(z)}
  \right) dQ(z).
\end{equation*}
Jensen's inequality guarantees that $\fdiv{P}{Q} \ge 0$ always, and
familiar examples include $f(z) = z \log z$, which induces the KL-divergence,
and $f(t) = \half (t - 1)^2$, which gives the $\chi^2$-divergence.
We study problem~\eqref{eqn:robust-quantile-prob} in the case where
$\mc{P}(P_0)$ is an $f$-divergence ball
of radius $\rho$ around $P_0$:
\begin{align}
  \label{eqn:P-constraint-set}
  \mc{P}_{f,\rho}(P_0) \defeq \left\{ P ~\mbox{s.t.}~
  \fdiv{P}{P_0}  \le \rho \right\}.
\end{align}
Unlike most work in the DRO literature, instead of trying to build a
model minimizing a DRO-type loss, we assume we already have a model and wish
to robustly validate it: to
provide predictive confidence sets that are valid and robust to distribution
shifts no matter the model's form.
  By the data processing inequality, all distributions $Q$ on $(X,Y)$
  satisfying $\fdivs{Q}{Q_0} \le \rho$ induce a
  distribution $P$ on $\score(X, Y)$ satisfying $\fdivs{P}{P_0} \le \rho$,
  so solving
  problem~\eqref{eqn:robust-quantile-prob} with $\mc{P}_{f,\rho}(P_0)$
  provides coverage for all sufficiently small shifts on
  $(X, Y) \sim Q_0$.

We show how to solve problem~\eqref{eqn:robust-quantile-prob} for fixed
$f$ and $\rho$ defining the constraint~\eqref{eqn:P-constraint-set} by
characterizing worst-case quantiles, essentially reducing the problem to
a one-parameter (Bernoulli) problem. The choice of $f$ and $\rho$ determine
plausible amounts of shift---appropriate choices are a longstanding
problem~\cite{DuchiNa21}---and we defer approaches for selecting them to the
sequel.
For $\alpha \in (0, 1)$ and any distribution $P$ on the real line,
we define the $(\alpha, \rho, f)$-worst-case quantile
\begin{equation}
  \label{eqn:wc-quantile}
  \WCQuantile_{f,\rho}(\alpha; P) \defeq 
  \sup_{\fdivs{P_1}{P} \le \rho} \Quantile(\alpha; \, P_1).
\end{equation}
Key to our results on valid coverage in
Section~\ref{sec:get-coverage-for-divs} is that this worst-case quantile is
a standard quantile of $P$ at a  level that
depends only on $f, \rho$, and $\alpha$, but not on $P$.
\begin{proposition} 
  \label{proposition:rho-vs-alpha}
  Define the function $g_{f, \rho} : [0, 1] \to [0, 1]$ by
  \begin{equation*}
    g_{f,\rho}(\beta) \defeq \inf\left\{z \in [0, 1]
    : \beta f\left(\frac{z}{\beta}\right) + (1 - \beta) f\left(\frac{1 - z}{
      1 - \beta}\right) \le \rho \right\}.
  \end{equation*}
  Then the inverse
  \begin{equation*}
    g_{f,\rho}^{-1}(\tau) = \sup\{ \beta \in [0, 1] : g_{f,\rho}(\beta) \le \tau\}
  \end{equation*}
  guarantees that for all distributions $P$ on $\R$ and $\alpha \in (0, 1)$,
  \begin{align*}
    \WCQuantile_{f,\rho}(\alpha; P) =
    \Quantile (g_{f,\rho}^{-1}(\alpha); P ).
  \end{align*}
\end{proposition}
\noindent
See Appendix~\ref{sec:proof-rho-vs-alpha} for a proof of the proposition.

Proposition~\ref{proposition:rho-vs-alpha} shows that it is easy to compute
$g_{f,\rho}$ and $g_{f,\rho}^{-1}$, as they are both solutions to
one-dimensional convex optimization problems and therefore admit efficient
binary search procedures. In some cases, we have closed forms; for example
 $f(t) = (t - 1)^2$ gives $g_{f,\rho}(\beta) = \hinges{\beta - \sqrt{2
    \rho \beta(1 - \beta)}}$, while $f(t) = \left|t-1\right|$ yields $g_{f,\rho}(\beta) = (\beta - \rho/2)_+$.
Another example:

\begin{example}[Total variation distances]
  \label{example:tv-distances}
  The total variation distance $\tvnorm{P - Q}$ corresponds
  to the choice $f(t) = |t - 1|$ via the identity
  $2 \tvnorm{P - Q} = \fdiv{P}{Q}$. For this case, we see
  immediately that $g_{f,\rho}^{-1}(\tau) = \min\{\tau + \frac{\rho}{2},
  1\}$, and then $g_{f,\rho}(\beta) = \hinge{\beta - \rho/2}$.
\end{example}
\noindent
Letting $g = g_{f,\rho}$ for shorthand,
we sketch how to compute $g^{-1}$ efficiently in more generality.
Computing the inverse $g^{-1}(\tau)$ is equivalent to solving the
optimization problem
\begin{align*}
  \maximize_{0 \le \beta, z \le 1} ~ \beta
  ~~~ \subjectto ~ z \le \tau, ~~
  \beta f\left(\frac{z}{\beta}\right) + (1 - \beta) f\left(\frac{1 - z}{
    1 - \beta}\right) \le \rho.
\end{align*}
We seek the largest $\beta \ge \tau$ feasible for this problem (as $\beta =
\tau$ is feasible); because $h(\beta, z) = \beta f(z/\beta) + (1 -
\beta) f((1 - z) / (1 - \beta))$ is convex and minimized at any $z = \beta$
with $h(z,z) = 0$,
for $\beta \ge \tau$ it is evident that
$\inf_{0 \le z \le \tau} h(\beta, z) = h(\beta, \tau)$. Thus
may equivalently write
\begin{align*}
  g^{-1}_{f,\rho}(\tau)
  = \sup\left\{\beta \in [\tau, 1] \mid
  \beta f\left(\frac{\tau}{\beta}\right) + (1 - \beta)
  f\left(\frac{1 - \tau}{1 - \beta}\right) \le \rho \right\},
\end{align*}
which a binary search over feasible $\beta \in [\tau, 1]$ solves
to accuracy $\epsilon$ in time $\log \frac{1 - \tau}{\epsilon}$.

\subsection{Achieving coverage with empirical estimates}
\label{sec:get-coverage-for-divs}

With the characterization of $\WCQuantile$, we can define
the corresponding prediction set
\begin{align}
  \label{eqn:robust-prediction-set}
  C_{f,\rho}(x; P)
  \defeq \{ y \in \mc{Y} \mid \score(x,y) \le \WCQuantile_{f,\rho}(1-\alpha; P)
  \}.
\end{align}
As we observe only a sample $\{\scorerv_i\}_{i=1}^n \simiid P_0$, we
use the empirical plug-in to develop confidence
sets~\eqref{eqn:robust-prediction-set} (and therefore in
problem~\eqref{eqn:robust-quantile-prob}), considering
$\what{C}_{n,f,\rho}(x) \defeq C_{f,\rho}(x; \empP)$.  By doing this,
Proposition~\ref{proposition:rho-vs-alpha} allows us to derive guarantees
for the prediction set~\eqref{eqn:robust-prediction-set} from standard
quantile statistics.  In particular, the next proposition, whose proof we
give in Appendix~\ref{sec:proof-cvg-only-test}, lower bounds future coverage
conditionally on the validation set $\{(X_i,Y_i)\}_{i=1}^n$ and relates
future test coverage to the amount of shift.
\begin{proposition}
  \label{proposition:cvg-only-test}
  Let $\scorerv_{n+1} = \score(X_{n+1}, Y_{n+1}) \sim P_{\textup{test}}$ be
  independent of $\{\scorerv_i\}_{i=1}^n \simiid P_0$, and let $\rho\opt =
  \fdivs{P_{\textup{test}}}{P_0} \in \openright{0}{\infty}$.  Let $F_0$ be
  the c.d.f.\ of $P_0$.  Then the confidence set $\what{C}_{n,f,
    \rho}(x) \defeq C_{f,\rho}(x; \empP)$ satisfies
  \begin{align*}
    \P \left( Y_{n+1} \in \what{C}_{n,f, \rho}(X_{n+1}) \mid
    \{(X_i,Y_i)\}_{i=1}^n \right)
    & \geq
    g_{f,\rho^\star} \left( F_0 \big( \WCQuantile_{f,\rho}(1-\alpha; \hat P_n) \big) \right) \\
    & =  g_{f,\rho^\star} \left( F_0 \big( \Quantile ( g_{f,\rho}^{-1}(1-\alpha); \hat P_n ) \big) \right).
  \end{align*}
\end{proposition}

With the two preceding propositions, we turn to the main coverage
theorem and a few corollaries, which provide the
validity of coverage as long as the true shift between $P_0$ and
$P_{\textup{test}}$ is no more than our guess.
We provide the proof of the theorem in
Appendix~\ref{sec:proof-robust-coverage-marginal}.
\begin{theorem}
  \label{theorem:robust-coverage-marginal}
  Assume that $\scorerv_{n+1} = \score(X_{n+1}, Y_{n+1}) \sim \Ptest$ is
  independent of $\{S_i\}_{i=1}^n \simiid P_0$, and let $\rho\opt
  = \fdivs{\Ptest}{P_0} < \infty$.
  Then
  \begin{equation*}
    \P\left(Y_{n + 1} \in \what{C}_{n,f,\rho}(X_{n+1})\right)
    \ge g_{f,\rho\opt}\left(\frac{\ceil{n g_{f,\rho}^{-1}(1 - \alpha)}}{n+1}
    \right).
  \end{equation*}
\end{theorem}

The theorem as stated is a bit unwieldy, so we develop a few
corollaries, whose proofs we provide in
Appendix~\ref{sec:proof-alpha-coverages}. In each, we assume that the
$\rho$ we use to construct the confidence
sets~\eqref{eqn:robust-prediction-set} satisfies $\rho \ge \rho\opt =
\fdivs{\Ptest}{P_0}$, which guarantees validity.
\begin{corollary}
  \label{corollary:almost-alpha-coverage}
  Let the conditions of Theorem~\ref{theorem:robust-coverage-marginal}
  hold, but additionally assume that
  $\rho\opt = \fdivs{\Ptest}{P_0} \le \rho$. Then
  for $c_{\alpha,\rho,f} \defeq g_{f,\rho}^{-1}(1 - \alpha) g_{f,\rho}'(
  g_{f,\rho}^{-1}(1 - \alpha)) < \infty$, we have
  \begin{equation*}
    \P\left(Y_{ n +1} \in \what{C}_{n,f,\rho}(X_{n+1})\right)
    \ge 1 - \alpha - \frac{c_{\alpha,\rho,f}}{n + 1}.
  \end{equation*}
\end{corollary}
\noindent
If instead
we replace $\alpha$ in the definition~\eqref{eqn:robust-prediction-set}
of the confidence set $C_{f,\rho}(x; P)$ with
\begin{equation*}
  \alpha_n \defeq 1 - g_{f,\rho}\left((1 + 1/n) g_{f,\rho}^{-1}(1 - \alpha)
  \right)
  = \alpha - O(1/n),
\end{equation*}
we can construct the corrected empirical confidence set
\begin{equation*}
  \what{C}_{n,f,\rho}^{\textup{corr}}(x)
  \defeq \left\{y \in \mc{Y} \mid \score(x, y)
  \le \WCQuantile_{f,\rho}(1 - \alpha_n; \empP) \right\}.
\end{equation*}
We then have the correct level $\alpha$ coverage:
\begin{corollary}
  \label{corollary:corrected-alpha-coverage}
  Let the conditions of Corollary~\ref{corollary:almost-alpha-coverage}
  hold. Then
  \begin{equation*}
    \P\left(Y_{n + 1} \in \what{C}_{n,f,\rho}^{\textup{corr}}(X_{n + 1})\right)
    \ge 1 - \alpha.
  \end{equation*}
\end{corollary}
\noindent
An easier corollary is immediate via Example~\ref{example:tv-distances},
which shows that when the data distribution changes in variation
distance by at most $\rho$, we have (nearly) correct coverage by
an identical increase in the choice of quantile level:
\begin{corollary} \label{corollary:total-variation}
  Let $f(t) = |t - 1|$. Then
  \begin{equation*}
    \what{C}_{n,f,\rho}(x) \defeq
    \left\{y \in \mc{Y} \mid \score(x, y) \le
    \Quantile\left(1 - \alpha + \frac{\rho}{2}; \hat{P}_n\right)\right\}
  \end{equation*}
  and if $2 \tvnorm{\Ptest - P_0} \le \rho$, then
  \begin{equation*}
    \P\left(Y_{n + 1} \in \what{C}_{n,f,\rho}(X_{n+1})\right)
    \ge 1 - \alpha - \frac{1}{n}.
  \end{equation*}
\end{corollary}

Summarizing, the empirical prediction sets $\what{C}_{n,f,\rho}$ and
$\what{C}^{\textup{corr}}_{n,f,\rho}$ achieve nearly or better than
$1-\alpha$ coverage if the $f$-divergence between the new distribution
$\Ptest$ and the current distribution $P_0$ remains below $\rho$. When this
fails, Theorem~\ref{theorem:robust-coverage-marginal} shows graceful
degradation in coverage as long as the divergence between $\Ptest$ and the
validation population $P_0$ is not too large.

\ifdefined\removenotes
\else
\subsection{Towards more general uncertainty sets}
We now provide an alternative characterization of $\WCQuantile$ that also adapts to other types of uncertainty sets.  
Our jumping off point is the quantile definition~\eqref{eqn:wc-quantile}: for any continuous distribution $P$,  observe that
\begin{align}
    \sup_{\fdivs{P_1}{P} \le \rho} \Quantile(1-\alpha; \, P_1) & = \sup_{\fdivs{P_1}{P} \le \rho} \inf_{t \in \R} \Bigg\{ t : P_1(s(X,Y) \leq t) \geq 1-\alpha \Bigg\} \notag \\
    & \overset{(i)}= \sup_{\fdivs{P_1}{P} \le \rho} \sup_{t \in \R} \Bigg\{ t : P_1(s(X,Y) > t) > \alpha \Bigg\} \notag \\
    & = \sup_{t \in \R} \Bigg\{ t : \sup_{\fdivs{P_1}{P} \le \rho} \Big\{ \mathbb{E}_{P_1} \big[ \indic{s(X,Y) > t} \big] \Big\} > \alpha \Bigg\}, \label{eqn:wc-quantile-simple}
\end{align}
where equality $(i)$ follows from the fact that for all $t \in \R$,  $t < \Quantile(1-\alpha, P)$ if and only if $P( \scorerv \le t) < 1-\alpha$, i.e.\ if and only if $P( \scorerv > t) > \alpha$.

\textbf{Efficient computation for Wasserstein balls.}  This simple duality result~\eqref{eqn:wc-quantile-simple} gives us an avenue for computing the worst-case quantile w.r.t.~different uncertainty sets.  For example, if we replace the $f$-divergence ball $\fdivs{P_1}{P}$ appearing in our derivations above with the $p$-Wasserstein distance ball $W_p(P_1, P)$, then we need only replace the above inner maximization problem in~\eqref{eqn:wc-quantile-simple} with
\begin{align}
\label{eqn:wasserstein-inner-max}
\maximize_{W_p(P_1, P) \le \rho} ~ P_1\left[ \scorerv > t \right],
\end{align}
and we can then still perform the outer maximization via bisection. 
Problem~\eqref{eqn:wasserstein-inner-max} aims to transport as much mass as possible to $t+ \varepsilon$ for some infinitesimal $\varepsilon>0$, and hence is equivalent to the maximization program
\begin{align*}
\maximize_{w : \R \to [0,1]} ~ \E_P\left[ w(S) \right] ~ \subjectto ~ \E_P \left[ (t - S)_+^p w(S)\right] \le \rho^p.
\end{align*}
Letting $t^\star \in \R \cup \{ - \infty\}$ and $w \in [0,1]$ solve the equation
\begin{align*}
\E_P\left[ (t - S)_+^p \left( \indic{S > t^\star} + w \indic{S = t^\star} \right) \right] = \rho^p,
\end{align*}
we see that the solution is $w^\star(s) \defeq \indic{s > t^\star}$ for $s \neq t^\star$ and $w^\star(t^\star) = w$. 

\fi


\section{Procedures for estimating future distribution shift}
\label{sec:futureshiftestimation}

While the results in the previous section apply for a fixed shift amount
$\rho$, a fundamental challenge is---given a validation data set---to
determine the amount of shift against which to protect.  We suggest a
methodology to identify shifts motivated by two
(somewhat oppositional) perspectives: first, the variability in predictions
in current data is suggestive of the amount of variability we might expect
in the future; second, from the perspective of protection against future
shifts, that there is no reason future data would \emph{not} shift as much
as we can observe in a given validation set.
%
As a motivating thought experiment, consider the case that the data is a
mixture of distinct sub-populations.  Should we provide valid coverage for
each of these sub-populations, we expect our coverage to remain valid if the
future (test) distribution remains any mixture of the same sub-populations.
In empirical risk minimization (ERM)-based models, we expect rarer
sub-populations to have higher non-conformity scores than average, and
building on this intuition, our procedures look for regions in validation
data with high non-conformity scores, choosing $\rho$ to give valid coverage
in these regions.

We adopt a two-step procedure to describe the set of shifts we consider.
Abstractly, let $\mc{V}$ be a (potentially infinite) set indexing
``directions'' of possible shifts, and to each $v \in \mc{V}$ associate a
collection $\mc{R}_v$ of subsets of $\mc{X}$.  (Typically, we either take
$\mc{V} \subset \R^d$ when $\mc{X} \subset \R^d$, or $\mc{V}$ a subset of
functions of $\mc{X}$, with each $\mc{R}_v$ then a collection of level
sets).  For each $R \in \mc{R} \defeq \bigcup_{v \in \mc{V}} \mc{R}_v
\subset \mc{P}(\mc{X})$, we consider the shifted distribution
\begin{align}
  \label{eqn:distribution-shifts-covariate}
  dQ_{R}(x,y) = \frac{\indic{x \in R }}{Q_0(X \in R)}dQ_0(x,y)
  = dQ_0(x,y \mid x \in R),
\end{align}
which restricts $X$ to a smaller subset $R$ of the feature space without
changing the conditional distribution of $Y \mid X$.  The intuition
behind the approach is twofold: first, conditionally
valid predictors remain valid under covariate shifts of only
$X$ (so that we hope to identify failures of validity under such shifts),
and second, there may exist 
privileged directions of shift in the
$\mc{X}$-space (e.g.\ time in temporal data or protected
attributes in data with sensitive features) for which we wish to provide
appropriate $1-\alpha$ coverage.

\begin{example}[Slabs and Euclidean balls]
  \label{ex:slab-euclideanballs}
  Our prototypical example is slabs (hyperplanes) and Euclidean
  balls, where we take $\mc{V} \subset \R^d$, both of which
  have VC-dimension $O(d)$. In the slab case, for
  $v \in \R^d$ we define the collection
  of slabs orthogonal to $v$,
  \begin{equation*}
    \mc{R}_v = \left\{ \{x \in \R^d \mid a \le v^T x \le b \}
    ~\mbox{s.t.}~ a < b \right\}.
  \end{equation*}
  In the Euclidean ball case, we consider $\mc{R}_v = \{ \{x
  \in \R^d \mid \ltwo{x-v} \le r \} ~\mbox{s.t.}~ r>0 \}$, the collection of
  $\ell_2$-balls centered at $v \in \mc{V} = \R^d$.
\end{example}

\begin{example}[Upper-level functional sets]
  \label{ex:level-sets}
  A more general example takes $\mc{V}$ be a collection of real-valued
  functions, for instance, a reproducing kernel Hilbert space (RKHS).  For
  each $v \in \mc{V}$, $ \mc{R}_v$ is then the collection of upper
  level-sets
  \begin{equation*}
   \left\{ \{x \in \mc{X} \mid v(x) \ge a \}
    ~\mbox{s.t.}~ a \in \R \right\}.
  \end{equation*}
  Were $\mc{V}$ all measurable functions, this would guarantee coverage
  under any covariate shift; practically, $\mc{V}$ is a (much)
  smaller collection.
\end{example}

Given $\delta \in (0, 1)$, we define the \emph{worst coverage}
for a confidence set mapping $C : \mc{X} \toto \mc{Y}$
over $\mc{R}$-sets of size $\delta$ by
\begin{align}
  \label{eqn:worst-global-coverage}
  \wcoverage(C, \mc{\mc{R}}, \delta; \, Q) \defeq
  \inf_{R \in \mc{R}}
  \left\{Q\left( Y \in C(X) \mid X \in R \right)
  ~ \mbox{s.t.}~
  Q(X \in R) \ge \delta \right\}
\end{align}
Our goal is to find a (tight) confidence set
$\what{C}$ such that
$\wcoverage(\what{C}, \mc{R}, \delta; Q_0) \ge 1 - \alpha$, which,
in the setting of Section~\ref{sec:robustpredinference},
corresponds to
choosing $\rho > 0$ such that
\begin{equation*}
  \wcoverage(\what{C}_{n,f,\rho},
  \mc{R}, \delta; \, Q_0) \ge 1-\alpha.
\end{equation*}
That is, we seek $1 - \alpha$ coverage over all large enough subsets
of $X$-space. 

\citet{BarberCaRaTi19a} show that one can theoretically construct such a
confidence set when the collection of sets $\mc{R}$ is not too large,
e.g.\ if it has finite VC-dimension.  Unfortunately, the computation of the
worst coverage~\eqref{eqn:worst-global-coverage} is usually challenging when
the dimension $d$ of the problem grows (as in
Example~\ref{ex:slab-euclideanballs}), as it typically involves
minimizing a non-convex function over a $d$-dimensional domain.  This makes
the estimation of quantity~\eqref{eqn:worst-global-coverage} intractable for
large $d$ and hints that requiring such coverage to hold uniformly
over all directions $v \in \mc{V}$ may be too stringent for practical
purposes.  However, for a fixed $v \in \R^d$, both sets $\mc{R}_v$ in
Example~\ref{ex:slab-euclideanballs} admit $O(d \cdot n)$-time algorithms
for computing $\wcoverage(C, \mc{R}_v, \delta; \what{Q}_n)$ for any
empirical distribution $\what{Q}_n$ with support on $n$ points, which in the
slab case is the maximum density segment problem~\cite{ChungLu03}. Thus
instead of the full worst coverage~\eqref{eqn:worst-global-coverage}, we
typically resort to a slightly weaker notion of robust coverage, where we
require coverage to hold for ``most'' distributions of the
form~\eqref{eqn:distribution-shifts-covariate}.  In the next two sections,
we therefore consider two approaches: one that samples directions $v \in
\mc{V}$, seeking good coverage with high probability, and the other that
proposes surrogate convex optimization problems to find the worst direction
$v$, which we can show under (strong) distributional assumptions is optimal.


\subsection{High-probability coverage over specific classes of shifts}
\label{sec:coverage-high-probability-over-shifts}

Our first approach is to let $\Pv$ be a distribution on $v \in \mc{V}$ that models plausible future shifts. 
A natural desiderata here is to provide
coverage with high probability, that is, conditional on $\what{C}$, to
guarantee that for a hyperparameter $0 < \levelv < 1$ and
for $v \sim \Pv$,
\begin{align}
  \label{eqn:worst-quantiled-coverage}
  \Pv \left[ \wcoverage(\what{C}, \mc{R}_v, \delta; \, Q_0) \geq 1 - \alpha \right]
  \geq 1-\levelv.
\end{align}
Thus with $\Pv$-probability $1 - \levelv$ over the direction $v$ of shift,
the confidence set $\what{C}(X)$ provides $1-\alpha$ coverage over all $R
\in \mc{R}_v$ satisfying $Q_0(X\in R) \ge \delta$.  The
coverage~\eqref{eqn:worst-quantiled-coverage} becomes more conservative as
$\levelv$ decreases to $0$, reducing to
condition~\eqref{eqn:worst-global-coverage} when $\levelv = 0$.

Before presenting the procedure, we index the confidence sets by the
threshold $q$ for the score function $\score$, providing a complementary
condition via the robust prediction set~\eqref{eqn:robust-prediction-set}.

\begin{definition}
  \label{def:smallest-rho}
  For $q \in \R$, the \emph{prediction set at level} $q$ is
  \begin{align*}
    C^{(q)}(x) \defeq \{ y \in \mc{Y} \mid \score(x,y) \le q \}.
  \end{align*}
  For a distribution $P$ on $\R$, the
  value $\rho$ provides
  \emph{sufficient divergence for threshold $q$} if
  \begin{equation*}
    C_{f,\rho}(x; P) \supset C^{(q)}(x) ~ \mbox{for~all~} x \in \mc{X}.
  \end{equation*}
\end{definition}
\noindent
By the definition~\eqref{eqn:robust-prediction-set} of
$C_{f,\rho}$ and Proposition~\ref{proposition:rho-vs-alpha}, we see that $\rho$ 
gives sufficient divergence for threshold $q$ if and only if
\begin{equation*}
  \WCQuantile_{f,\rho}(1 - \alpha; P)
  = \Quantile(g_{f,\rho}^{-1}(1 - \alpha; P)) \ge q.
\end{equation*}

To output a confidence set $\what{C}$ satisfying the high probability
worst-coverage~\eqref{eqn:worst-quantiled-coverage}, we wish to find $q \in
\R$ such that $\Pv[ \wcoverage(C^{(q)}, \mc{R}_v, \delta; \, Q_0) \geq 1 -
  \alpha] \geq 1-\levelv$. Notably, any choice of $\rho$ satisfying
$\WCQuantile_{f,\rho}(1 - \alpha; P_0) \ge q$ yields a prediction set
$C_{f,\rho}(\cdot; P_0)$ that both provides coverage for covariate shifts
$Q_R$ of the form~\eqref{eqn:distribution-shifts-covariate} across most
directions $v \in \mc{V}$, in agreement
with~\eqref{eqn:worst-quantiled-coverage}, and enjoys the protection against
distribution shift we establish in Section~\ref{sec:robustpredinference} for
the given value $\rho$ (including against more than covariate shifts).
Algorithm~\ref{alg:rho-selection-procedure} performs this using plug-in
empirical estimators for $P_0$, $Q_0$ and $\Pv$.

\begin{algorithm}
  \caption{Worst-subset validation procedure}
  \label{alg:rho-selection-procedure}
  \begin{algorithmic}
    \STATE {\bf Input:} sample $\{(X_i,Y_i) \}_{i=1}^n$ with empirical
    distribution $\empQ$; score function $\score: \mc{X} \times \mc{Y}
    \to \R$ with empirical distribution $\empP$ on $\{\score(X_i,
    Y_i)\}_{i=1}^n$; levels $\alpha, \levelv \in (0,1)$; divergence function
    $f : \R_+ \to \R$; smallest subset size $\delta \in (0,1)$;
    number of sampled directions $k \ge 1.$


    \STATE {\bf Do:} Sample $\{ v_j \}_{j=1}^k \simiid \Pv$, and let
    $\Pvemp$ be their empirical distribution and set
    \begin{align}
      \label{eqn:rhohat-criterion}
      \what q_\delta
      \defeq \inf \Bigg\{ q \in \R : \Pvemp \Big( \wcoverage(C^{(q)}, \mc{R}_v, \delta; \, \what{Q}_{n}) \geq 1-\alpha \Big) \geq 1-\levelv \Bigg\}.
    \end{align}

    \STATE Set $\hat{\rho}_\delta$ to be any sufficient
    divergence level for threshold $\what{q}_\delta$.

    \STATE \textbf{Return:}
    confidence set mapping $\what{C} : \mc{X} \toto \mc{Y}$ with
    $\what C(x)  \defeq C^{(\what q_\delta ) }(x)$ or
    $\what{C}(x) \defeq C_{f,\hat \rho_\delta}(x; \empP)$.
  \end{algorithmic}
\end{algorithm}


We show that procedure~\ref{alg:rho-selection-procedure}
approaches uniform $1 - \alpha$ coverage if the subsets
in $\mc{R}$ have finite VC-dimension in Appendix~\ref{sec:adaptive_procedure_shifts_high_prob_coverage}.

\subsection{Finding directions of maximal shift}

In this section, we revisit worst potential shifts, designing a
methodology to estimate the worst direction and protect
against it, additionally providing sufficient conditions for consistency.
For a confidence
set mapping $C : \mc{X} \toto {Y}$, we define the worst shift direction
\begin{equation}
  v\subopt(C) \defeq \argmin_{v \in \mc{V}} \wc(C, \mc{R}_v, \delta; \, Q_0),
  \label{eqn:worst-shift-def}
\end{equation}
which evidently satisfies
\begin{equation*}
  \wc(C, \mc{R}_{v\subopt(C)}, \delta; \, Q_0) 
  = \wc(C, \mc{R}, \delta; \, Q_0) 
  \defeq \inf_{v \in \mc{V}} \wc(C, \mc{R}_{v}, \delta; \, Q_0).
\end{equation*}
If we could identify such a worst direction, and it is consistent across
thresholds $q$ in our typical definition $C(x) = \{y \in \mc{Y} \mid
\score(x, y) \le q\}$ (a strong condition), then the procedures in the
preceding sections allow us to choose thresholds to guarantee coverage.
The intuition here is that there may exist a direction with
higher variance in predictions, for example, time in a temporal system.
A more explicit example comes from heteroskedastic regression:

\begin{example}[Heteroskedastic regression]
  \label{example:heterogeneous-regression}
  Let the data $(X, Y) \in \R^d \times \R$ follow the model
  \begin{equation*}
    Y = \mu\opt(X) + h(\vvar^T X) \noise
  \end{equation*}
  where $h : \R \to \R_+$ is non-decreasing, $\noise \sim \normal(0,1)$
  independent of $X$, which generalizes the standard regression model
  to have heteroskedastic noise, with the noise increasing in the
  direction $\vvar$. Evidently the oracle (smallest length) conditional
  confidence set for $Y \mid X = x$ is the
  interval $[\pm z_{1 - \alpha/2} \sqrt{h(\vvar^T x)}]$ where $z_{1 - \alpha}$
  is the standard normal quantile. The standard split
  conformal methodology (Section~\ref{sec:split-conformal-intro})
  will undercover for those $x$ such that $\vvar^T x$ is large: shifts
  of $X$ in the direction $v\subopt = \vvar$ may decrease coverage.
\end{example}

With this example as motivation, we propose identifying challenging
directions for dataset shift by separating those datapoints $(X_i, Y_i)$
with large nonconformity scores $\score(X_i, Y_i)$ from those with lower
scores.  In principle, one can use any M-estimator to find such a
discriminator. 
\begin{definition}
  \label{def:smallest-rho-s}
  For $q \in \R$ and a score  $\score : \mc{X} \times
  \mc{Y} \to \R$, the \emph{$\score$-prediction set at level $q$} is
  \begin{equation}
    \label{eqn:confidence-set-from-score}
    C^{(q,\score)}(x) \defeq \{ y \in \mc{Y} \mid \score(x,y) \le q \}.
  \end{equation}
\end{definition}

We assume in this section that $\mc{V} \subset L^2(Q_{0,X})$ is an RKHS, or a
subset thereof, with associated Hilbert norm $\norm{\cdot}_{\mc{V}}$, and
each collection $\mc{R}_v$ is as in Example~\ref{ex:level-sets}.  The case
where $\mc{R}$ is the collection of all half-spaces corresponds to $\mc{V} =
\left\{ x \mapsto v^T x \mid v\in \R^d \right\}$.  Additionally, for every
$v \in \mc{V}$ we let $F_v$ be the cumulative distribution function of
$v(X)$ when $X \sim Q_{0,X}$ and $F_v^-(t) \defeq \P(v(X) < t)$
its left-continuous version.

\begin{algorithm}
  \caption{Worst-direction validation given a score function}
  \label{alg:worst-direction-validation}
  \begin{algorithmic}
    \STATE {\bf Input:} sample $\{(X_i,Y_i) \}_{i=1}^n$; score function
    $\scoreest: \mc{X} \times \mc{Y} \to \R$ independent of the sample;
    coverage rate $1-\alpha \in (0,1)$; divergence function $f : \R_+ \to
    \R$; smallest subset size $\delta \in (0,1)$, worst direction estimation procedure $\mc{M} :(\R \times \mc{X})^\star \to \mc{V}$.

    \STATE {\bf Initialize:} Split sample $\{ (X_i, Y_i) \}_{i=1}^n$ into
    $\{ (X_i,Y_i) \}_{i=1}^{n_1}$, $\{ (X_i,Y_i) \}_{i=n_1+1}^{n_1+n_2}$
    with empirical distributions $\empQone$, and $\empQtwo$
    (resp.\ $\empPone$ and $\empPtwo$ for the scores).
    
    \STATE {\bf Do}: Estimate the worst direction of shift on the first sample distribution $\empQone$:
    \begin{align*}
    \what{v}_n \defeq \mc{M}(\{ \scoreest(X_i, Y_i), X_i \}_{i=1}^{n_1}).
    \end{align*}

    \STATE Use the second subsample to set the threshold
    $\what{q}_\delta$ to
    \begin{align}
      \label{eqn:rhohat-criterion-worst-direction-after-learning-scores}
      \what{q}_\delta
      \defeq \inf \left\{ q \in \R : \textup{WC}(C^{(q,\scoreest)}, \mc{R}_{\hat{v}_n}, \delta; \, \empQtwo) \geq 1-\alpha \right\}.
    \end{align}

    \STATE Set $\hat{\rho}_\delta \defeq \rho_{f,\alpha}( \what q_\delta;
    \empPtwo) = \sup\{\rho \ge 0 \mid
    \WCQuantile_{f,\rho}(1 - \alpha; \empPtwo) \le q\}$
    as in Lemma~\ref{lemma:equivalence-rho-threshold}.

    \STATE \textbf{Return:} the confidence set mapping $\what{C}_n(x)
    = C^{(\what q_\delta, \scoreest)}(x) = C_{f,\hat{\rho}_\delta}(x;
    \empPtwo)$.
  \end{algorithmic}
\end{algorithm}

The intuition behind Algorithm~\ref{alg:worst-direction-validation} is
simple: we seek a direction $v$ in which shifts in $X$ make the given
nonconformity score $\scoreest$ large, then guarantee coverage for
shifts in that direction and, via the distributionally robust
confidence set $C_{f,\hat{\rho}}$ the procedure returns, any
future distributional shift for which the distribution
$P_{\textup{new}}$ of scores $\score(X, Y)$ satisfies
$\fdiv{P_{\textup{new}}}{P_0} \le \hat{\rho}$. Because we need only
solve a single M-estimation problem---rather than sample a large
number of directions $v$ as in Alg.~\ref{alg:rho-selection-procedure}---the estimation methodology is more computationally efficient.

In Appendix~\ref{subsec:population-level-consistency}, we study different worst direction estimation procedures,  for instance the non-parametric estimator
\begin{align}
\label{eqn:penalized-linear-loss-finite-sample-worst-direction-estimator}
\hat{v}_{n,  \lambda_n} \defeq \argmin_{v \in \mc{V}} \left\{ \frac{1}{n(n-1)} \sum_{1\le i \neq j \le n} \left( v(X_i) - \indic{\scorerv_i \ge \scorerv_j} \right)^2  + \lambda_n \norm{v}_{\mc{V}}^2 \right\},
\end{align}
whose consistency to an oracle worst direction depends on stochastic order assumptions. 

In our subsequent experiments, with a high-dimensional feature space, we use
simpler least-squares and SVM estimators of the scores as a fitting
procedure for the worst direction of shift, considering linear shifts only.
This parametric approach is admittedly more restrictive, and obtaining
consistency requires even stronger distributional assumptions; we present one example of such in Appendix~\ref{subsec:consistency-linear-shift-estimator}.


\newcommand{\predsetthresh}{t}
\newcommand{\infdiv}[2]{D_{\infty}\left({#1} |\!| {#2}\right)}
\newcommand{\Aug}{\textrm{A}}
\newcommand{\nBatch}{B}
\newcommand{\batch}{b}
\newcommand{\qfunc}{\mc{Q}}

\newcommand{\Pcovset}{\mc{P}_{\textup{cov},I}}

\section{Coverage sensitivity under covariate shifts}
\label{sec:sensitivity}

To this point in the paper, the approaches we take for robustness to
distribution shifts may often be conservative. Here, we take a complementary
and exploratory viewpoint to identify ways in which a predictor may be
sensitive. While coverage guarantees of standard predictive inference
methods may fail when new data comes from a shifted distribution (recall
Section~\ref{sec:motivation-exp}), protecting against all possible shifts
can lead to conservative predictive sets.  It is thus of practical interest
to identify the particular directions in which a predictive model is indeed
distributionally unstable.  We therefore propose a measure that evaluates
coverage sensitivity under distribution shifts of interest, and we study
this measure's convergence properties by building on a recent line of work
distribution shift sensitivity~\citep{JeongNa20, SubbaswamyAdSa21,
  GuptaRo21}.

%

For a choice of threshold $\predsetthresh \in \R$, we wish to understand the
sensitivity of (mis)-coverage of the predictive set $C^{(\predsetthresh)}$
(as in Eq.~\eqref{eqn:confidence-set-from-score}) under covariate specific
distribution shifts.  In distinction from
Section~\ref{sec:robustpredinference}, where we consider general shifts on
the score distribution, we now focus on covariate shift.
For an index set $I \subset [d]$,
this consists
of allowing only the distribution of $X_I$ to
vary while the conditional distribution of $\score(X,Y)$ given $X_I$
remains invariant.  Thus, if we let $P_{0,I}$ be the distribution
of $(\score(X,Y), X_I)$ when $(X,Y) \sim Q_0$, we consider shifts
of measures on $(S, X_I)$ belonging to
\begin{align*}
  \Pcovset(\rho, P_{0,I}) \defeq \left\{ P
  = \law(S, X_I)
  ~ \mid ~
  P(S \in \cdot \mid X_I) = P_{0,I}(S \in \cdot \mid X_I) \text{ and } \fdiv{P}{P_{0,I}} \le \rho \right\}.
\end{align*}
Assuming (as we will show is possible) that we can accurately
evaluate coverage under such shifts, if a given scoring function $\score$ is
insensitive, then we gain confidence in the performance of $\score$,
while scoring functions sensitive to such covariate shifts should give
us pause.

The challenge of calibrating the expected distribution shift, denoted as $\rho$, is akin to calibrating sensitivity parameters in causal inference \citep{HsuSm13,VeitchZa20}. Our methods identify and assess the sensitivity of coverage to shifts in specific covariate subsets. While training data alone can not provide such calibration, access to relevant test covariate subsets can help us approximate these shifts using techniques like \citep{NguyenWaJo10}, requiring only subset data rather than full labeling—a practical advantage in many cases.

\subsection{Covariate-specific sensitivity analysis}

\newcommand{\SFcov}{{\SF}_{\textup{cov}, I}^{(\predsetthresh)} }
\newcommand{\estSFcov}{{\what{\SF}}_{\textup{cov}, I}^{(\predsetthresh)} }
\newcommand{\MCov}{{\MC}_{P_{0,I}}^{(\predsetthresh)}}

Our goal here is to estimate scoring model's \emph{sensitivity}, which we
take to be the mis-coverage of the predictive set function
$C^{(\predsetthresh)}$ as the distribution of $(S, X_I)$ varies within
$\Pcovset(\rho, P_{0,I})$.
For shift amounts $\rho \ge 0$ and probability distributions
(indexed by $I$) $P_{0,I}$ on $\R \times \R^I$,
we therefore define the covariate specific sensitivity function
\begin{align*}
  \SFcov (\rho,P_{0,I}) \defeq \sup_P \left\{  \E_{P}[\indic{S >
  \predsetthresh}] \mid P \in
  \Pcovset(\rho, P_{0,I}) \right\}.
\end{align*}
Define the conditional miscoverage function on $\R^I$ by
\begin{equation}
  \label{eqn:conditional-miscoverage-func}
  \MCov (x) \defeq \E_{P_{0,I}}[\indic{S>\predsetthresh} \mid X_I=x],
\end{equation}
so we can express the sensitivity function as
\begin{align}
  \label{eqn:covsens-fcn}
  \SFcov (\rho, P_{0,I}) = \sup \left\{ \E_{P}[ \MCov(X_I) ] \mid  P \in \Pcovset(\rho, P_{0,I}) \right\},
\end{align}
as the covariate shift only affects the marginal distribution of $X_I \in
\R^I$ by assumption.

The goal is to leverage equation~\eqref{eqn:covsens-fcn} to build a
consistent estimator of the sensitivity function.  Given a sample $\{
\scorerv_i, X_{I,i} \}_{i=1}^n \simiid P_{0,I}$, a natural approach is to
follow a two step procedure, by first computing an estimate
$\what{\MC}_{n_1}^{(\predsetthresh)}$ of the miscoverage function using the
first $n_1$ points, and then approximating the sensitivity function with the
remaining $n_2 = n-n_1$ data points, forming the naive estimate
\begin{equation*}
  \what{\SF}_{\textup{naive},I}(\rho)
  \defeq \sup_{Q} \left\{ \E_{X_I \sim Q}\big[\what{\MC}_{n_1}^{(\predsetthresh)}(X_I)\big]
  \text{ s.t. } \fdivs{Q}{\hat{Q}_{n_2, I}}  \leq \rho \right\},
\end{equation*}
where $\hat{Q}_{n_2, I}$ is the empirical distribution of $\{ X_{I,i}
\}_{n_1 < i \le n}$.

Unfortunately, if
$\what{\MC}_{n_1}^{(\predsetthresh)}$ converges to
$\MC_{P_{0,I}}^{(\predsetthresh)}$ at a slower rate than $\sqrt{n}$, we
expect the same behavior from $\what\SF_{\textup{naive}, I}$, so we take a
different tack.
In the next section, we show instead how, given an additional (large)
sample of \emph{unlabeled} data $\{X_i\}_{i = 1}^N$, we can achieve a
$\sqrt{n}$-consistent asymptotically normal estimate of $\SFcov$ using a
debiasing correction~\citep{ChernozhukovChDeDuHaNeRo16,
  JeongNa20,SubbaswamyAdSa21}.\footnote{\label{footnote:hong-wrong}
Notably, \citet{JeongNa20} and
\citet{SubbaswamyAdSa21} perform sensitivity analyses to distribution shift
for various semiparametric functionals related to that here. We present
alternative results and proofs as their results appear to have incorrect
proofs. \citet[Thm.~1]{SubbaswamyAdSa21} builds off of \citet[Lemma
  14]{JeongNa20}, whose proof~\cite[Appendix C.3]{JeongNa20} appears to have
a mistake: in the final line of the proof, they use that their functionals
$\mu : \mc{X} \to \R$ of interest have densities uniformly bounded away from
0, but nowhere do they assume this or argue that it must hold.}
A trade-off is that our debiasing typically leads to a loss
of monotonicity of the
estimate $\what{\SF}_{\text{cov}, I, n}^{(\predsetthresh)}$ in the parameter
$\rho \ge 0$. For
clarity we focus on a particular limiting divergence, the
R\'{e}nyi $\infty$-divergence
\begin{align*}
  \infdiv{P}{Q} \defeq \lim_{k \to \infty} \frac{1}{k-1} \log \left\{ \int \left( \frac{dP(z)}{dQ(z)} \right)^k dQ(z) \right\} = \log \esssup_Q \left\{ \frac{dP}{dQ} \right\}.
\end{align*}
A quick calculation shows this corresponds to distribution balls of the form  
\begin{align*}
  \{ P : \infdiv{P}{P_{0,I}} \le \rho \} = \{P \mid
  \text{ there exists } P_{1, I}  \text{ s.t. } P_{0,I} = e^{-\rho} P +(1-e^{-\rho})P_{1,I} \},
\end{align*}
which offers a simpler dual representation for the sensitivity
function~\eqref{eqn:covsens-fcn}:
\begin{lemma}[Example 3, \citet{DuchiNa21}]
  \label{lemma:cvar-calc}
  Let $\Pcovset$ be defined via the
  R\'{e}nyi divergence $D_\infty$.
  Then the sensitivity function $\SFcov (\rho, P_{0,I})$ satisfies
  \begin{align*}
    \SFcov (\rho, P_{0,I}) =\inf_{\eta \in \R} \bigg\{
    e^{\rho}\E_{P_{0,I}}
    \left[
      \hinge{\MCov(X_I)-\eta}
      \right] +\eta
    \bigg\}.
  \end{align*}
\end{lemma}

\subsection{Cross-fit dual estimation of the sensitivity function}
\label{subsec:sens-f-inf-divergence}

In the general shift case, the finite sample estimator $\SF_\text{gen}(\rho,Q,\hat{P}_n)$ is $\sqrt{n}$-consistent for $\SF_\text{gen}(\rho,Q,P_0)$,  hence we wish to construct an estimator with an analogous guarantee for the covariate specific sensitivity function $\SFcov(\rho, P_{0,I})$.

For any pair of functions $h : \R^+ \times \R^I \to
\R$ and $m: \R^I \to \R$, define the augmentation function
$\Aug^{(\predsetthresh)}_{h,m} : \R_+ \times \R \times \R^I \to \R$ by
\begin{align*}
  \Aug^{(\predsetthresh)}_{h,m}(\rho,  s,  x)
  \defeq h(\rho, x) \left(\indic{s>\predsetthresh}- m(x) \right).
\end{align*}
Let $\qfunc_0(m,\rho) \defeq \argmin_{\eta \in \R}
\big\{e^{\rho}\E[\hinge{m(X_I)-\eta}] +\eta \big\}$ be the $1-e^{-\rho}$
quantile of $m(X)$ under $P_{0,I}$.  For shorthand, omit the subscript on
the miscoverage~\eqref{eqn:conditional-miscoverage-func} to write
${\MC}^{(\predsetthresh)} \equiv \MCov$, define $\qfunc^{(\predsetthresh)}(\rho)
\defeq \qfunc_0({\MC}^{(\predsetthresh)},\rho)$, and choose
$h^{(\predsetthresh)}(\rho, x) \defeq
e^{\rho}\indics{{\MC}^{(\predsetthresh)}(x)>
  \qfunc^{(\predsetthresh)}(\rho)}$.  Then
$\E_{P_{0,I}}[\Aug^{(\predsetthresh)}_{h^{(\predsetthresh)},
    {\MC}^{(\predsetthresh)}}(\rho, \scorerv, X_I)] = 0$, so for all $\rho
>0$, Lemma~\ref{lemma:cvar-calc} shows that
\begin{align}
  \label{eqn:covsens-fcn-aug}
  \SFcov (\rho, P_{0,I}) &= e^{\rho} \E\left[
    \hinge{{\MC}^{(\predsetthresh)}(X_I)- \qfunc^{(\predsetthresh)}(\rho)}\right]
  +\qfunc^{(\predsetthresh)}(\rho) +
  \E[\Aug^{(\predsetthresh)}_{h^{(\predsetthresh)}, {\MC}^{(\predsetthresh)}}(\rho, \scorerv,  X_I)].
\end{align}

Algorithm~\ref{alg:sensitivity1} proceeds by first estimating
$\MC^{(\predsetthresh)}$, $\qfunc^{(\predsetthresh)}$ and
$h^{(\predsetthresh)}$ successively, before leveraging
equation~\eqref{eqn:covsens-fcn-aug} to form a ``debiased" cross-fit
estimator of $\SFcov (\rho, P_{0,I})$.  As mentioned above, it assumes
access to a set of unlabeled examples $\{ X_{I,j} \}_{1\le j \le N}$ where
$N \gg n$, which we use to estimate $\qfunc^{(\predsetthresh)}$ from
${\MC}^{(\predsetthresh)}$. Intuitively, this allows us to accurately
estimate properties of $\sigma(X_I)$-measurable variables, and it is
reasonable in semi-supervised regimes where unsupervised examples are cheaper than labeled data.  Appendix~\ref{sec:cross-fit-sketch-proof-intuition} provides additional intuition on the introduction of the augmentation term $\Aug^{(\predsetthresh)}_{h,m}$.

\begin{algorithm}
  \caption{Covariate sensitivity estimation}
  \label{alg:sensitivity1}
  \begin{algorithmic}
    \STATE {\bf Input:} $\nBatch$-fold partition $\cup_{\batch=1}^\nBatch \mc{I}_\batch=[n]$ of $\{(S_i,X_{I,i})\}_{i=1}^n$ s.t. $\left| \mc{I}_\batch \right| =\frac{n}{\nBatch}$, unlabeled samples $\{ X_{I, j} \}_{1\le j \le N}$, fitting procedure $\mc{A} : \left( \R \times \R^I \right)^\star \to \{ \R^I \to \R\}$.

 \FOR{$\batch \in [\nBatch]$}

  \STATE Fit an estimator
$
  \what {\MC}^{(\predsetthresh)}_\batch \defeq \mc{A}\left( (S_i,X_{I,i})_{i \in \mc{I}_\batch^c} \right)
$ of the miscoverage function ${\MC}^{(\predsetthresh)}$.

  \STATE Compute the $e^{-\rho}$-approximate quantile of $\what {\MC}^{(\predsetthresh)}_\batch$ as
  \begin{equation}
    \label{eqn:quantile-est-unlabeled}
    \what \qfunc_\batch^{(\predsetthresh)}(\rho) \defeq
    \argmin_{\eta \in \R} \left\{ \sum_{j = 1}^N
    \frac{e^{\rho}}{N}
    \hinge{\what {\MC}^{(\predsetthresh)}_\batch(X_{I,j})-\eta} + \eta \right\}.
  \end{equation}

  \STATE Set $\what h^{(\predsetthresh)}_\batch(\rho, x) \defeq e^\rho \indics{ \what {\MC}^{(\predsetthresh)}_\batch(x) > \what \qfunc_\batch^{(\predsetthresh)}(\rho)}$.

    \STATE Compute the $b$-th fold augmented estimator
    \begin{equation}
      \label{eqn:sens-func-aug-estimator}
      \what{\SF}^{(q)}_{\batch,n}(\rho) \defeq  \frac{1}{\left| \mc{I}_\batch \right|} \sum_{i \in \mc{I}_\batch}  \left\{ e^{\rho}
      \hinge{\what{\MC}^{(\predsetthresh)}_\batch(X_{I,i})- \what \qfunc_\batch^{(\predsetthresh)}(\rho)} + \Aug^{(\predsetthresh)}_{\what h^{(\predsetthresh)}_\batch, \what {\MC}^{(\predsetthresh)}_\batch}(\rho, \scorerv_i, X_{I,i})  \right\} +  \what \qfunc_\batch^{(\predsetthresh)}(\rho).
    \end{equation}

   \ENDFOR
   \RETURN
    \begin{align}
    \label{eqn:final-sens-aug-estimator}
    \what{\SF}^{(\predsetthresh)}_n(\rho) \defeq \frac{1}{\nBatch} \sum_{\batch=1}^\nBatch \what{\SF}^{(\predsetthresh)}_{\batch,n}(\rho)
    \end{align}

  \end{algorithmic}
\end{algorithm}

\section{Empirical analysis}
\label{sec:exps}

Given the challenges arising in the practice of machine learning and
statistics, this paper argues that methodology equipping models with a
notion of validity in their predictions---\eg, conformalization procedures
as in this paper---is essential to any modern prediction pipeline.  Section
\ref{sec:motivation-exp} illustrates the need for these sorts of procedures,
showing that the standard conformal methodology is sensitive to even small
shifts in the data, through (semi-synthetic) experiments on data from the
UCI repository.  In Section~\ref{sec:robustpredinference}, we propose
methods for robust predictive inference, giving methodology that estimates
the amount of shift to which we should be robust.  Fuller justification
requires a more careful empirical study that highlights both the failures of
non-robust prediction sets on real data as well as the potential to handle
such shifts using the methodology here.  To that end, we turn to
experimental work:

\begin{itemize}
\item Section~\ref{sec:real-experiments} shows
  evaluation centered around the new MNIST, CIFAR-10, and ImageNet test
  sets. These datasets exhibit real-world distributional
  shifts, and we test whether our methodology of estimating
  plausible shifts is sufficient to provide coverage in these real-world shifts.
  \item In Appendix~\ref{sec:uci-experiments},
  we resume the evaluation of our own methodology on the
  semi-synthetic data from Section~\ref{sec:motivation-exp}.
\item In Appendix~\ref{sec:real-experiments-covid}, we consider a time
  series where the goal is to predict the fraction of people testing
  positive for COVID-19 throughout the United States.
\item In Appendix~\ref{sec:covariate-sensitivity},
  we apply Algorithm~\ref{alg:sensitivity1} to evaluate the sensitivity
  of predictive methods to individual covariate shifts.
\end{itemize}

\subsection{CIFAR-10, MNIST, and ImageNet datasets}
\label{sec:real-experiments}

We evaluate our procedures on the CIFAR-10,
ImageNet, and MNIST datasets~\cite{KrizhevskyHi09,
  RussakovskyDeSuKrSaMaHuKaKhBeBeFe15, LeCunCoBu98}, which
continue to play a central role in the evaluation of machine learning
methodology.
Concerns about overfitting to these benchmarks
motivate \citet*{RechtRoScSh19} to create new test sets for both
CIFAR-10 and ImageNet by carefully following the original dataset creation
protocol. Though these new test sets strongly resemble the original
datasets, as \citeauthor{RechtRoScSh19} observe, the natural variation
arising in the creation of the new test sets yield evidently significant
differences, giving organic dataset shifts on which to
evaluate our procedures.  Our goal here is to show that even when we do not know the actual amount of shift, our methodology from Section~\ref{sec:coverage-high-probability-over-shifts} can still give reasonable estimates of it that translate into marginal coverage on these datasets.

\begin{figure}[ht]
    \centering
    \includegraphics[scale=0.5]{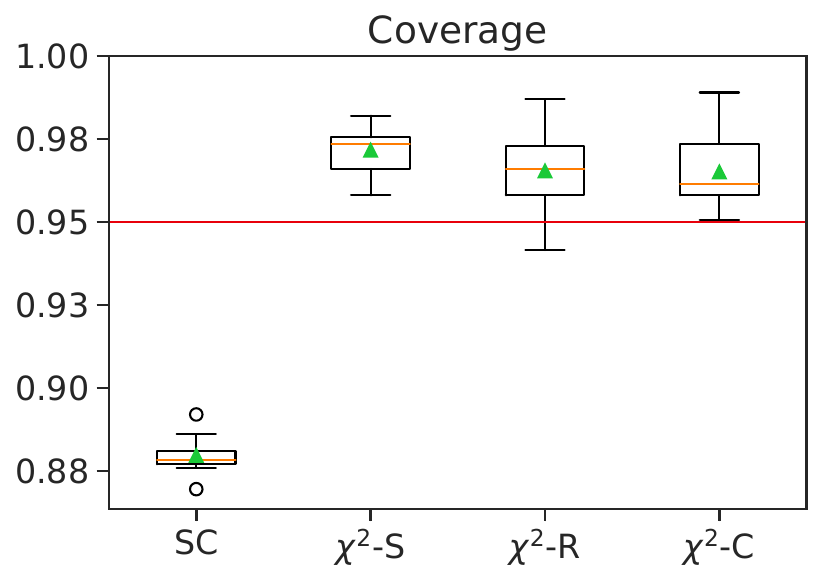}
    \hfill
    \includegraphics[scale=0.5]{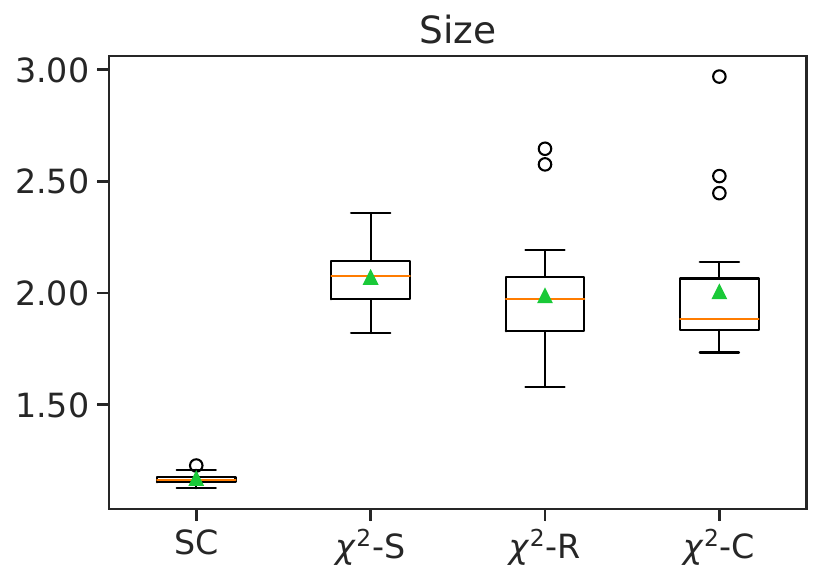}
    \caption{Empirical coverage and average size for the prediction sets
    generated by the standard conformal methodology (``SC'') and the chi-squared
    divergence, across 20 random splits of the CIFAR-10 data.  We set $\rho$
    according to the sampling (``$\chi^2$-S''), regression (``$\chi^2$-R''), and
    classification-based (``$\chi^2$-C'') strategies for estimating the amount
    of shift that we describe in Section in \ref{sec:futureshiftestimation}.
    The horizontal red line marks the marginal coverage $.95$.}
    \label{fig:cifar-10}
\end{figure}
\begin{figure}[ht]
    \centering
    \includegraphics[scale=0.5]{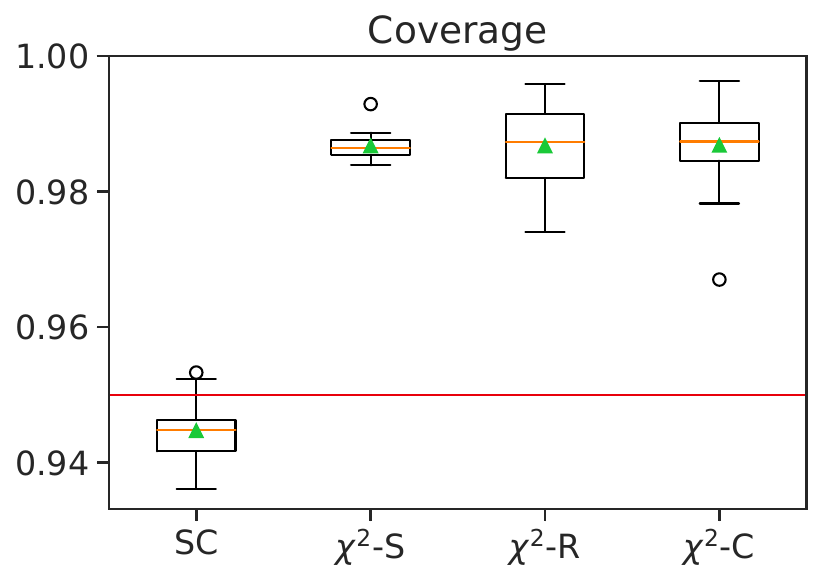}
    \hfill
    \includegraphics[scale=0.5]{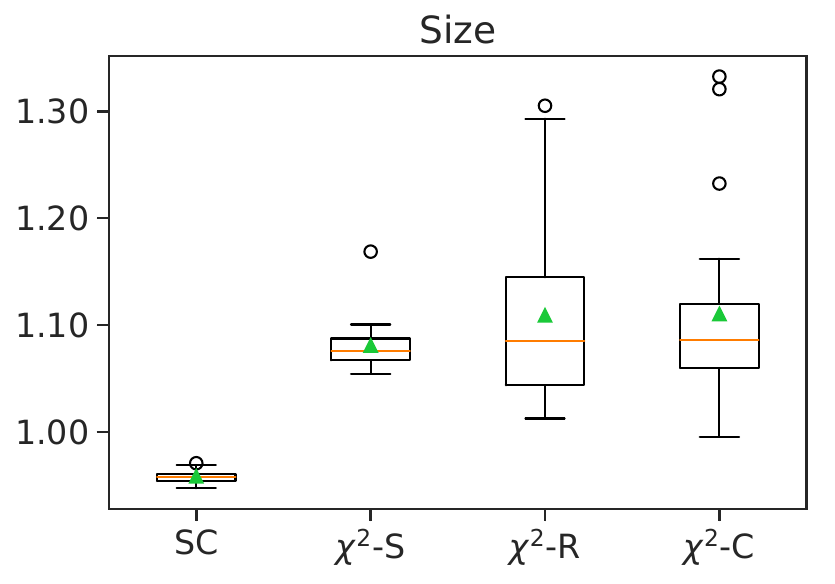}
    \caption{Empirical coverage and average size for the prediction sets
    generated by the standard conformal methodology (``SC'') and the chi-squared
    divergence, across 20 random splits of the MNIST data.  We set $\rho$
    according to the sampling (``$\chi^2$-S''), regression (``$\chi^2$-R''), and
    classification-based (``$\chi^2$-C'') strategies for estimating the amount
    of shift that we describe in Section in \ref{sec:futureshiftestimation}. The
    horizontal red line marks the marginal coverage $.95$.}
    \label{fig:mnist}
\end{figure}
\begin{figure}[ht]
    \centering
    \includegraphics[scale=0.5]{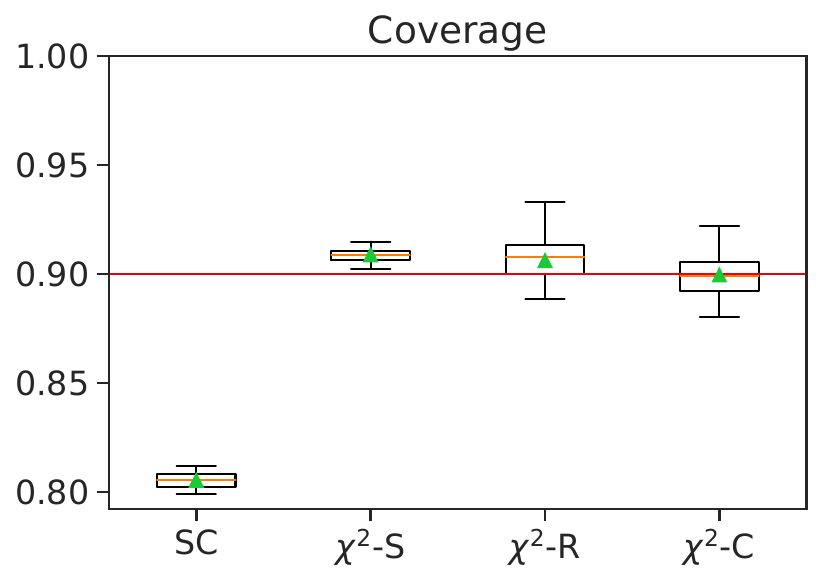}
    \hfill
    \includegraphics[scale=0.5]{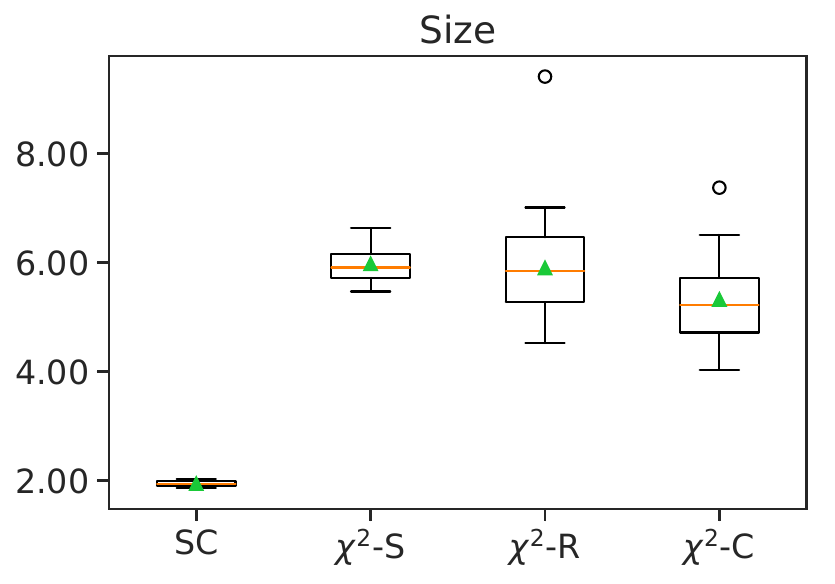}
    \caption{Empirical coverage and average size for the prediction sets
    generated by the standard conformal methodology (``SC'') and the chi-squared
    divergence, across 20 random splits of the ImageNet data.  We set $\rho$
    according to the sampling (``$\chi^2$-S''), regression (``$\chi^2$-R''), and
    classification-based (``$\chi^2$-C'') strategies for estimating the amount
    of shift that we describe in Section in \ref{sec:futureshiftestimation}. The
    horizontal red line marks the marginal coverage $.9$.}
    \label{fig:imagenet}
\end{figure}

We evaluate on the three datasets as follows.  We use 70\% of the
original CIFAR-10, MNIST, and ImageNet datasets for training, and treat the
remaining 30\% as a validation set.  We fit a standard
ResNet-50~\cite{HeZhReSu16} to the training data, and use the negative
log-likelihood $\score(x, y) = -\log p_\theta(y \mid x)$, where $p_\theta(y
\mid x)$ is the output of the (top) sigmoid layer of the network, as the
scoring function on the validation data for our conformalization procedures.
We compare our procedures to the split conformal methodology on three new
datasets nominally generated identically to the initial datasets: the
CIFAR-10.1~v4 dataset~\cite{RechtRoScSh19}, which consists of 2,000 
32$\times$32 images from 10 different classes; the QMNIST50K data,
which extends MNIST to consist of 50,000 28$\times$28 images
from 10 classes~\cite{YadavBo19}; and the ImageNetV2 Threshold0.7
data~\cite{RechtRoScSh19}, consisting of 10,000 images from 200 classes.  In
each test of robust predictive inference, we set the level of robustness to achieve the nominal coverage $\alpha = .05$ for the CIFAR-10 and MNIST datasets and $\alpha = .1$ for the ImageNet dataset, by using the data-driven strategies that we detail in
Section~\ref{sec:futureshiftestimation}: sampling directions of shift from
the uniform distribution on the unit sphere
(Alg.~\ref{alg:rho-selection-procedure}), estimating the shift direction via
regression (Alg.~\ref{alg:worst-direction-validation}) or via
classification, which replaces the regression step in
Alg.~\ref{alg:worst-direction-validation} with a support vector machine
(SVM) to separate the largest 50\% of scores $\score(X_i, Y_i)$ from the
smallest.
In contrast to our experiments from Section~\ref{sec:uci-experiments} with semi-synthetic data, we cannot compute the exact level of shift here; the
question is whether the provided methodology provides marginal coverage.

Figures \ref{fig:cifar-10}, \ref{fig:mnist}, and \ref{fig:imagenet} present
the results for each setup over 20 random splits of the data. As is apparent
from the figures, we see that the standard conformal methodology fails to
correctly cover.  As both the new CIFAR-10 and ImageNet test sets exhibit
larger degradations in classifier performance (increased error) than does
the MNIST test set~\cite{RechtRoScSh19}, we expect the failure of standard
conformal to be pronounced on these two datasets.  Indeed, the split
conformal method (Sec.~\ref{sec:split-conformal-intro}) provides especially
poor coverage on these datasets, where it yields average coverage .88
(instead of the nominal .95) and .8 (instead of the nominal .9) on the new
CIFAR-10 and ImageNet test sets, respectively.  On the other hand, our
inferential methodology consistently gives more coverage regardless of the
strategy used to estimate the amount of divergence $\rho$, with the sampling
strategy notably consistently delivering marginal coverage without
over-covering.  The uniformity in coverage across the three strategies is
notable, as our procedures for estimating the amount of shift assume some
structure for the underlying shift, which is unlikely to be consistent with
the provenance of the new test sets.

In our experiments, estimating the direction of shift using either
regression or classification (Alg.~\ref{alg:worst-direction-validation}) is
faster than sampling directions (Alg.~\ref{alg:rho-selection-procedure});
the former takes time $O(nd \min\{n, d\})$ and
the latter $O(k n d)$, where $k$ is the number
of sampled directions $v$, using a linear-time implementation for
computing the worst coverage (maximum density segment) along
a direction $v$~\cite{ChungLu03}.
The difference of course depends on the desired sampling frequency
$k$.


Finally, the aforementioned validity does not (apparently) come
with a significant loss in statistical efficiency:
Figures~\ref{fig:cifar-10}, \ref{fig:mnist}, and \ref{fig:imagenet} show
that our confidence sets are not substantially larger than those coming from
standard conformal inference---which may be somewhat surprising, given the
relatively large number of classes present in the ImageNet dataset.

\section{Discussion and conclusions}

We have presented methods and motivation for robust predicitve inference,
seeking protection against distribution shift. Our arguments and perspective
are somewhat different from the typical approach in distributional
robustness~\cite{ DuchiNa21, BlanchetKaMu19,
  SagawaKoHaLi20}, as we wish to maintain validity in prediction.  
A number of future directions and questions remain unanswered. Perhaps
the most glaring is to fully understand the ``right'' level of
robustness. While this is a longstanding problem~\cite{DuchiNa21}, we
present approaches to leverage the available validation data.  Alternatives
might be compare new covariates and test data
$X$ to the available validation data. \citet{TibshiraniBaCaRa19}
suggest an importance-sampling approach for this, reweighting
data based on likelihood ratios, which may sometimes be feasible
but is likely impossible in high-dimensional scenarios. It would be
interesting, for example, to use projections of the data
to match $X$-statistics on new test data, using this to
generate appropriate distributional robustness sets.
We hope that the perspective here inspires renewed consideration of
predictive validity.

\section{Disclosure statement:}

The authors report there are no competing interests to declare.



\newpage

\setcounter{page}{1}
\appendix

\begin{center}
{\large\bf SUPPLEMENTARY MATERIAL for Robust Validation: Confident Predictions Even When Distributions Shift}
\end{center}
\section{Theoretical developments on procedures for estimating future distribution shift}

\subsection{High-probability coverage over specific classes of shifts}
\label{sec:adaptive_procedure_shifts_high_prob_coverage}

\begin{assumption}[Score continuity]
  \label{assumption:continuity-scores-v}
  The distribution of the scores under $P_0$ is continuous.
\end{assumption}

\begin{theorem}
  \label{theorem:high-probability-coverage}
  Let $\what C$ be the prediction set
  Alg.~\ref{alg:rho-selection-procedure} returns. Assume
  that $\mc{R} = \bigcup_{v
    \in \mc{V}} \mc{R}_v$ has VC-dimension $\VC(\mc{R}) < \infty$.
  Then there exists a universal constant $c < \infty$ such that the
  following holds. For
  all $t > 0$, defining
  \begin{equation*}
    \alpha_{t,n}^\pm \defeq \alpha \pm c
    \sqrt{\frac{\VC(\mc{R}) \log n + t}{\delta n}},
    \; \textrm{ and }\;
    \delta_{t,n}^\pm = \delta \pm c \sqrt{\frac{\VC(\mc{R}) \log n + t}{n}},
  \end{equation*}
  then with probability at least $1 - e^{-t}$ over
  $\{X_i, Y_i\}_{i = 1}^n \simiid Q_0$ and
  $\{v_i\}_{i = 1}^k \simiid \Pv$,
  \begin{equation*}
    \Pv \Big( \wcoverage(\what C, \mc{R}_v, \delta_{t,n}^+; \, Q_0)
    \ge 1 - \alpha_{t,n}^+ \Big) \ge 1 - \levelv - c \sqrt{\frac{1 + t}{k}}.
  \end{equation*}
  If additionally Assumption~\ref{assumption:continuity-scores-v}
  holds, then
  \begin{align*}
    \Pv  \Big( \wcoverage(\what C, \mc{R}_v, \delta_{t,n}^-; \, Q_0)
    \le 1 - \alpha_{t,n}^-  \Big)
    \ge \levelv - \frac{1}{k} - c\sqrt{\frac{1 + t}{k}}.
  \end{align*}
\end{theorem}
\noindent
See Appendix~\ref{sec:proof-high-probability-coverage} for a proof of
the theorem.

Theorem~\ref{theorem:high-probability-coverage} shows that
Procedure~\ref{alg:rho-selection-procedure} approaches uniform $1-\alpha$
coverage if the subsets in $\mc{R}$ have finite VC-dimension. More
precisely, the estimate $\hat{\rho}_\delta$ almost achieves the randomized
worst-case coverage~\eqref{eqn:worst-quantiled-coverage}: with probability
nearly $1-\levelv$ over the direction $v \sim \Pv$, $\what{C}$ provides
coverage at level $1-\alpha - O(1/\sqrt{n})$ for all shifts $Q_R$ (as in
Eq.~\eqref{eqn:distribution-shifts-covariate}) satisfying $R \in \mc{R}_v$
and $Q_0(X \in R) \ge \delta$. The second statement in the theorem
is an insurance against drastic overcoverage:
while we cannot guarantee that the
worst coverage is always no more than $1 - \alpha$, we can guarantee
that---if the scores are continuous---then the empirical set $\what{C}$
has worst coverage \emph{no more} than $1 - \alpha + O(1/\sqrt{n})$
for at least a fraction $\levelv$ of directions $v \sim \Pv$.
In a sense, this is unimprovable: if the worst coverage
$W = \wcoverage(C, \mc{R}_v, \delta; Q_0)$ is continuous
in $v$, the best we could expect is that
$\Pv(W \ge 1 - \alpha) = 1 - \levelv$ while
$\Pv(W < 1 - \alpha) = \levelv$.

As a last remark, we note that when the scores are distinct,
there is a complete equivalence between thresholds $q$ and
divergence levels $\rho$ in Algorithm~\ref{alg:rho-selection-procedure};
see Appendix~\ref{sec:proof-equivalence-rho-threshold} for a proof.
\begin{lemma}
  \label{lemma:equivalence-rho-threshold}
  Assume that the scores $\score(X_i, Y_i)$ are all distinct.
  Define $\rho_{f, \alpha}(q; P) \defeq \sup\{\rho \ge 0 \mid
  \WCQuantile_{f,\rho}(1 - \alpha; P) \le q\}$ and let
  $\what{\rho}_\delta = \rho_{f,\alpha}(\what{q}_\delta, \empP)$. Then
  $C^{(\what{q}_\delta)} = C_{f, \what{\rho}_\delta}(\cdot; \empP)$.
\end{lemma}

\subsection{Population-level consistency of the worst direction}
\label{subsec:population-level-consistency}
The consistency of Algorithm~\ref{alg:worst-direction-validation} with the adequate worst-direction estimation procedure $\mc{M}$ requires strong
assumptions, somewhat oppositional to the distribution-free coverage we seek
(though again we still have the distributionally robust protections).  Yet
it is still of interest to understand conditions under which Alg.~\ref{alg:worst-direction-validation} is consistent; as we show here,
in examples such as heteroskedastic regression
(Ex.~\ref{example:heterogeneous-regression}), this holds.
We turn to our assumptions.


A challenge is that the worst direction $v\subopt(C^{(q,\score)})$ may vary substantially in $q$.  One condition sufficient to ameliorate this reposes on stochastic orders, where for random variables $U$ and $V$ on
$\R^d$, we say $U$ stochastically dominates in the \emph{upper orthant order} $V$, written $U \stocuo V$,
if $\P(U \ge t) \ge \P(V \ge t)$ for all $t \in \R^d$ (see~\cite[Ch.~6]{ShakedSh07}, where this is called the
\emph{usual stochastic order}).  
Letting $\mc{L}$ denote the law of a random variable, we write $\law(U) \stocuo \law(V)$ if $U \stocuo V$.

\begin{assumption}
  \label{assumption:stochastic-dominance}
  There is a direction $v\opt \in \mc{V}$ such that $v\opt(X)$ has a continuous distribution and,  for all $v \in \mc{V}$, 
  \begin{equation*}
    \left( \score(X,Y), F_{v\opt}(v\opt(X)) \right) \stocuo  \left( \score(X,Y), F_{v}^-(v(X)) \right).
  \end{equation*}
\end{assumption}
\noindent
The intuition is that covariate shifts in direction
$v\opt$ not only increase nonconformity,  but that $v\opt$ is the worst such direction.
Assumption~\ref{assumption:stochastic-dominance} focuses on the dependence (copula) between $\score(X,Y)$ and $v(X)$, when $v$ ranges over all potential directions of shift in $\mc{V}$, and states that $v\opt(X)$ and $\score(X,Y)$ are more likely to take on larger values together. 
It only characterizes $v\opt$ up to an increasing transformation,  which is desirable as any such transformation leaves the collection $\mc{R}_v$ of upper-level sets invariant.

Under Assumption~\ref{assumption:stochastic-dominance}, confidence sets share the same worst shift $v\opt$:
\begin{lemma}
  \label{lemma:stochastic-domination-direction}
  Let Assumption~\ref{assumption:stochastic-dominance} hold. Then $v\opt$ is a worst shift~\eqref{eqn:worst-shift-def} for the confidence set~\eqref{eqn:confidence-set-from-score}, i.e. $v\opt \in v\subopt(C^{(q, \score)})$ for all $q \in \R$.
\end{lemma}
\noindent
We present the (nearly immediate) proof of
Lemma~\ref{lemma:stochastic-domination-direction} in
Appendix~\ref{sec:proof-stochastic-domination-direction}.
While Assumption~\ref{assumption:stochastic-dominance}
is admittedly strong, the next lemma (whose proof we provide in
Appendix~\ref{sec:proof-heterogeneous-regression-verifies}) shows
that it holds for linear shifts in the heteroskedastic regression
case of Example~\ref{example:heterogeneous-regression}.
\begin{lemma}
  \label{lemma:heterogeneous-regression-verifies}
  Assume the regression model of
  Example~\ref{example:heterogeneous-regression},
  $Y = \mu\opt(X) + h(X^T \vvar) \noise$,
  with noncomformity score
  \begin{equation*}
    \score(x, y) = (y - \mu\opt(x))^2
    ~~ \mbox{or} ~~
    \score(x, y) = |y - \mu\opt(x)|,
  \end{equation*}
  and let $\mc{V} = \{ x \mapsto v^T x \mid v \neq 0\}$ be the set
  of linear functions.  If $v^T X$ has a continuous distribution whenever $v
  \neq 0$, then $v\opt = \vvar$ satisfies
  Assumption~\ref{assumption:stochastic-dominance}.
\end{lemma}

We also suggest potential procedures for identifying the worst direction of
shift under limited computational and statistical power.  Ideally, a worst
shift direction should allow ranking examples by difficulty, with larger
values of $v\opt(X)$ corresponding to larger values of $\score(X,Y)$.  The
following lemma, whose proof we provide in
Appendix~\ref{sec:proof-worst-direction-order-consistency}, formalizes this
intuition, stating that the function $v\opt$ maximizes the correspondence of
the ranks of $n$ samples $(\scorerv_1, \dots, \scorerv_n)$ and $(v\opt(X_1),
\dots, v\opt(X_n))$.  For ease of notation, we denote $\scorerv_i \defeq
\score(X_i,Y_i)$ when appropriate.

\begin{lemma}
  \label{lemma:worst-direction-order-consistency}
  Let Assumption~\ref{assumption:stochastic-dominance} hold.  Given three
  i.i.d.\ samples $(X_1,Y_1)$, $(X_2, Y_2)$ and $(X_3, Y_3)$, the worst
  direction $v\opt$ satisfies
  \begin{align}
    \label{eqn:worst-direction-order-consistency}
    v\opt \in \argmax_{v \in \mc{V}} \left\{ \P\left( \scorerv_1 \ge \scorerv_2,  v(X_1) > v(X_3) \right) \right\}.
  \end{align}
\end{lemma}

While the natural empirical (finite-sample and non-convex) approximation of
the problem~\eqref{eqn:worst-direction-order-consistency} enjoys
$\sqrt{n}$-consistency, and \citet{Sherman94} characterizes its asymptotic
distribution, such statistical consistency often comes at the cost of
computational tractability, necessitating alternative
approaches~\cite{ClemenconLuVa08, DuchiMaJo13}.  Thus, we reframe our
problem as a binary classification problem with label $\indic{\scorerv_1 \ge
  \scorerv_2} \in \{0,1\}$ and feature vector $X_1 \in \mc{X}$, and consider
the following least squares problem:
\begin{align}
  \label{eqn:penalized-linear-loss-worst-direction-estimator}
  \minimize_{\substack{v \in \mc{V}}} \left\{  \E\left[ 
    \left( v(X_1) - \indic{\scorerv_1 \ge \scorerv_2} \right)^2  \right] \right\}.
\end{align}
The following lemma, whose proof we provide in
Appendix~\ref{sec:proof-penalized-linear-loss-expression}, shows that the
minimization
problem~\eqref{eqn:penalized-linear-loss-worst-direction-estimator} amounts
to projecting the function
\begin{align*}
  \eta_S (x) \defeq \P( \score(x,Y) \ge \score(X',Y') \mid X=x),
\end{align*}
where $(X, Y)$ and $(X', Y')$ are independent,
onto $\mc{V} \subset L^2(Q_{0,X})$.

\begin{lemma}
  \label{lemma:penalized-linear-loss-expression}
  The minimization
  problem~\eqref{eqn:penalized-linear-loss-worst-direction-estimator} is
  equivalent to
  \begin{align*}
    \minimize_{\substack{v \in \mc{V}}} \E\left[ \left( v(X) - \eta_S(X) \right)^2 \right].
  \end{align*}
  Additionally, if $\eta_S(X)$ has a continuous distribution, and if
  $(X_1,Y_1)$, $(X_2, Y_2)$ and $(X_3, Y_3)$ are i.i.d.\ and $\mc{F} = \{ f:
  \mc{X} \to \R \text{ measurable } \}$, then
  \begin{align*}
    \eta_S \in \argmax_{f \in \mc{F}} \left\{\P\left[ \scorerv_1 \ge \scorerv_2,  f(X_1) > f(X_3) \right] \right\} .
  \end{align*}
\end{lemma}

The function $\eta_S$ quantifies the ``hardness" of an instance $x \in
\mc{X}$ by comparing the score $\score(x,Y)$ to an independent sample
$\scorerv' = \score(X', Y')$: if $F_S$ is the c.d.f.\ of $\scorerv$, then
$\eta_S(x) = \E \left[ F_{\scorerv}(\score(X,Y)) \mid X=x \right]$.  At the
same time, it is the nonparametric
analogue of the the maximizers in
definition~\eqref{eqn:worst-direction-order-consistency}.

Moving to the finite-sample case, with a sample $\{(X_i,Y_i) \}_{i=1}^n$,  we solve the following convex minimization problem~\eqref{eqn:penalized-linear-loss-finite-sample-worst-direction-estimator} with a penalty $\lambda_n > 0$:
\begin{align*}
\hat{v}_{n,  \lambda_n} \defeq \argmin_{v \in \mc{V}} \left\{ \frac{1}{n(n-1)} \sum_{1\le i \neq j \le n} \left( v(X_i) - \indic{\scorerv_i \ge \scorerv_j} \right)^2  + \lambda_n \norm{v}_{\mc{V}}^2 \right\}.
\end{align*}
Under appropriate conditions on the RKHS $\mc{V}$, this method is provably consistent, in the sense that $\hat{v}_{n,  \lambda_n}$ converges towards $\eta_S$ as $n \to \infty$.
This also entails that, if Assumption~\ref{assumption:stochastic-dominance} holds for a vector space $\mc{V}$ sufficiently large,  then $v\opt$ must be a non-decreasing function of $\eta_S$. 
We summarize these results in the next proposition,  which essentially states that we can asymptotically recover the worst direction up to a non-decreasing function, and whose proof we provide in Appendix~\ref{sec:proof-consistency-worst-direction-rkhs}.

\begin{proposition}
  \label{prop:consistency-worst-direction-rkhs}
  Assume that $\mc{X}$ is a closed measurable space, and that $\mc{V}\subset
  L^2(Q_{0,X})$ is a dense, separable RKHS with a bounded measurable kernel.
  For any sequence $\lambda_n \to 0$ such that $n^{1/4}\lambda_n \to
  \infty$, we have
  \begin{align*}
    \int_{x \in \mc{X}} \left( \hat{v}_{n,  \lambda_n}(x) - \eta_S(x) \right)^2 dQ_{0,X}(x) \cas 0.
  \end{align*}
  Additionally, let Assumption~\ref{assumption:stochastic-dominance} hold
  and $\eta_S(X)$ have continuous distribution. Then there exists a
  non-decreasing function $F = F_{v\opt}^{-1} \circ F_{\eta_S}$ such that
  $v\opt(X)= F(\eta_S(X))$ almost surely.
\end{proposition}

\subsubsection{Consistency of linear shifts estimators}
\label{subsec:consistency-linear-shift-estimator}

Even if the parametric approach we adopt in our experiments is more restrictive than the above estimators, we can
still show that various M-estimators can identify the direction $v\opt$ when
Assumption~\ref{assumption:stochastic-dominance} holds.  We present one such
plausible result here, assuming (i)
Assumption~\ref{assumption:stochastic-dominance} holds for $\mc{V}$
consisting of linear shifts indexed by unit-norm vectors,
(ii) that for
some $\Sigma \succ 0$, $\Sigma^{-1/2} X$ is
rotationally invariant and has finite second moments,
and (iii) that
$\score(X, Y)$ is nonnegative and satisfies
$\E\left[ \score(X,Y)X\right] \neq 0$, and
$\E[ \score(X,Y)^2] < \infty$.

\begin{proposition}
  \label{proposition:vopt-least-squares}
  Let conditions (i)--(iii) above hold. Then $v\opt$ is proportional to the
  least-squares solution
  \begin{align}
    \label{eqn:vopt-least-sq}
    v\opt \propto \argmin_{v \in \R^d}
    \E\left[ \left( \score(X,Y) - v^T X  \right)^2 \right].
  \end{align} 
\end{proposition}
\noindent
See
Appendix~\ref{sec:proof-vopt-least-squares} for a proof.

Example~\ref{example:heterogeneous-regression} with $X \sim
\normal(0,\Sigma)$
and typical nonconformity scores
satisfies the conditions of
Proposition~\ref{proposition:vopt-least-squares}.  While in more general
models least squares estimation need not find the worst shift
direction, Proposition~\ref{proposition:vopt-least-squares} suggests
it may be a reasonable heuristic. We present the asymptotic estimation of the worst direction in Appendix~\ref{sec:asymptotic_est_worst_direction}.

\subsection{Asymptotic estimation of the worst direction}
\label{sec:asymptotic_est_worst_direction}

To enable our coming
analysis, we elaborate slightly and modify notation to reflect that the
scoring function $\score_n$ may change with sample size $n$.  We also refine
Definition~\ref{def:smallest-rho} of the confidence sets to explicitly
depend on both the threshold $q$ and score function $\score$.

With the population level recovery guarantees we establish in
Propositions~\ref{prop:consistency-worst-direction-rkhs} and~\ref{proposition:vopt-least-squares}, it
is now of interest to understand when we may recover the optimal worst
direction and corresponding confidence set $C$ using
Algorithm~\ref{alg:worst-direction-validation}, which has access only to samples from the population $Q_0$.  An immediate corollary of
Theorem~\ref{theorem:high-probability-coverage} ensures that, under
the same conditions,
Algorithm~\ref{alg:worst-direction-validation}
returns a confidence set mapping $\what{C}_n$ that satisfies, conditionally
on $\scoreest$ and $\hat{v}_n$ and with probability $1-e^{-t}$ over the second
half of the validation data,
\begin{align}
  \label{eqn:worst-coverage-v-hat}
  \wc( \what C_n, \mc{R}_{\hat{v}}, \delta_{t,n_2}^+; \, Q_0)
  \ge 1 - \alpha_{t,n_2}^+ ~ \text{and} ~
  \wc( \what C_n, \mc{R}_{\hat{v}}, \delta_{t,n_2}^-; \, Q_0)
  \le 1 - \alpha_{t,n_2}^-.
\end{align}
Recalling the definition~\eqref{eqn:worst-global-coverage}, it remains to
understand how close we can expect the uniform quantity $\wc( \what{C}_n,
\mc{R},\delta; \, Q_0)$ to be to $1-\alpha$.  By the
bounds~\eqref{eqn:worst-coverage-v-hat}, if the worst coverage is continuous
in $v \in \mc{V}$ and $\scoreest$ and $\hat{v}_n$ are appropriately
consistent, we should expect a uniform $1-\alpha$ coverage guarantee in the
limit as $n \to \infty$.

To present such a consistency result, we require a few additional
assumptions.
\begin{assumption}[Consistency of scores and directions]
  \label{assumption:consistency-of-scores-and-vhat}
  As $n \to \infty$, we have
  \begin{align*}
    \ltwopxs{\scoreest - \score}^2
    &\defeq \int_{\mc{X} \times \mc{Y}} \left(\scoreest(x,y) - \score(x,y) \right)^2
    dQ_0(x,y)  = o_{P}(1)
    ~~\text{and} ~~
    \norm{\what{v}_n-v^*}_{L^2(Q_{0,X})}  = o_{P}(1).
  \end{align*}
\end{assumption}

\begin{assumption}[Continuous distributions]
  \label{assumption:continuity-worst-coverage-1}
  For $(X,Y) \sim Q_0$, the random variables
  $\score(X,Y)$ and $v\opt(X)$ have continuous
  distributions. 
  Additionally,  $\hat v_n(X)$ has a continuous distribution with probability tending to $1$ as $n \to \infty$.
\end{assumption}

\begin{assumption}[Distinct scores]
  \label{assumption:estimated-scores-as-distinct}
  The scores are asymptotically distinct in probability,
  \begin{align*}
    Q_{0}^n \left[ \mbox{there~exist} ~ i,j \in [n], i \neq j ~ \mbox{s.t.}
      ~ \scoreest(X_i,Y_i) = \scoreest(X_j,Y_j) \right] \cp 0.
  \end{align*}
\end{assumption}
\noindent
Assumption~\ref{assumption:estimated-scores-as-distinct} is a technical
assumption that will typically hold whenever
Assumption~\ref{assumption:continuity-worst-coverage-1} holds, for example,
if $\scoreest$ belongs to a parametric family.

Under these assumptions, Theorem~\ref{theorem:uniform-asymptotic-coverage}
proves that we asymptotically provide uniform coverage at level $1-\alpha$
over all shifts $Q_R$, $R \in \mc{R}$.  See
Appendix~\ref{sec:proof-uniform-asymptotic-coverage} for a proof.

\begin{theorem}
  \label{theorem:uniform-asymptotic-coverage}
  Let Assumptions~\ref{assumption:stochastic-dominance},
  \ref{assumption:consistency-of-scores-and-vhat},
  and
  \ref{assumption:continuity-worst-coverage-1} hold.  Then
  Algorithm~\ref{alg:worst-direction-validation}
  returns a confidence set mapping $\what{C}_n$ that satisfies
  \begin{align*}
    \wc( \what{C}_n, \mc{R}, \delta; \,  Q_{0}) = 1 - \alpha + u_n +
    \varepsilon_n
  \end{align*}
  where $u_n \ge 0$ and $\varepsilon_n \cp 0$ as $n \to \infty$.
  If additionally Assumption~\ref{assumption:estimated-scores-as-distinct}
  holds, then $u_n \cp 0$.
\end{theorem}

To conclude, we see that the M-estimation-based
Procedure~\ref{alg:worst-direction-validation} to find the worst shift
direction can be consistent.  Yet even without the (strong) assumptions
Theorem~\ref{theorem:uniform-asymptotic-coverage} requires, we contend the
methodology in Algorithm~\ref{alg:worst-direction-validation} (and
Alg.~\ref{alg:rho-selection-procedure}) is valuable: it is important to look
for variation in coverage within a dataset and to protect against
possible future dataset shifts.
In particular,  Assumption~\ref{assumption:stochastic-dominance} only ensures that the function $\eta_S$ is the worst shift independently of the threshold $q \in \R$, i.e that $ \wc( C^{(q,\score)}, \mc{R}_{\eta_S}, \delta; \, Q_0) = \wc( C^{(q,\score)},  \mc{R}, \delta; \, Q_0)$ for all $q \in \R$,  but in general, the function $\eta_S$ remains a reasonable estimation target in itself: one can view it as the ``average" worst direction over a random choice of threshold $\scorerv' \sim P_0$.

\section{Consistency results for Algorithm~\ref{alg:sensitivity1}}

\subsection{Intuition and sketch proof for the cross-fit augmented estimator~\eqref{eqn:final-sens-aug-estimator}}
\label{sec:cross-fit-sketch-proof-intuition}
We develop a debiased cross-fit estimator of $\SFcov$
using the representation in Lemma~\ref{lemma:cvar-calc}
and an additional unlabeled sample of covariates $X$, which helps to
estimate expectations of the form $\E[\hinge{M(X) - \eta}]$.
To build intuition, consider an (abstract)
functional of the form in Lemma~\ref{lemma:cvar-calc},
so that for a function $M : \mc{X} \to \R$ and $\eta \in \R$ we wish
to estimate
\begin{equation*}
  F_\rho(M, \eta) \defeq \E\left[e^\rho \hinge{M(X) - \eta}\right] + \eta.
\end{equation*}
Consider a first-order expansion of $F_\rho$ around $M_0, \eta_0$, where
$M_0(x) = \P(S > \predsetthresh \mid X = x)$ (recall
Eq.~\eqref{eqn:conditional-miscoverage-func}) and $\eta_0$ minimizes
$\E[e^\rho\hinge{M_0(X) - \eta}] + \eta$ (as in Lemma~\ref{lemma:cvar-calc})
and is thus the $1 - e^{-\rho}$ quantile of $M_0$. Then using that the
subdifferential $\frac{\partial}{\partial t} \hinge{t - \eta} = \indic{t >
  \eta}$, we heuristically (ignoring interchanges of differentiation and
integration) write
\begin{align*}
  F_\rho(M, \eta)
  & \approx F_\rho(M_0, \eta_0)
  + \E\left[ e^\rho \indic{M_0(X) > \eta_0}
    \left(M(X) - M_0(X) + \eta_0 - \eta\right) \right]
  + \eta - \eta_0,
\end{align*}
and rearranging,
\begin{align*}
  F_\rho(M_0, \eta_0)
  & \approx F_\rho(M, \eta)
  - \E\left[ e^\rho \indic{M_0(X) > \eta_0}
    \left(M(X) - M_0(X)\right)\right] \\
  & \qquad - (\eta - \eta_0)
  \left(1 - e^\rho \P(M_0(X) > \eta_0)\right) \\
  & = F_\rho(M, \eta) - e^\rho \E\left[\indic{M_0(X) > \eta_0}
    \left(M(X) - M_0(X)\right)\right],
\end{align*}
where we used that $\P(M_0(X) > \eta_0) = e^{-\rho}$ by construction.
For $(M, \eta)$ ``near enough'' to $M_0, \eta_0$, we
have $\E[\indic{M_0 > \eta_0}(M - M_0)]
\approx \E[\indic{M > \eta}(M - M_0)]
\approx \E[\indic{M > \eta} (M - \indic{S > \predsetthresh})]$.
In short, we have sketched that
\begin{equation}
  F_\rho(M_0, \eta_0)
  \approx F_\rho(M, \eta)
  + \E\left[e^\rho
    \indic{M(X) > \eta} \left(\indic{S > \predsetthresh}
    - M(X) \right) \right],
  \label{eqn:intuition-for-augmentation}
\end{equation}
for $(M, \eta)$ appropriately near to the population quantities $M_0,
\eta_0$.  Our idea, then, is to use the first-order term in
Eq.~\eqref{eqn:intuition-for-augmentation}
to correct an empirical calculation of $F_\rho(M, \eta)$.

We first split the data $\{ (\scorerv_i, X_{I, i}) \}_{i=1}^n \simiid
P_{0,I}$ into $\nBatch \ge 2$ batches $\mc{I}_1, \dots, \mc{I}_\nBatch$ of
size $\frac{n}{\nBatch}$.  For each batch $\batch \in [\nBatch]$, we form an
estimate $\what {\MC}^{(\predsetthresh)}_\batch$ of
${\MC}^{(\predsetthresh)}$ using all samples from $[n] \setminus
\mc{I}_\batch$.  Using the pool of unlabeled samples, we then compute an
estimate $\what \qfunc_\batch^{(\predsetthresh)}(\rho)$ of the quantile
$\qfunc_0(\what {\MC}^{(\predsetthresh)}_\batch,\rho)$ (see
step~\eqref{eqn:quantile-est-unlabeled}), which then gives the $\batch$th
augmented estimator~\eqref{eqn:sens-func-aug-estimator}.  Average over all
$\nBatch$ batches gives the final estimator
$\what{\SF}^{(\predsetthresh)}_n(\rho)$.  (The augmentation term
$\Aug^{(\predsetthresh)}_{\what h^{(\predsetthresh)}_\batch, \what
  {\MC}^{(\predsetthresh)}_\batch}$ makes
$\what{\SF}^{(\predsetthresh)}_n(\rho)$ potentially non-monotonic in
$\rho$.) We study the consistency results of the sensitivity estimator~\eqref{eqn:final-sens-aug-estimator} in Appendix~\ref{sec:sensitivity_consistency}.

\subsection{Consistency and convergence rate of the augmented estimator $ \what{\SF}^{(q)}_n$}
\label{sec:sensitivity_consistency}
We study the consistency and rate of covergence of the sensitivity
estimator~\eqref{eqn:final-sens-aug-estimator} as a path function of $\rho
\in \R_+$, for which we require a few assumptions below.
Assumption~\ref{ass:miscov-estimator-consistent} states that the fitted
estimator $\what \MC_\batch$ needs to be appropriately consistent for
$\MCov$, while Assumption~\ref{ass:miscov-quantile-estimator-consistent}
basically ensures the pool of unlabeled samples is large enough to provide a
good estimate of the quantiles of $\what \MC_\batch$.
Assumption~\ref{ass:miscov-density} is technical and prevents
the random variable $\MCov(X_I)$ from being too concentrated,
thus allowing quantile estimation.
%

\begin{assumption}[Miscoverage estimation]
  \label{ass:miscov-estimator-consistent}
  For each batch $\batch \in [\nBatch]$,
  we have
  \begin{align*}
    \norm{\what {\MC}^{(\predsetthresh)}_\batch- {\MC}^{(\predsetthresh)}}_{L^\infty(P_{0,I})} = o_p(n^{-1/4}).
  \end{align*}
\end{assumption}

\begin{assumption}[Quantile estimation]
  \label{ass:miscov-quantile-estimator-consistent}
  For each $\batch \in [\nBatch]$ and every
  compact $K \subset \R_+$,
  the quantile estimator $\what \qfunc_\batch$ satisfies
  \begin{align*}
    \sup_{\rho \in K} \left| \what \qfunc_\batch(\rho) - \qfunc_0(\what \MC_\batch, \rho)\right| = o_p(n^{-1/4}).
  \end{align*}
\end{assumption}

\begin{assumption}
\label{ass:miscov-density}
The random variable $\MCov(X_I)$ has a bounded density $f_{\MC}$ on $[0,1]$.
\end{assumption}

Under these assumptions, we have the following
theorem, which shows that for every compact set $K \subset \R$, the sequence of
processes $\{ \sqrt{n}(\what{\SF}^{(\predsetthresh)}_n(\rho)-
\SFcov(\rho)) \}_{\rho \in K}$ converges in distribution in $L^\infty(K)$ to
a tight Gaussian process, whose covariance is tedious to specify
though we characterize it in the proof of the theorem in
Appendix~\ref{proof-thm-unif-conv-sens}.
\begin{theorem}
  \label{thm:unif-conv-sens}
  Let
  Assumptions~\ref{ass:miscov-estimator-consistent}--\ref{ass:miscov-density}
  hold. Then there exists a tight Gaussian process $\mathbb{G}$ such that,
  for every compact set $K \subset \R^+$, we have
  \begin{align*}
    \{ \sqrt{n}(\what{\SF}^{(\predsetthresh)}_n(\rho)- \SFcov(\rho)) \}_{\rho \in K}
    \cd
    \{ \mathbb{G}(\rho) \}_{\rho \in K}
    ~~~ \mbox{as~elements~of~~}L^\infty(K).
  \end{align*}
\end{theorem}

A few consequences of Theorem~\ref{thm:unif-conv-sens} are immediate.
First, we have $\sqrt{n}$-consistency:
\begin{align*}
  \sqrt{n} \cdot
  \sup_{\rho \in K}|\what{\SF}^{(\predsetthresh)}_n(\rho)- \SFcov(\rho))|
  \cd \sup_{\rho \in K} \mathbb{G}(\rho)
\end{align*}
as the supremum mapping is continuous in $L^\infty$, and so
$\linfs{\what{\SF}^{(t)}_n - \SFcov} = O_P(1/\sqrt{n})$. As another
immediate consequence, we see that under the assumptions of
Theorem~\ref{thm:unif-conv-sens},
for every $\rho >0$, there exists $\sigma^2(\rho) < \infty$ such that
\begin{align*}
  \sqrt{n}(\what{\SF}^{(\predsetthresh)}_n(\rho)- \SFcov(\rho)) \cd
   \normal(0,\sigma^2(\rho)).
\end{align*}
(This is similar to the result \cite[Thm.~1]{SubbaswamyAdSa21},
but as we note in footnote~\ref{footnote:hong-wrong},
the papers~\cite{SubbaswamyAdSa21,JeongNa20} may have technical
mistakes.)

\section{Further empirical results}

\subsection{UCI datasets}
\label{sec:uci-experiments}

We revisit the experiments in Section~\ref{sec:motivation-exp}, focusing on
evaluating our methodology for robust predictive inference.  Our goal here
is to show that when our estimate of the amount of shift is comparable to
the actual amount of shift, our methodology delivers coverage without
inflating prediction sets too much.  Accordingly, throughout these
experiments, we fix the desired robustness level $\rho = .01$, corresponding
(approximately) to the median chi-squared divergence between the natural and
tilted empirical distributions across the nine data sets and values of the
tilting parameter $a$.  We therefore expect
Algorithm~\ref{alg:rho-selection-procedure}, which emphasizes robustness to
worst-case shifts, to restore the coverage level for the tiltings from
Section~\ref{sec:motivation-exp} that possess (roughly) this level of shift.

Figure \ref{fig:cvgs_only_chisq} presents the results for the chi-squared
divergence (the results for the Kullback-Leibler divergence are similar).
Although not perfect, we see that the methodology often restores validity
for the shifts from Section~\ref{sec:motivation-exp}.  We see clearly
improved performance over the standard conformal methodology on the abalone,
delta ailerons, kinematics, puma, and airfoil datasets (compare to
Figure~\ref{fig:cvgs_only_std}).  On all of these datasets, the robust
methodology consistently yields average coverage above the nominal level,
while standard conformal fails to cover on each of the datasets.  Treating
the test sample as truth, we evaluate the median chi-squared divergence
between these natural and shifted distributions across values of the tilting
parameter $a$; the divergence values are .03, .02, .04, .05, and 3.65,
respectively, while the level of divergence for the remaining datasets
(ailerons---which still covers---banking, and Boston and California housing)
is roughly twice as large, which explains the loss in coverage.  We note in
passing that in other experiments we omit for brevity, the trends above hold
for other types of shifts.

\begin{figure}[h!]
    \centering
    \includegraphics[width=0.3\linewidth]{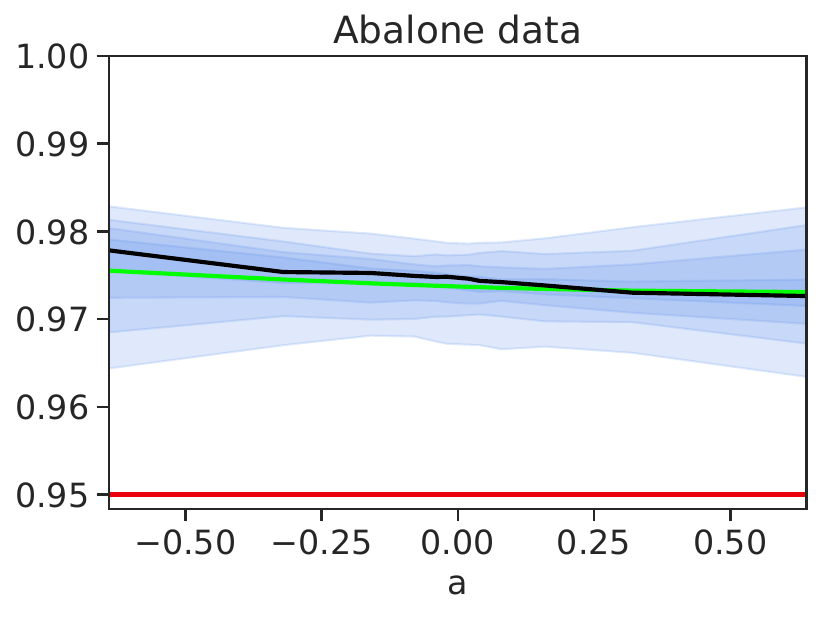}
    \hfill
    \includegraphics[width=0.3\linewidth]{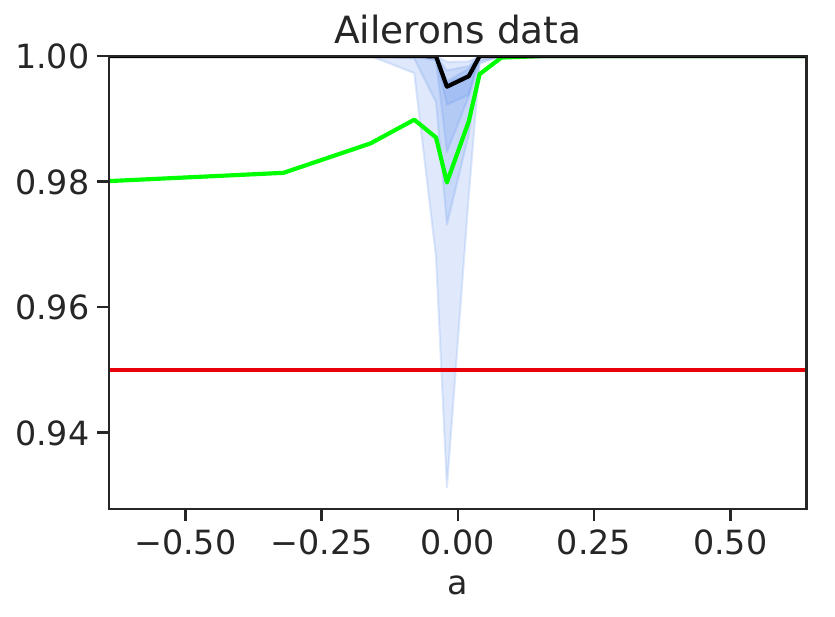}
    \hfill
    \includegraphics[width=0.3\linewidth]{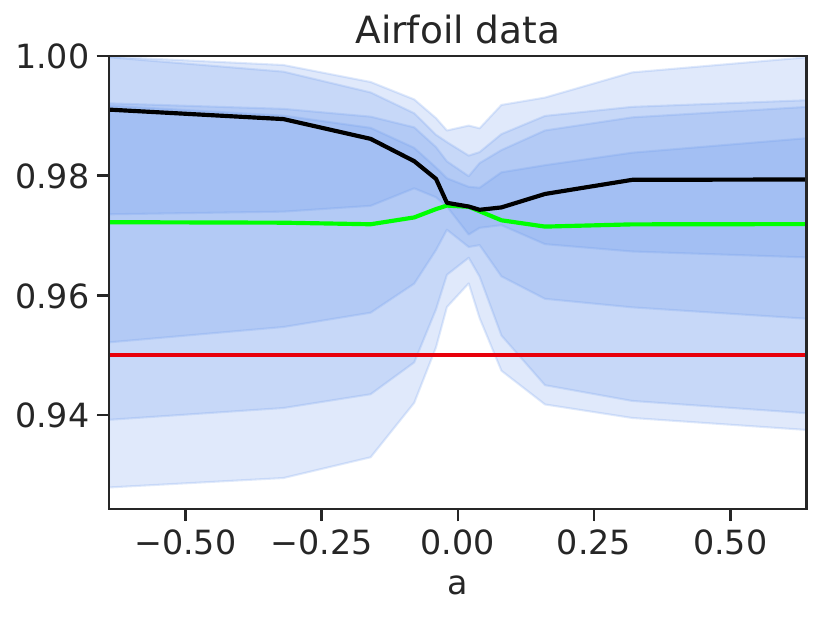}
    \\
    \includegraphics[width=0.3\linewidth]{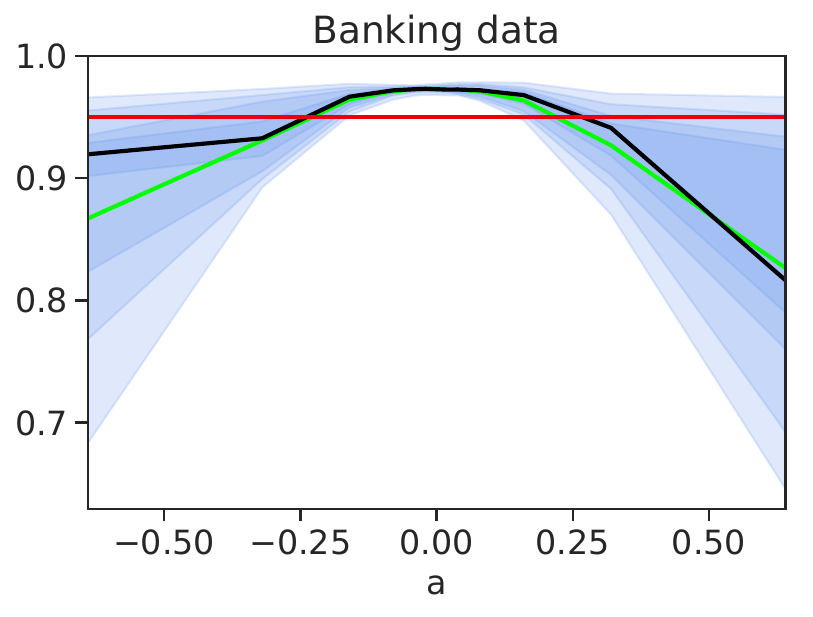}
    \hfill
    \includegraphics[width=0.3\linewidth]{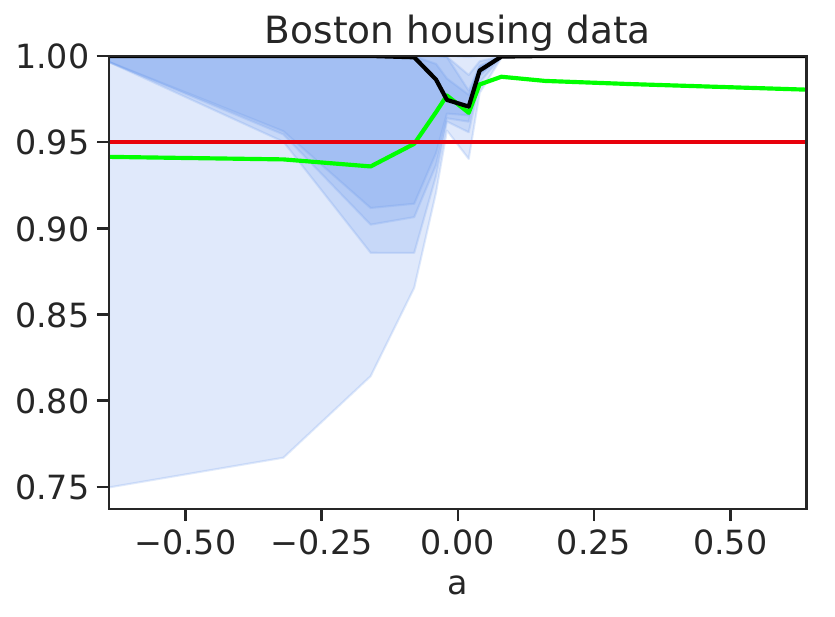}
    \hfill
    \includegraphics[width=0.3\linewidth]{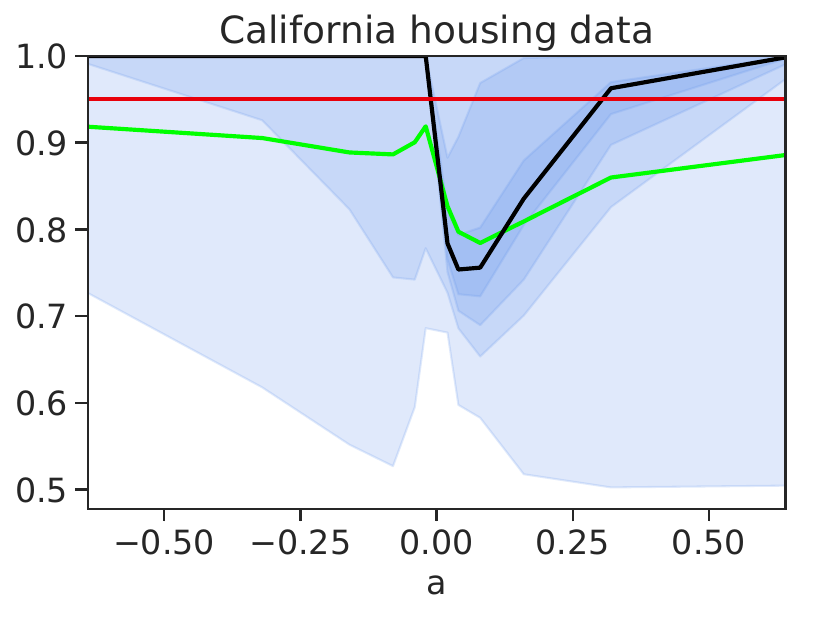}
    \\
    \includegraphics[width=0.3\linewidth]{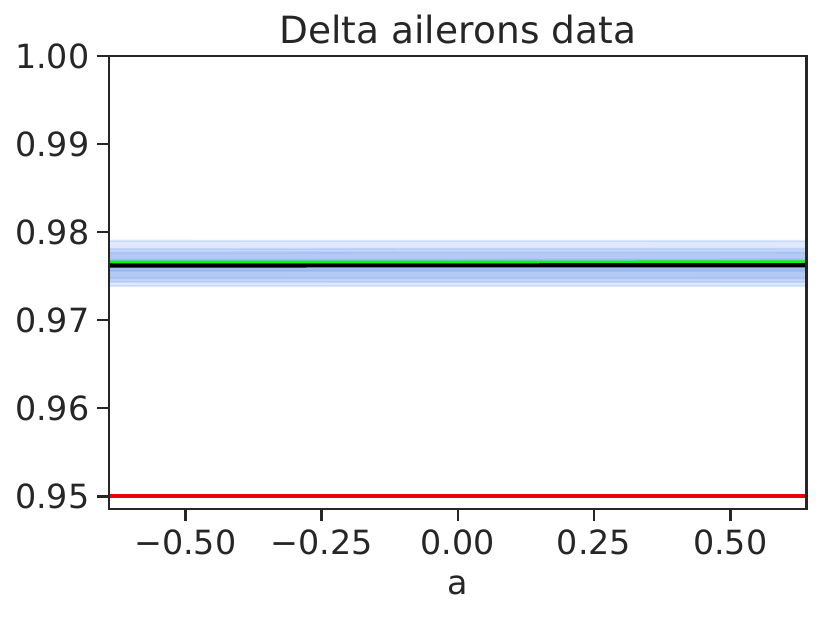}
    \hfill
    \includegraphics[width=0.3\linewidth]{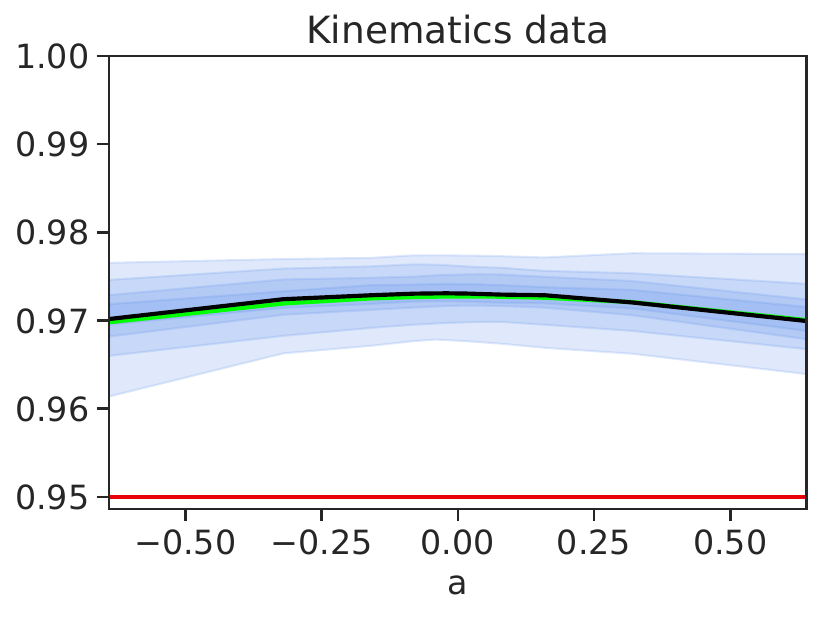}
    \hfill
    \includegraphics[width=0.3\linewidth]{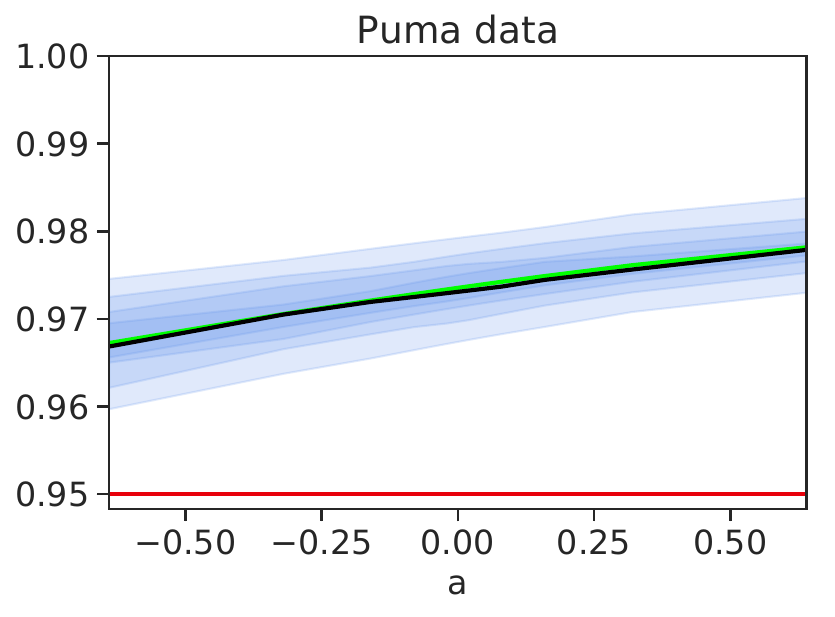}  
    \caption{Empirical coverage for the prediction sets generated by the
    chi-squared divergence, following the same experimental setup from Section
    \ref{sec:motivation-exp}. The horizontal axis gives the value of the tilting
    parameter $a$; the vertical the coverage level. A green line marks the
    average coverage, a black line marks the median coverage, and the horizontal
    red line marks the nominal coverage $.95$. The blue bands show the coverage
    at various deciles.}
    \label{fig:cvgs_only_chisq}
\end{figure}

\subsection{COVID-19 forecasting}
\label{sec:real-experiments-covid}

Our final evaluation of prediction accuracy under shifts is to predict test
positivity rates for COVID-19, in each of $L = 3{,}140$ United States
counties in a time series over $T = 34$ weeks from January through the
beginning of August 2021, using demographic features.  As a
non-stationary time series, robustness is essential here as a fixed model of
course cannot adapt to the underlying distributional changes.

Our prediction task is as follows.  For each of $t=1,\ldots,T$ weeks, and at
each of $\ell=1,\ldots,L$ locations (counties), we observe a
real-valued response $Y_{\ell, t} \in [0,1]$, $\ell=1,\ldots,L$,
$t=1,\ldots,T$, measuring the fraction of people with COVID-19.
We use data from the DELPHI group at Carnegie Mellon University
\citep{ArnoldBiBrCoFaGrMaReTi21,Tibshirani20} and consider a similar
featurization, using the following
trailing average features within each county: (1)
the number of COVID-19 cases per 100{,}000 people; (2) the number of doctor
visits for COVID-like symptoms; and (3) the number of responses
to a Facebook survey indicating respondents have COVID-like symptoms.  We
standardize both the features and responses so that they lie in $[0,1$], and
collect the features into vectors $X_{\ell, t} \in \R^3$, $\ell=1,\ldots,L$,
$t=1,\ldots,T$.

At each week $t = 1, 4, 7, \ldots$, 
we fit a simple logistic model where for a fixed $t$, we compute
\begin{equation}
  \label{eqn:logistic-goof}
  (\hat \alpha^{(t)}, \hat \beta^{(t)}) \in
  \argmin_{\alpha \in \R, \beta \in \R^d}
  ~
  \sum_{\ell = 1}^L
  \left[\log(1 + e^{\alpha + X_{\ell,t}^T \beta}) -
    Y_{\ell,t+1} (\alpha + X_{\ell,t}^T \beta)\right]
\end{equation}
We treat the data at the weeks $t=2,5,8,\ldots$ as the validation set, the
data at the remaining weeks $t=3,6,9,\ldots$ as the test set, so that at
each time $t$ we fit the single most recent time period's data. We make
predictions on a new example $x$ at time $t$ via the logistic link
\begin{equation*}
  \what{y} = \frac{e^{\hat{\alpha}^{(t)} + x^T \hat{\beta}^{(t)}}}{1 +
    e^{\hat{\alpha}^{(t)} + x^T \hat{\beta}^{(t)}}}.
\end{equation*}


For our robust conformalization procedures, we consider the Kullback-Leibler
divergence and estimate the divergence $\rho$ between weeks $1,4,7,\ldots$
and $2,5,8,\ldots$ via regression
(Alg.~\ref{alg:worst-direction-validation}), as well as with a nonparametric
divergence estimator~\citep{NguyenWaJo10}; given this $\rho$ we then make
robust predictions at the test times $t = 3, 6, \ldots$.  We compare to the
standard split conformal methodology---which is of course not robust to
departures from the validation distribution---but also consider the standard
conformal methodology with the more conservative miscoverage level
$\alpha/2$ to attain robustness to a variation distance shift of $\alpha/2$
(recall Corollary \ref{corollary:total-variation}).  We set $\alpha = .1$
throughout these experiments.


Figures \ref{fig:covid-coverage}--\ref{fig:covid-hi-lo} present the results.  From Figure \ref{fig:covid-coverage}, we can see that the standard conformal methodology (once again) fails to cover, whereas our (two) robust conformalization procedures retain validity.  These results are in line with our expectations: we expect the standard methodology to undercover as it is not robust to distributional changes, and we expect both Alg.~\ref{alg:worst-direction-validation} as well as the nonparametric divergence estimator of \citet{NguyenWaJo10} to deliver reasonably accurate estimates of the divergence level $\rho$ given the low ambient dimension of the feature space (recall that $d=3$), translating into generally good coverage here.  We can also see that the standard conformal methodology with the conservative miscoverage level $\alpha/2$ gives coverage at roughly the right level, though it is does not adapt the miscoverage level to the problem at hand (as estimating an appropriate level of divergence is an important component).  Along these lines, Figure \ref{fig:covid-size} reveals a more complete picture: the heuristic also gives rise to (slightly) longer confidence intervals than most of the other methods---which is intuitive as again we have no guarantee that $\alpha/2$ corresponds to the true amount of divergence between the validation and test distributions.  Overall, our robust conformalization procedure combined with Nguyen et al.'s  nonparametric divergence estimator~\citep{NguyenWaJo10} appears to strike the best balance between coverage and confidence interval length in this instance.

\begin{figure}[ht]
  \centering
  \includegraphics[width=0.9\linewidth]{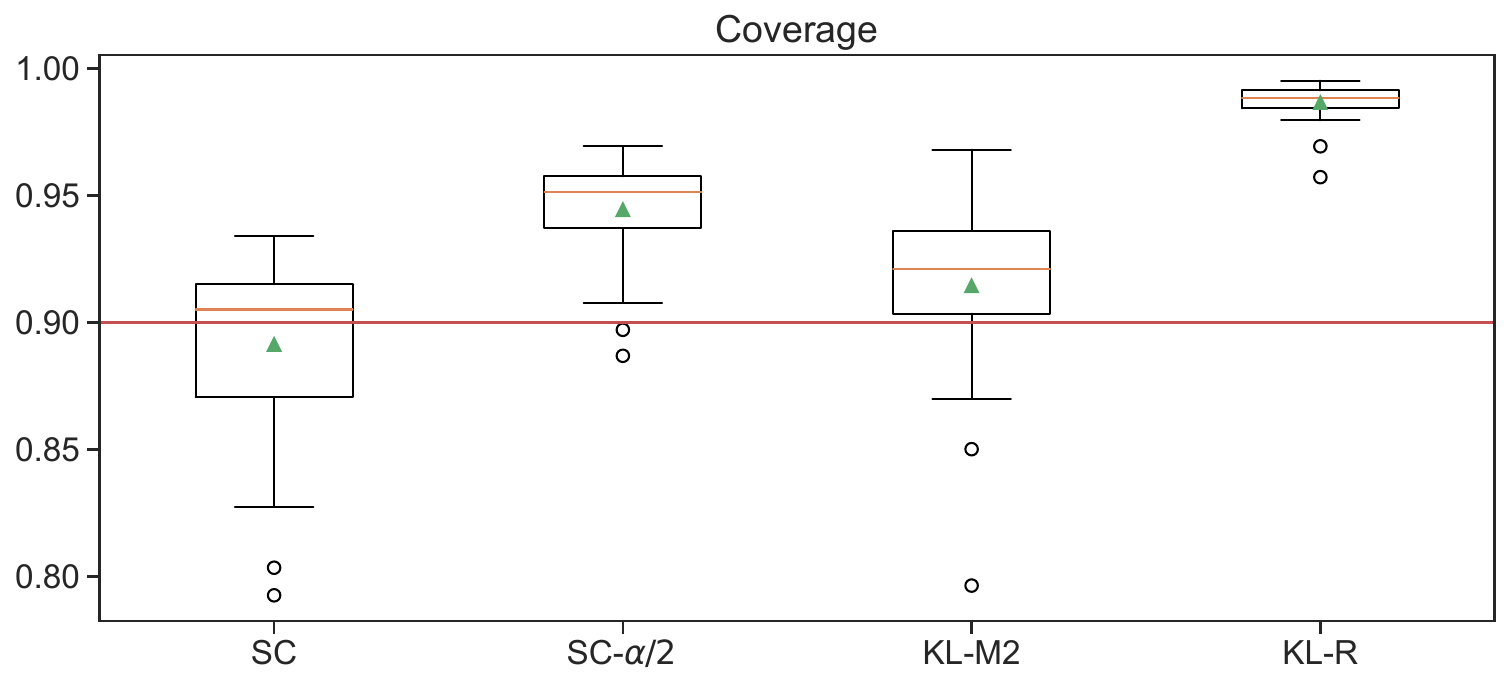}
  \caption{Empirical coverage for the prediction sets generated by the standard conformal methodology (``SC''), the standard conformal methodology where we simply set $\alpha/2$ (``SC-$\alpha/2$''), and the Kullback-Leibler divergence on the COVID-19 time series.  We set $\rho$ according to the regression-based strategy (``KL-R'') for estimating the amount of shift that we describe in Section~\ref{sec:coverage-high-probability-over-shifts}, as well as via the nonparametric divergence estimator due to \citet{NguyenWaJo10} (``KL-M2'').  The horizontal red line marks the marginal coverage $.9$.}
  \label{fig:covid-coverage}
\end{figure}

\begin{figure}[ht]
  \centering
  \includegraphics[width=0.9\linewidth]{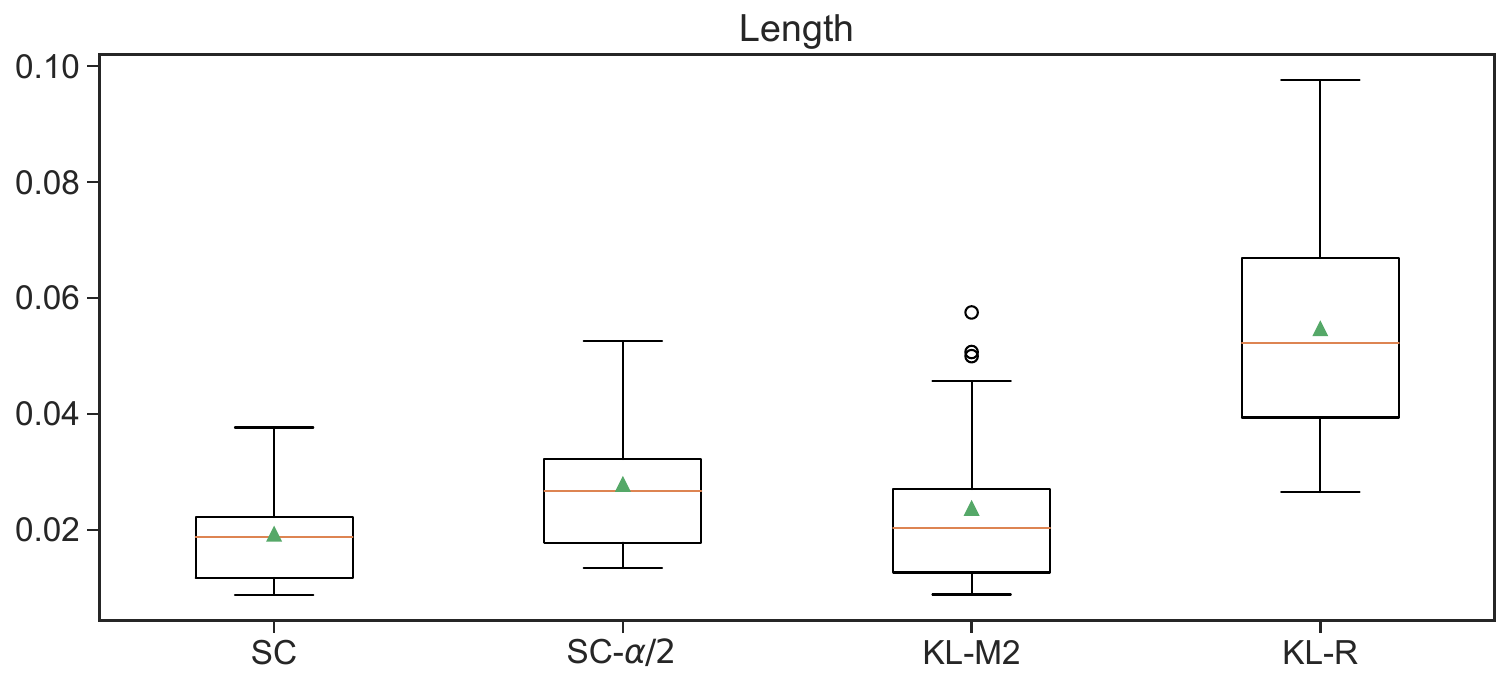}
  \caption{Average length for the prediction intervals generated by the standard conformal methodology (``SC''), the standard conformal methodology where we simply set $\alpha/2$ (``SC-$\alpha/2$''), and the Kullback-Leibler divergence on the COVID-19 time series.  We set $\rho$ according to the regression-based strategy (``KL-R'') for estimating the amount of shift, as well as via the nonparametric divergence estimator due to \citet{NguyenWaJo10} (``KL-M2'').}
  \label{fig:covid-size}
\end{figure}

We view these results from a more qualitative perspective in Figures \ref{fig:covid-true} and \ref{fig:covid-hi-lo}.  In Figure \ref{fig:covid-true}, we show the actual number of COVID-19 cases on April 16, 2021, when the state of Michigan saw a sudden spike in the incidence of COVID-19 after several weeks of implementing precautionary measures.  As an especially pronounced example of distributional shift, it is natural to ask whether our procedures might offer any kind of protection in this instance.  Figure \ref{fig:covid-hi-lo} shows the upper and lower endpoints of the confidence intervals that our robust conformalization procedure generates at this point in time.  By comparing the colors in the figures, we can see that our robust prediction intervals generally contain the true response value both across the United States as well as in Michigan, in particular---despite the presence of such a severe distributional shift.

\begin{figure}[ht!]
  \centering
  \includegraphics[width=0.9\linewidth]{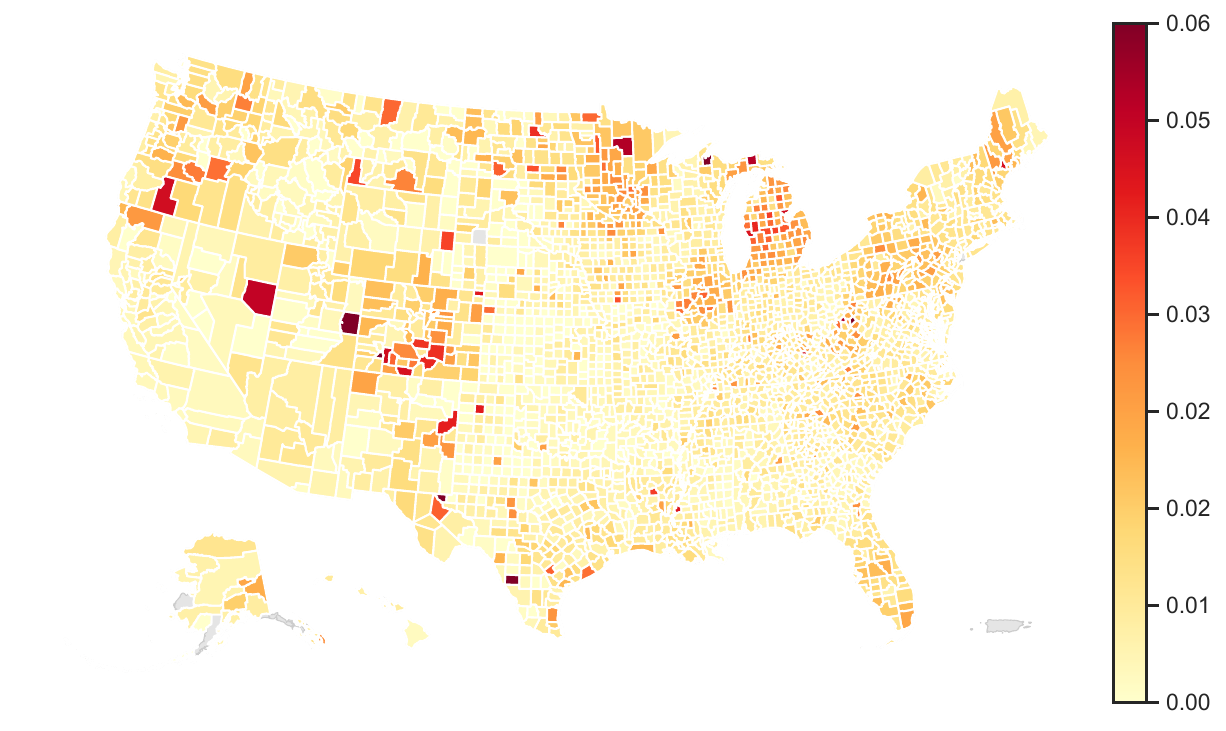}
  \caption{The true (normalized) number of COVID-19 cases per 100{,}000 people, smoothed over the previous week, across the United States on April 16, 2021.}
  \label{fig:covid-true}
\end{figure}

\clearpage
\begin{figure}[h!]
  \centering
  \includegraphics[width=0.9\linewidth]{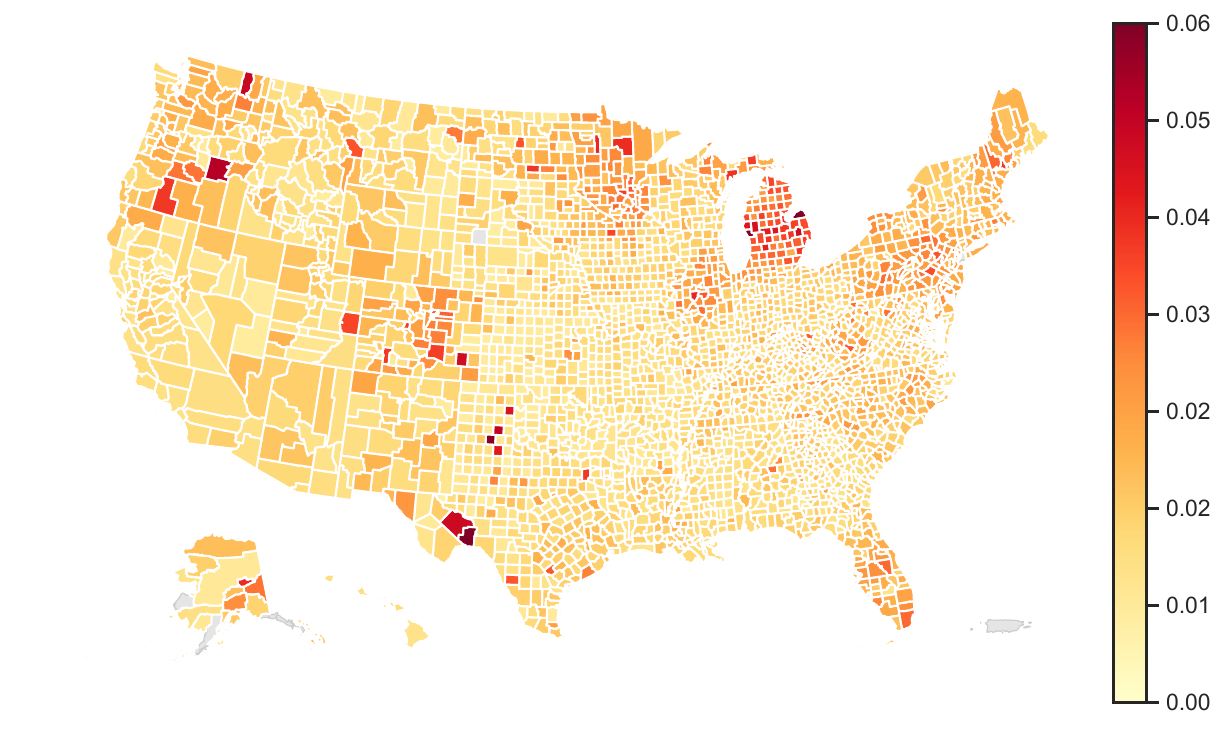} \\
  \includegraphics[width=0.9\linewidth]{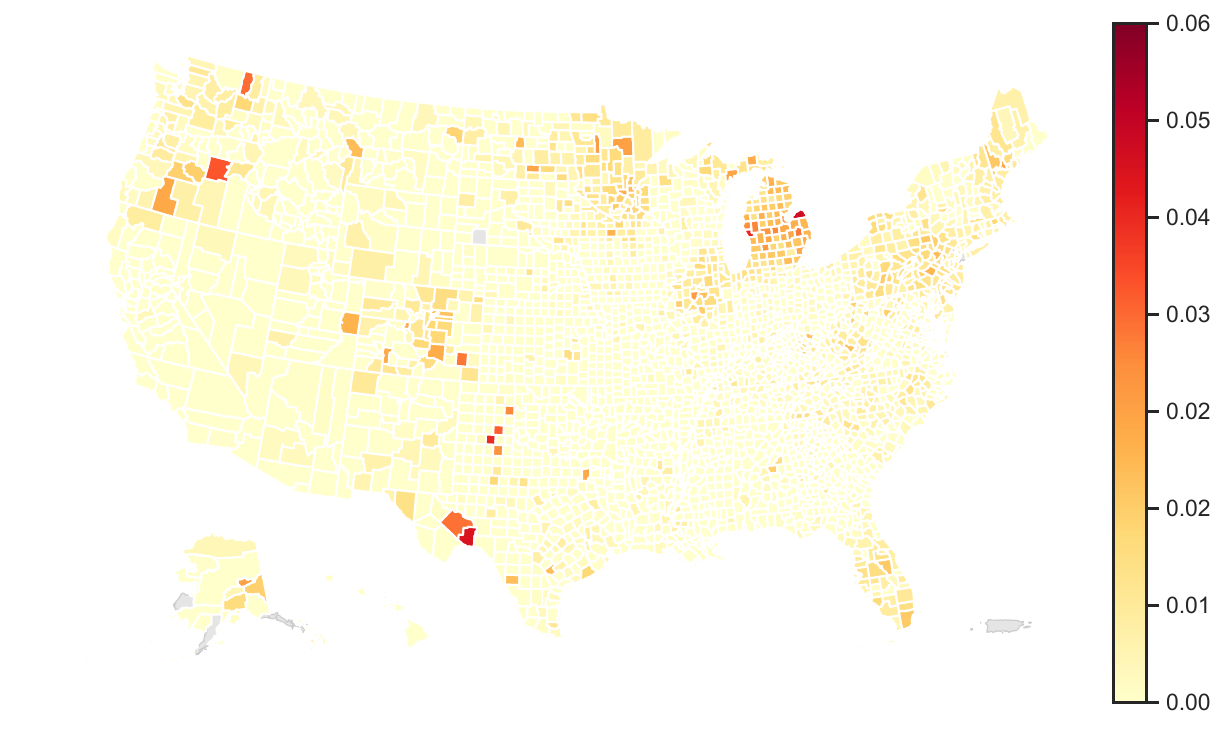}
  \caption{The upper (top panel) and lower (bottom panel) endpoints of the confidence intervals that our robust conformal methodology generates across the United States on April 16, 2021.}
  \label{fig:covid-hi-lo}
\end{figure}
\clearpage

\subsection{Experiments on covariate sensitivity}
\label{sec:covariate-sensitivity}

Our final experiment is to evaluate our sensitivity predictions for
covariate shift, as in Sec.~\ref{sec:sensitivity}. The point is twofold: we
(i) identify covariates for which covarage may be sensitive, then (ii) test
whether these putative sensitivities are indeed present in data. To do so,
we consider three datasets from the UCI repository \citep{DuaGr17}:
real-estate data, weather history data, and wine quality data.

We repeat the following experiment 25 times:
we randomly partition each
dataset into disjoint sets $D_\train, D_\val, D_{\text{sens}}, D_\test$ each
containing respectively $40\%, 10\%,30\%,10\%$ of the data, then fit a
linear regression model $\mu$ using $D_\train$ and construct conformal
intervals of the form~\eqref{eqn:confidence-set} with $\score(x, y) =
|\mu(x)- y|$, so that $\what{C}_n(x) = \{y \in \R \mid |\mu(x) - y| \le
\hat{t}\}$, setting the threshold $\hat{t}$ so that we achieve coverage at
nominal level $\alpha = .1$ on $D_\val$.  We estimate the sensitivity
function using $D_{\text{sens}}$ as in Algorithm~\ref{alg:sensitivity1}
for each singleton covariate (i.e.\ the covariate set $I = \{i\}$ for
each of $i = 1, 2, \ldots$),
where we estimate the conditional probabilities of miscoverage using default
tuning parameeters in R's version of random forests.

Figure~\ref{fig:sens-fig} shows the results. The plot is somewhat complex:
for each of the three datasets, we estimate sensitivity (as a function of
shift $\rho$) for each covariate in the dataset (e.g.\ \texttt{House age} in
the real estate data). Then for an individual covariate, we plot (estimated)
maximum miscoverage as a function of the radius $\rho$ of potential shift in
that covariate (the estimated sensitivity function~\eqref{eqn:covsens-fcn},
where $I = \{i\}$ is the covariate of interest); this is the red solid line
in each plot. As we are curious about coverage losses under covariate
shifts, we plot miscoverage (dashed lines) on the subset of the test data
$D_\test$ containing examples either from the upper or lower $e^{-\rho}$
quantiles of each covariate, which corresponds to R\'{e}nyi
$\infty$-divergence $\rho$, as in Lemma~\ref{lemma:cvar-calc}.  We expect
that these miscoverages to fall below the maximum miscoverage line,
which we observe across all three datasets.
Specifically, we see that for real estate data, coverage of the
corresponding confidence sets drops most when the marginal distribution of
the covariate ``House age'' shifts while that for weather history data, the
coverage drops most for shifts in the ``Pressure'' covariate.  For the wine
quality dataset, coverage seems almost equally sensitive to all covariates.
An interesting question for future work is to identify those directions
which \emph{are} sensitive---as opposed to the approach here, which
identifies potentially sensitive covariates.

\begin{figure}
  \centering
\begin{overpic}[
  				scale=0.28]{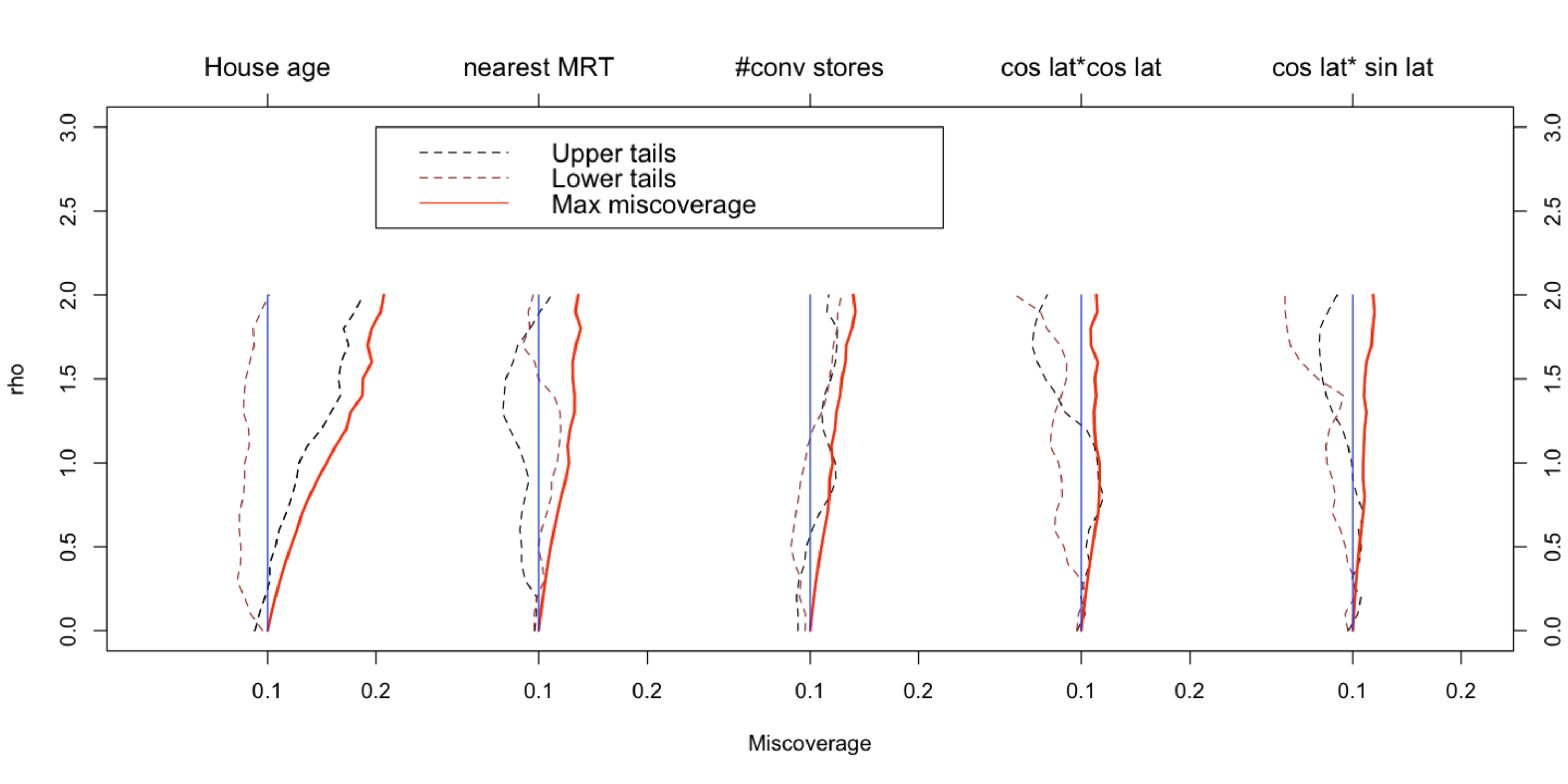}
  \put(40, -1){
        \small Real estate data}
 \put(-1,10){
      \tikz{\path[draw=white, fill=white] (0, 0) rectangle (.3cm, 6cm)}
    }
   \put(-1,25){\rotatebox{90}{
      $\rho$}
    }
  \end{overpic}

    \centering
  \begin{overpic}[
  				scale=0.28]{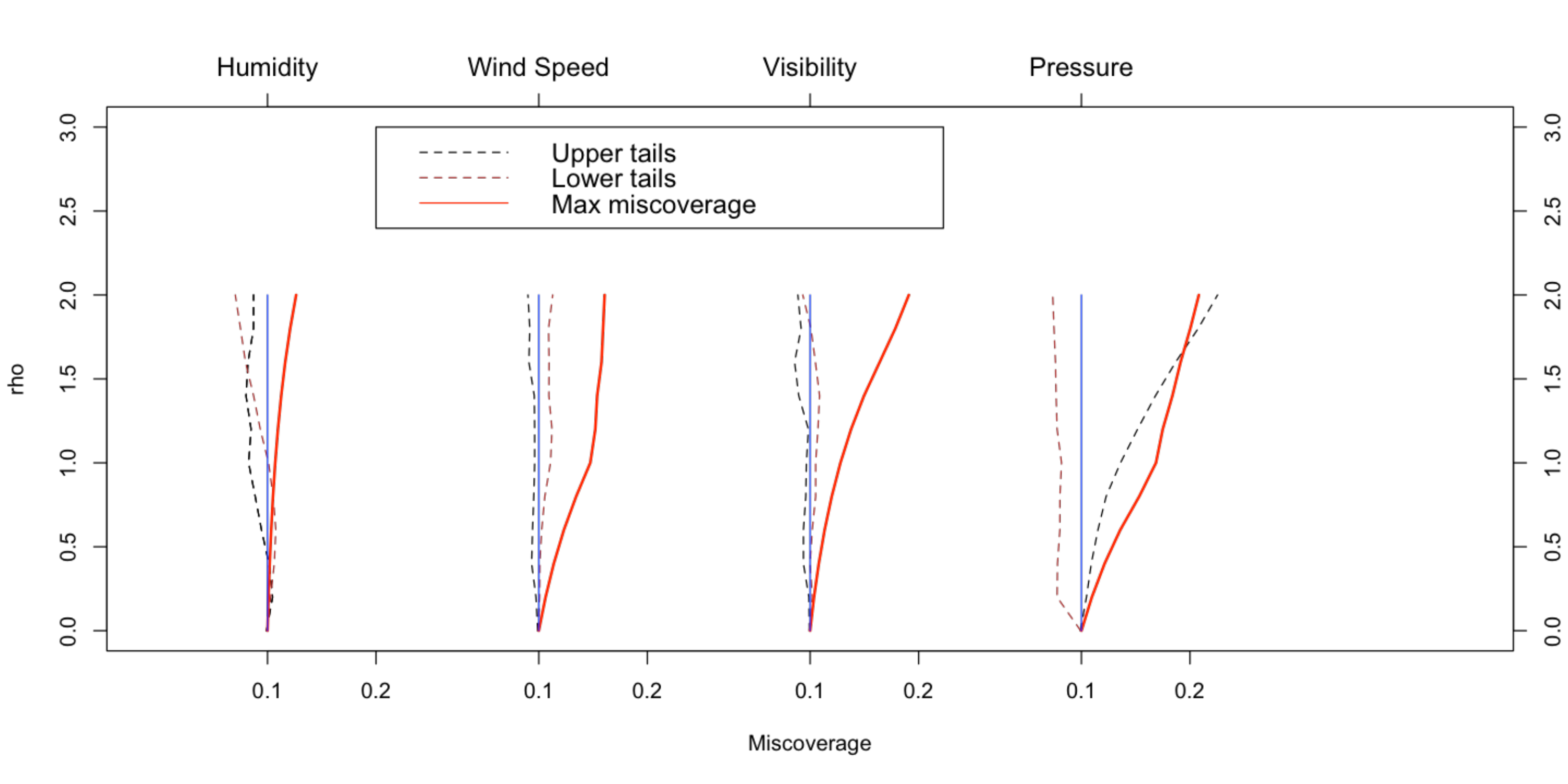}
  \put(40, -1){
        \small Weather history data }
 \put(-1,10){
      \tikz{\path[draw=white, fill=white] (0, 0) rectangle (.3cm, 6cm)}
    }
   \put(-1,20){\rotatebox{90}{
      $\rho$}
    }
  \end{overpic}

    \centering
    \begin{overpic}[
  				scale=0.28]{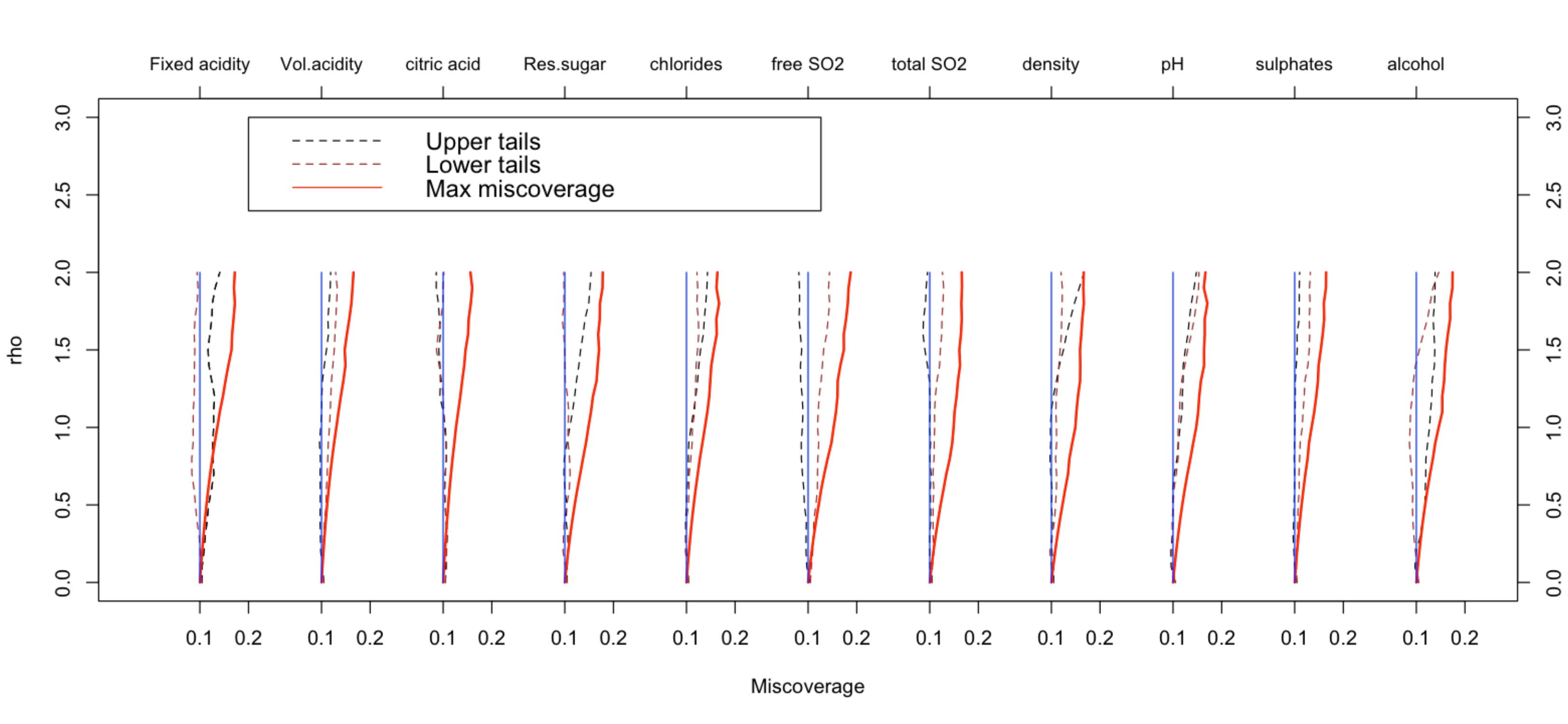}
  \put(40, -1){
        \small Wine quality data }
 \put(-1,10){
      \tikz{\path[draw=white, fill=white] (0, 0) rectangle (.3cm, 6cm)}
    }
   \put(-1,20){\rotatebox{90}{
      $\rho$}
    }
  \end{overpic}

  \vspace{0.2in}
 
  \caption{Sensitivity of (mis)-coverage for three datasets. Red line shows maximum miscoverage possible within a given shift in marginal distribution of a covariate with respect to limiting $f$-divergences. Dashed lines show miscoverage on a subset of test data that contains samples for which the corresponding covariate takes values in the upper or lower $e^{-\rho}$ quantiles of that covariate.}
  \label{fig:sens-fig}
\end{figure}

\section{Proofs of results on robust inference}

\subsection{Proof of Proposition~\ref{proposition:rho-vs-alpha}}
\label{sec:proof-rho-vs-alpha}

We provide several properties of $g_{f,\rho}(\beta) = \inf\{z \in [0, 1] :
\beta f(\frac{z}{\beta}) + (1 - \beta) f(\frac{1-z}{1 - \beta}) \le \rho\}$,
deferring their proof to Sec.~\ref{sec:proof-properties-robust-cdf}.
\begin{lemma}[Properties of $g_{f,\rho}$]
  \label{lemma:properties-robust-cdf}
  Let $f$ be a closed convex function such that $f(1) = 0$ and $f(t) < \infty$ for all $t > 0$. Then the function
  $g_{f,\rho}$ satisfies the following.
  \begin{enumerate}[(a)]
  \item $(\beta, \rho) \mapsto g_{f,\rho}(\beta)$ is a convex
    function.
  \item \label{item:gf-rho-monotonicity}
    $g_{f,\rho}$ is non-increasing in $\rho$ and non-decreasing in
    $\beta$. Moreover, for all $\rho >0$, there exists $\beta_0(\rho) \defeq
    \sup \{\beta \in (0,1) \mid g_{f,\rho}(\beta) = 0 \}$, and
    $g_{f,\rho}$ is strictly increasing for $\beta > \beta_0(\rho)$.
  \item \label{item:gf-rho-continuity}
    $(\beta, \rho) \mapsto g_{f,\rho}(\beta)$ is continuous for $\beta
    \in [0, 1]$ and $\rho \in (0,\infty)$.
  \item For $\beta \in [0,1]$ and $\rho > 0 $, $g_{f,\rho}(\beta) \le
    \beta$.  For $\rho > 0$, equality holds for $\beta = 0$,
    strict inequality holds for $\beta \in (0,1)$ and $\rho > 0$,
   and $g_{f,\rho}(1)=1$ if and only if $f'(\infty) = \infty$.
  \item Let $g_{f,\rho}^{-1}(t) = \sup\{\beta : g_{f,\rho}(\beta) \le t\}$
    as in the statement of Proposition~\ref{proposition:rho-vs-alpha}.
    Then for $\beta \in (0, 1)$, $g_{f,\rho}(\tau) \ge \beta$ if and only
    if $g_{f,\rho}^{-1}(\beta) \le \tau$.
  \end{enumerate}
\end{lemma}

We now define the worst-case cumulative distribution function, which
generalizes the c.d.f.\ of a distribution in the same way the worst-case
quantile generalizes standard quantiles.
\begin{definition}[$f$-worst-case c.d.f.]
  Let $\rho > 0$ and consider any distribution $P$ on the real line.
  The \emph{$(f,\rho)$-worst-case cumulative distribution function} is
  \begin{align}
    \label{eqn:robust-f-div-cdf-def}
    \cdfrob_{f,\rho}(t; \, P) \defeq 
    \inf \big\{ P_1(S \le t) \mid S \sim P_1, \; \fdivs{P_1}{P}
    \le \rho \big\}.
  \end{align}
\end{definition}

Proposition~\ref{proposition:rho-vs-alpha} will then follow from the coming
lemma.
\begin{lemma}
  \label{lemma:effective-cdf}
  Let $P$ be a distribution on $\R$ with c.d.f.\ $F$.
  Then
  \begin{align}
    \cdfrob_{f,\rho}(t; \, P) = g_{f,\rho}(F(t)).
  \end{align}
\end{lemma}
\noindent Deferring the proof of this lemma as well (see
Sec.~\ref{sec:proof-effective-cdf}), let us see how it implies
Proposition~\ref{proposition:rho-vs-alpha}. Observe that for all
$\beta \in (0,1)$, and any real distribution $P$ with c.d.f.\ $F$, we have
\begin{align*}
  \WCQuantile_{f,\rho}(\beta; \, P) 
  &=  
  \inf \big\{ q \in \R \mid \cdfrob_{f,\rho}(q, P) \geq \beta \big\} \\
  & \stackrel{(i)}{=}
  \inf \big\{ q \in \R \mid g_{f,\rho}(F(q)) \ge \beta \big\} \\
  & \stackrel{(ii)}{=} \inf
  \big\{ q \in \R \mid F(q) \ge g_{f,\rho}^{-1}(\beta) \big\}
  = \Quantile(g_{f,\rho}^{-1}(\beta); \, P),
\end{align*}
where equality~$(i)$ uses Lemma~\ref{lemma:effective-cdf} and~$(ii)$
follows because by Lemma~\ref{lemma:properties-robust-cdf}, as
$g_{f,\rho}(\tau) \ge \beta$ if and only if $g_{f,\rho}^{-1}(\beta) \le
\tau$.

\subsubsection{Proof of Lemma~\ref{lemma:properties-robust-cdf}}
\label{sec:proof-properties-robust-cdf}
\providecommand{\persp}{_{\textup{per}}}

It is no loss of generality to assume that $f'(1) = 0$ and $f \ge 0$,
as replacing $f$ by $f_0(t) \defeq f(t) -f'(1)(t - 1)$ generates the same
$f$-divergence and evidently $\inf_t f_0(t) = f_0(1) = 0$.

\begin{enumerate}[(a)]
\item Let $f\persp(t, \beta) = \beta f(t / \beta)$ be the perspective
  transform of $f$, which is convex, with the understanding that
  \begin{itemize}
  \item $f\persp(0, \beta) = f(0) = f(0^+)$ for $\beta >0$,
  \item $f\persp(0, 0) = 0 f(0/0) = 0$,
  \item $f\persp(t, \beta) = 0 f(t/0) = t f'(\infty)$ for all $t>0$, where $f'(\infty) = \lim_{a \to \infty} f'(a) \in (0,\infty]$.
  \end{itemize} 
  Then $g_{f,\rho}(\beta)$ is the partial
  minimization of the convex function $(\rho, \beta, z) \mapsto z +
  \mathbf{I}(f\persp(z, \beta) + f\persp(1-z, 1 - \beta) \le \rho)$ and hence
  convex, where $\mathbf{I}(\cdot)$ is the convex indicator function,
  $+\infty$ if its argument is false
  and $0$ otherwise.
  (See \cite[Ch.~IV]{HiriartUrrutyLe93} for proofs of each of
  these claims and that the limits indeed exist.)
	
\item That $\rho \mapsto g_{f,\rho}(\beta)$ is non-increasing is evident.
  As $g$ is nonnegative, convex, and $g_{f,\rho}(0) = 0$, it must therefore
  be non-decreasing.  That $g_{f,\rho}(\beta) > 0$ is strictly increasing in
  $\beta > \beta_0(\rho)$ is again immediate by convexity as $g_{f,\rho}(0)
  = 0$.
  
  \item Any convex function is continuous on the interior of its domain,
    thus $g$ is continuous on $(0,1) \times (0,\infty)$.  To see that
    $g_{f,\rho}$ is continuous from the left at $\beta = 1$, first observe
    that $\beta \mapsto g_{f,\rho}(\beta)$ is non-decreasing by
    \eqref{item:gf-rho-monotonicity} (which
    only uses convexity and the fact that $g_{f,\rho}(0) = 0 f(0/0) = 0$),
    so we only need to prove that
 \begin{align*}
 \limsup_{\beta \uparrow 1} g_{f,  \rho}(\beta) \ge g_{f,\rho}(1) &= \inf \left\{ z\in [0,1]: f(z) + f'(\infty)(1-z) \le \rho \right\} \\  &= \sup \left\{ z \in (0,1) : f(z) + f'(\infty)(1-z) > \rho \right\},
 \end{align*}
 where the last equality follows from the fact that $z \mapsto f(z) + f'(\infty)(1-z)$ is decreasing on $[0,1]$.
However,  for any $z \in (0,1)$ such that $f(z) + f'(\infty)(1-z) > \rho$,  the continuity of $f$ in $z$ and the fact that $t f((1-z)/t) \underset{t \to 0}{\to} f'(\infty) (1-z)$ ensure the existence of $\beta_0 \in [z,1)$ such that $\beta_0 f(z / \beta_0) + (1-\beta_0) f((1-z)/(1-\beta_0)) > \rho$. 
Since $\tilde{z} \mapsto f\persp(\tilde{z}, \beta_0) + f\persp(1-\tilde{z}, 1-\beta_0)$ is non-increasing on $[0,\beta_0]$, this implies that $g_{f,\rho}(\beta_0) \ge z$, hence that $\limsup_{\beta \to 1} g_{f,\rho}(\beta) \ge z$, which concludes the proof.

  That $g_{f,\rho}$ is right continuous at $\beta = 0$ is immediate
  because $g_{f,\rho}$ is non-decreasing and convex.

\item
  The non-strict inequality is immediate by considering $z = \beta$ and
  using that $f(1) = 0$. The strict inequality is immediate
  because $f$ is continuous near $1$,  the equality for $\beta=0$ is trivial since $0 \le g_{f,\rho}(\beta) \le \beta$, and $g_{f,\rho}(1) = \inf \left\{ z\in [0,1]: f(z) + f'(\infty)(1-z) \le \rho \right\}$ equals $1$ if and only if $f'(\infty) = \infty$.
  
\item
  Let $g = g_{f,\rho}$ for shorthand.
  Suppose that $g(\tau) \ge \beta > 0$. Then
  as $g$ is strictly increasing when it is positive, we have
  $g(t) > g(\tau) \ge \beta$ for all $t > \tau$, so that
  $g^{-1}(\beta) \le t$ for any $t > \tau$, or $g^{-1}(\beta) \le \tau$.

  Now, assume the converse, that is, that $g^{-1}(\beta) \le \tau$, and
  assume for the sake of contradiction that $g(\tau) < \beta$.  By
  part~\eqref{item:gf-rho-monotonicity}, we must therefore have $\tau < 1$.
  As $g$ is continuous by part~\eqref{item:gf-rho-continuity}, we have
  $g(\tau + \epsilon) \le \beta$ for all sufficiently small $\epsilon > 0$,
  contradicting that $g^{-1}(\beta) \le \tau$.  Thus we must have $g(\tau)
  \ge \beta$.
\end{enumerate}

\subsubsection{Proof of Lemma~\ref{lemma:effective-cdf}}
\label{sec:proof-effective-cdf}

Recall that $P$ is a real distribution with c.d.f.\ $F$.  
We treat the cases $F(t) = 0$, $F(t) \in (0,1)$ and $F(t) = 1$ separately.

\begin{itemize}
\item If $F(t) = 0$, the result is immediate, since we have $0 \le \cdfrob_{f,\rho}(t; P) \le F(t)$.
\item Suppose now that $0 < F(t)
= P( S \le t) < 1$.
The inequality $\cdfrob_{f,\rho}(t; \, P) \le g_{f,\rho}(F(t))$ is
immediate:
\begin{align*}
  \lefteqn{
    \inf\left\{P_1(S \le t) \mid \fdiv{P_1}{P} \le \rho \right\}
  } \\
  & \le
  \inf\left\{P_1(S \le t) \mid \fdiv{P_1}{P} \le \rho,
  ~ \frac{dP_1}{dP} ~ \mbox{is~constant~on~}
  \{S \le t\} ~ \mbox{and} ~ \{S > t\} \right\}.
\end{align*}
The reverse inequality is a consequence of the data processing
inequality~\cite{LieseVa06}. Fix $t \in \R$. Let $P_1$
be a distribution satisfying $\fdiv{P_1}{P} \le \rho$. We show how
to construct $\tilde{P}$ with $\fdivs{\tilde{P}}{P}
\le \fdivs{P_1}{P}$ and $\tilde{P}(S \le t) = P_1(S \le t)$.
Indeed, define the Markov kernel $K$ by
\begin{align*}
  K \left(ds^\prime \mid s \right) \propto              
  \left\{
  \begin{array}{lr}
    dP(s^\prime) \indic{s^\prime \le t}, &\text{if } s \le t\\
    dP(s^\prime) \indic{s^\prime > t},   &\text{if } s > t. \\
  \end{array}\right.
\end{align*}
Then $P = K \cdot P$, while $\tilde{P} \defeq K \cdot P_1$
satisfies
\begin{align*}
  \fdivs{\tilde{P}}{P}
  = \fdivs{K \cdot P_1}{K \cdot P}
  \le \fdivs{P_1}{P} \le \rho
\end{align*}
by the data processing inequality. Now we observe that
\begin{align*}
  d\tilde{P}(s) = \left(
  \frac{P_1(S \le t)}{P(S \le t)}\indic{S \le t}
  +  \frac{P_1(S > t)}{P(S > t)}\indic{S > t} 
  \right) dP(s).
\end{align*}
By construction, $\tilde{P}(S \le t) = P_1(S \le t)$,
and it is immediate that
\begin{align*}
  \fdivs{\tilde{P}}{P}
  = P(S\le t)f\left(\frac{P_1(S \le t)}{P(S \le t)}\right)
  + P(S> t)f\left(\frac{P_1(S > t)}{P(S > t)}\right).
\end{align*}
Matching the expression of $\fdivs{\tilde{P}}{P}$ to the definition of
$g_{f,\rho}$ gives $g_{f,\rho}(F(t)) \le P_1(S \le t)$.  Taking the infimum
over all possible distributions $P_1$ concludes the proof.

\item Finally, if $F(t) = P(S \le t) = 1$, we have $\cdfrob_{f,\rho}(t; \, P) \le g_{f,\rho}(1)$ since for any $z \in (g_{f,\rho}(1),1]$, the distribution $P_{z,1} \defeq (1-z) \delta_{t+1} + z P$ satisfies $\fdiv{P_{z,1}}{P} \le \rho$ and $P_{z,1}(S\le t) = z$.
The proof of the other inequality is similar to the case where $F(t) \in (0,1)$, except a valid Markov kernel $K$ is now
\begin{align*}
  K \left(ds^\prime \mid s \right) \propto              
  \left\{
  \begin{array}{lr}
    dP(s^\prime) \indic{s^\prime \le t}, &\text{if } s \le t\\
    \delta_{s^\prime = t+1},   &\text{if } s > t,\\
  \end{array}\right.
\end{align*}
to account for the fact that $P(S>t) = 0$.
\end{itemize}


\subsection{Proof of Proposition~\ref{proposition:cvg-only-test}}
\label{sec:proof-cvg-only-test}

Since $\rho\opt = \fdivs{P_\textup{test}}{P_0} < \infty$, the definition of
$\cdfrob_{f,\rho}$ and Lemma~\ref{lemma:effective-cdf}
imply that for all $q \in \R$,
\begin{align*}
  F_\textup{test}(q) \ge \cdfrob_{f,\rho^\star}(q, P_0) = g_{f,\rho^\star}(F_0(q)).
\end{align*}
Applying this inequality with $q \defeq \WCQuantile_{f,\rho}(1-\alpha; \hat
P_n) = \Quantile(g_{f,\rho}^{-1}(1-\alpha); \hat P_n)$, we obtain
\begin{align*}
  \P \Big( Y_{n+1} \in \hat C_{n,f, \rho}(X_{n+1}) \mid \{(X_i,Y_i)\}_{i=1}^n
  \Big) & \stackrel{(i)}{=} F_\textup{test}( \WCQuantile_{f,\rho}(1-\alpha; \hat
  P_n) ) \\ &\ge g_{f,\rho^\star}(F_0(\WCQuantile_{f,\rho}(1-\alpha; \hat
  P_n))) \\ &\stackrel{(ii)}{=}
  g_{f,\rho^\star}(F_0(\Quantile(g_{f,\rho}^{-1}(1-\alpha); \hat P_n))),
\end{align*}
where equality~$(i)$ uses that $\score(X_{n+1},Y_{n+1}) \sim \Ptest$ is
independent of $\{(X_i, Y_i) \}_{i=1}^n$
and~$(ii)$ is Proposition~\ref{proposition:rho-vs-alpha}.

\subsection{Proof of Theorem~\ref{theorem:robust-coverage-marginal}}
\label{sec:proof-robust-coverage-marginal}

We require the following lemma to prove the theorem.
\begin{lemma}[Quantile coverage~\cite{VovkGaSh05,
      LeiGSRiTiWa18, BarberCaRaTi19a}]
  \label{lemma:conformal-inference-coverage-cdf}
  Assume that $\{ \scorerv_i \}_{i=1}^n \simiid P_0$ with c.d.f.\ $F_0$, and
  let $\hat P_n$ be their empirical distribution.  Then for all $\beta \in
  (0,1)$,
  \begin{align*}
    \E \left[ F_0\left( \Quantile(\beta; \, \hat P_n) \right) \right] \ge \frac{\ceil{n\beta}}{n+1}.
  \end{align*}
\end{lemma}

We include the brief proof of
Lemma~\ref{lemma:conformal-inference-coverage-cdf} below for completeness,
giving the proof of Theorem~\ref{theorem:robust-coverage-marginal} here.
By Proposition~\ref{proposition:cvg-only-test}, for
$\rho\opt = \fdivs{P_{\textup{test}}}{P_0} < \infty$,
we have
\begin{align*}
  \P \Big( Y_{n+1} \in \hat C_{n,f, \rho}(X_{n+1}) \mid  \{(X_i,Y_i)\}_{i=1}^n \Big)  
  \ge g_{f,\rho\opt}(F_0(\Quantile(g_{f,\rho}^{-1}(1-\alpha); \hat P_n))).
\end{align*}
Marginalizing over $(X_i, Y_i)$,
this implies that
\begin{align*}
  \P \Big( Y_{n+1} \in \hat C_{n,f, \rho}(X_{n+1}) \Big)  
  & \ge \E \left[ g_{f,\rho\opt}(F_0(\Quantile(g_{f,\rho}^{-1}(1-\alpha); \hat P_n))
    \right] \\
  & \stackrel{(i)}{\ge}  g_{f,\rho\opt}\left(
  \E\left[ F_0(\Quantile(g_{f,\rho}^{-1}(1-\alpha); \hat P_n) \right] \right) \\
  & \stackrel{(ii)}{\ge}
  g_{f,\rho\opt} \left(\frac{\ceil{n g_{f,\rho}^{-1}(1-\alpha)}}{n+1}\right),
\end{align*}
where inequality $(i)$ is a consequence of Jensen's inequality applied to
$g_{f,\rho\opt}$ (recall Lemma~\ref{lemma:properties-robust-cdf}(a)),
while inequality $(ii)$ uses
Lemma~\ref{lemma:conformal-inference-coverage-cdf} and that $\beta
\mapsto g_{f,\rho}(\beta)$ is non-decreasing.

\begin{proof-of-lemma}[\ref{lemma:conformal-inference-coverage-cdf}]
  Let $ S_{n+1} \sim P_0$ independent of $\{ S_i \}_{i=1}^n$. Then
  \begin{align*}
    \E \left[ F_0\left( \Quantile(\beta; \, P_n) \right) \right] 
    &= \P\left( S_{n+1} \le  \Quantile(\beta; \, P_n) \right) \\
    &\ge \P( \text{Rank of } S_{n+1} \text{ in } \{S_i\}_{i=1}^{n+1} \le \ceil{n \beta})
    = \frac{\ceil{n \beta}}{n+1},
  \end{align*}
  where we break ties uniformly at random to define the rank of $S_{n+1}$ in
  $\{S_i\}_{i=1}^{n+1}$, ensuring by exchangeability that it is uniform on
  $\{1,\ldots, n+1\}$.
\end{proof-of-lemma}


\subsection{Proof of Corollaries~\ref{corollary:almost-alpha-coverage}
  and~\ref{corollary:corrected-alpha-coverage}}
\label{sec:proof-alpha-coverages}

When $\rho\opt = \fdivs{\Ptest}{P_0} \ge \rho$,
Lemma~\ref{lemma:properties-robust-cdf} guarantees that $g_{f,\rho} \ge
g_{f,\rho\opt}$, so Theorem~\ref{theorem:robust-coverage-marginal}
gives
\begin{equation}
  \label{eqn:no-burritos-in-portland}
  \P(Y_{ n +1} \in \hat{C}_{n,f,\rho})
  \ge
  g_{f,\rho}\left(\frac{\ceil{n g_{f,\rho}^{-1}(1 - \alpha)}}{n+1}\right).
\end{equation}
To prove Corollary~\ref{corollary:almost-alpha-coverage},
note that as $g_{f,\rho}$ in convex, it has (at least) a left
derivative $g'_{f,\rho}$, which satisfies
\begin{align*}
g_{f,\rho}\Biggr( \frac{\ceil{n g_{f,\rho}^{-1}(1-\alpha)}}{n+1}\Biggr) \ge g_{f,\rho}\Biggr( \frac{n g_{f,\rho}^{-1}(1-\alpha)}{n+1}\Biggr) \ge 1 - \alpha - \frac{g_{f,\rho}^{-1}(1-\alpha) g'_{f,\rho} (g_{f,\rho}^{-1}(1-\alpha))}{n+1}.
\end{align*}
This gives the first corollary.

For the second corollary, replacing $\hat{C}$ in
Eq.~\eqref{eqn:no-burritos-in-portland}
with $\hat{C}^{\textup{corr}}$ gives
\begin{align*}
  \P(Y_{ n +1} \in \hat{C}^{\textup{corr}}_{n,f,\rho})
  & \ge
  g_{f,\rho}\left(
  \frac{\ceil{n g_{f,\rho}^{-1}
      \left(g_{f,\rho}\left((1 + 1/n) g_{f,\rho}^{-1}(1 - \alpha)\right)\right)}
  }{n+1}\right) \\
  & = g_{f,\rho}\left(\frac{\ceil{n(1 + 1/n) g_{f,\rho}^{-1}(1 - \alpha)}}{
    n + 1}\right)
  \ge g_{f,\rho}\left(g_{f,\rho}^{-1}(1 - \alpha)\right)
  \ge 1 - \alpha.
\end{align*}


\section{Proof of Theorem~\ref{theorem:high-probability-coverage}}
\label{sec:proof-high-probability-coverage}

Throughout the proof, we will typically not assume that the scores
$\score(X_i, Y_i)$ are distinct, and thus will not make
Assumption~\ref{assumption:continuity-scores-v}. Some inequalities will
require the assumption, which implies the distinctness
of the scores, and we will highlight those.

Recall that $\{(X_i,Y_i)\}_{i=1}^n \simiid Q_0$ and 
$\{\score(X_i,Y_i)\}_{i=1}^n \simiid P_0$, that for all $q \in \R$
\begin{align*}
  C^{(q)}(x) = \left\{ y \in \mc{Y} \mid \score(x,y) \le q \right\},
\end{align*}
and that we use $P_0(\cdot \mid X \in R)$ as shorthand for the law
of $\score(X, Y)$ for $(X, Y) \sim Q_0(\cdot \mid X \in R)$.
We also use $\empQ$ and $\empP$ for the empirical distributions f $Q$ and $P$,
respectively.
Observe that for all $q \in \R$ and $0<\delta < 1$, 
$\wcoverage(C^{(q)}, \mc{R}_v, \delta; \, Q_0) \ge 1 - \alpha$ 
if and only if
\begin{align*}
  \sup_{R \in \mc{R}_v: Q_0(R) \ge \delta}
  \Quantile(1-\alpha;P_0(\cdot \mid X \in R)) \le q.
\end{align*}

By a VC-covering argument (cf.~\cite[Sec.~A.4]{CauchoisGuDu21} or
\cite[Thm.~5]{BarberCaRaTi19a}), there exists a universal constant
$C_\varepsilon < \infty$ such that the following holds.  For $t>0$, define
$\epsilon_n(t) \defeq C_\varepsilon \sqrt{\dfrac{\text{VC}(\mc{R})\log(n) +
    t}{n}}$. Then with probability at least $1-\half e^{-t}$ over $\{X_i,Y_i
\}_{i =1}^n$, the following equations hold simultaneously for all $v \in
\mc{V}$:
\begin{align}
  \label{eqn:uniform-robust-estimation-cdf}
  \sup_{s \in \R} \left|
  \inf_{\substack{R \in \mc{R}_v \\ \empQ (R) \ge \delta }}
  \empP\left( \score(X,Y) \le s \mid X \in R \right) -  
  \inf_{\substack{R \in \mc{R}_v \\ \empQ(R) \ge \delta }}
  P_0\left( \score(X,Y) \le s \mid X \in R \right) \right|
  \le \frac{\varepsilon_n(t)}{\sqrt{\delta}}
\end{align} 
and
\begin{align}
  \label{eqn:uniform-probability-events}
  \sup_{R \in \mc{R}_v} \left|  \empQ( X \in R) - Q_0(X \in R) \right| \le  \varepsilon_n(t).
\end{align}
We assume for the remainder of the proof that
inequalities~\eqref{eqn:uniform-robust-estimation-cdf}
and~\eqref{eqn:uniform-probability-events} hold.

Define the 
empirical quantile
\begin{align*}
  \what{q}_n(v, \delta) \defeq \sup_{R \in \mc{R}_v}
  \left\{\Quantile(1-\alpha; \empP( \cdot \mid X \in R))
  ~ \mbox{s.t.}~
  \empQ( X \in R) \ge \delta \right\}.
\end{align*}
We first give a lemma on its coverage.
\begin{lemma}
  \label{lemma:empirical-q-once-coverage}
  Let the bounds~\eqref{eqn:uniform-robust-estimation-cdf}
  and~\eqref{eqn:uniform-probability-events} hold. Then
  \begin{equation}
    \label{eqn:quantile-comparison-over-v}
    \begin{split}
      \wcoverage(C^{(\what{q}_n(v, \delta))}, \mc{R}_v, \delta_{n}^{+}(t); \, Q_0)
      & ~ \ge ~ 1 - \alpha_{n}^{+}(t)  \\
      \wcoverage(C^{(\what{q}_n(v, \delta))}, \mc{R}_v, \delta_{n}^{-}(t); \, Q_0)
      & \stackrel{(\textup{\ref{assumption:continuity-scores-v}})}{\le}
      1 - \alpha_{n}^{-}(t)
    \end{split}
  \end{equation}
  simultaneously for all $v \in \mc{V}$,
  where the second inequality requires
  Assumption~\ref{assumption:continuity-scores-v}.
\end{lemma}
\begin{proof}
  Applying the bounds~\eqref{eqn:uniform-robust-estimation-cdf},
  we can bound the worst-case quantiles via
  \begin{equation}
    \begin{split}
      \lefteqn{
        \sup_{R \in \mc{R}_v}
        \left\{ \Quantile(1-\alpha^+_n(t);
        P_0( \cdot \mid X \in R)) ~ \mbox{s.t.} ~
        \empQ(X \in R) \ge \delta \right\}
      }
      \\
      & \quad \le
      \what{q}_n(v, \delta)
      \le 
      \sup_{R \in \mc{R}_v}
      \left\{
      \Quantile(1-\alpha^-_n(t); P_0( \cdot \mid X \in R))
      ~ \mbox{s.t.}~
      \empQ( X \in R) \ge \delta \right\}.
    \end{split}
    \label{eqn:bound-empirical-quantile}
  \end{equation}
  The inclusions
  \begin{align*}
    \{ R \in \mc{R} \mid Q_0(X \in R) \ge \delta + \varepsilon_n(t) \}
    & \subset
    \{ R \in \mc{R} \mid \empQ(X \in R) \ge \delta \} \\
    & \subset \{ R \in \mc{R}
    \mid Q_0(X\in R) \ge \delta - \varepsilon_n(t) \}
  \end{align*}
  are an immediate consequence of
  inequality~\eqref{eqn:uniform-probability-events}, and, in turn, imply that
  for all $\alpha \in (0,1)$,
  \begin{align*}
    \sup_{\substack{R \in \mc{R}_v \\ Q_0( X \in R) \ge \delta^+_n(t)}} \Quantile(1-\alpha; P_0( \cdot \mid X \in R))
    & \le
    \sup_{\substack{R \in \mc{R}_v \\ \empQ( X \in R) \ge \delta}} \Quantile(1-\alpha; P_0( \cdot \mid X \in R)) \\
    &\le
    \sup_{\substack{R \in \mc{R}_v \\ Q_0( X \in R) \ge \delta^-_n(t)}} \Quantile(1-\alpha; P_0( \cdot \mid X \in R)).
  \end{align*}
  Combining these inclusions with the
  inequalities~\eqref{eqn:bound-empirical-quantile}, we thus obtain
  \begin{align}
   \label{eqn:quantile-comparison}
    \lefteqn{q_n^{\inf}(v)
      \defeq
      \sup_{R \in \mc{R}_v}
      \left\{\Quantile(1-\alpha^+_n(t); P_0( \cdot \mid X \in R))
      ~ \mbox{s.t.} ~ Q_0( X \in R) \ge \delta^+_n(t) \right\} } \nonumber \\
    & \le
    \what{q}_n(v, \delta) \\
    & \le 
    \sup_{R \in \mc{R}_v}
    \left\{
    \Quantile(1-\alpha^-_n(t); P_0( \cdot \mid X \in R))
    ~ \mbox{s.t.}~
    Q_0( X \in R) \ge \delta^-_n(t)\right\}
    \eqdef q_n^{\sup}(v) \nonumber.
  \end{align}
  The infimum and supremum quantiles satisfy
  \begin{equation*}
    \begin{split}
      \wcoverage(C^{(q_n^\text{inf}(v))}, \mc{R}_v,
      \delta_{n}^{+}(t); \, Q_0)
      & ~ \ge ~ 1 - \alpha_{n}^{+}(t) \\
      \wcoverage(C^{(q_n^\text{sup}(v))}, \mc{R}_v, \delta_n^{-}(t); \, Q_0)
      & \stackrel{(\textup{\ref{assumption:continuity-scores-v}})}{=}
      1 - \alpha_{n}^{-}(t),
    \end{split}
  \end{equation*}
  where the inequality always holds and the equality requires
  Assumption~\ref{assumption:continuity-scores-v}.

  We now observe that for any fixed $(v, \delta) \in \mc{V} \times (0,1)$,
  the function $q \mapsto \wcoverage(C^{(q)}, \mc{R}_v, \delta; \, Q_0)$ is
  non-decreasing, since the confidence sets $C^{(q)}(x)$ increase as $q$
  increases.  Recalling inequalities~\eqref{eqn:quantile-comparison}, we
  conclude that
\begin{align*}
\wcoverage(C^{(\what{q}_n(v, \delta)}, \mc{R}_v,
      \delta_{n}^{+}(t); \, Q_0) \ge 
      \wcoverage(C^{(q_n^\text{inf}(v))}, \mc{R}_v,
      \delta_{n}^{+}(t); \, Q_0) 
      \ge 1 - \alpha_n^+(t) 
      \end{align*}
and
\begin{align*}
  \wcoverage(C^{( \what{q}_n(v, \delta))}, \mc{R}_v,
      \delta_{n}^{-}(t); \, Q_0) \le  
      \wcoverage(C^{(q_n^\text{sup}(v))}, \mc{R}_v,
      \delta_{n}^{-}(t); \, Q_0) 
      = 1 - \alpha_n^-(t),
\end{align*}
simultaneously for all $v \in \mc{V}$, with the second inequality requiring Assumption~\ref{assumption:continuity-scores-v}.
\end{proof}


Recall that $\what{q}_\delta$ in
Algorithm~\ref{alg:rho-selection-procedure} is the $(1-\levelv)$-empirical
quantile of $\{\what{q}_n(v_i, \delta)\}_{i = 1}^k$.  Then
inequalities~\eqref{eqn:quantile-comparison-over-v} in
Lemmma~\ref{lemma:empirical-q-once-coverage} and that $\wcoverage(C^{(q)},
\mc{R}_v, \delta; \, Q_0)$ is non-decreasing in $q$ imply
\begin{align*}
  \Pvemp\left[ 
    \wcoverage(C^{( \what q_\delta)}, \mc{R}_v, \delta^{+}_n(t); \, Q_0) 
    \ge 1 - \alpha^{+}_{n}(t) 
    \right] 
  \ge 
  \Pvemp \left[ \what q_\delta \ge \what{q}_n(v_i, \delta) \right] 
  \ge 1-\levelv,
\end{align*}
while under Assumption~\ref{assumption:continuity-scores-v},
we have the converse lower bound
\begin{align*}
  \Pvemp\left[ 
    \wcoverage(C^{( \what q_\delta)}, \mc{R}_v, \delta^{-}_n(t); \, Q_0) 
    \le 1 - \alpha^{-}_{n}(t) \right]
  \ge
  \Pvemp\left[\what{q}_\delta \le \what{q}_n(v, \delta)\right]
  \ge \levelv - \frac{1}{k},
\end{align*}
using the second inequality of Lemma~\ref{lemma:empirical-q-once-coverage}.

For $q \in \R$, define the functions $f^+_q(v) \defeq
\indic{\wcoverage(C^{(q)}, \mc{R}_v, \delta_n^{+}(t); \, Q_0) \geq 1-
  \alpha_n^{+}(t)} \in \{ 0, 1\}$ for all $q \in \R$.  The set of functions
$\{ f_q^+ \}_{q \in \R}$ is uniformly bounded (by $1$) and each is
non-decreasing in $q \in \R$ so that its VC-dimension cannot exceed 1.
Thus, there exists a universal constant $C < \infty$ such that, with
probability $1- 4^{-1}e^{-t}$~\cite[e.g.][Thm.~4.10,
  Ex.~5.24]{Wainwright19},
\begin{align*}
  \sup_{q \in \R} \left| \Pvemp f^{+}_q - \Pv f^{+}_q \right| 
  \le C\sqrt{\frac{1+t}{k}}.
\end{align*}
Similarly, if we define $f^-_q(v) \defeq \indic{\wcoverage(C^{(q)},
  \mc{R}_v, \delta_n^{-}(t); \, Q_0) \le 1- \alpha_n^{-}(t)} \in \{ 0, 1\}$, then with probability at least $1-4^{-1}e^{-t}$,
we have $\sup_{q \in \R} | \Pvemp f^{-}_q - \Pv f^{-}_q| \le
C\sqrt{\frac{1+t}{k}}$.  Combining the statements, we see that with
probability $1-e^{-t}$ over the
draw $(X_i,Y_i)_{i=1}^n \simiid Q_0$ and $\{v_i \}_{i=1}^k \simiid \Pv$,
we have
\begin{align*}
  \Pv f^{+}_{\what q_\delta} = \Pv \left[
    \wcoverage(C^{(\what q_\delta)}, \mc{R}_v, \delta^{+}_n(t); \, Q_0)
    \ge 1- \alpha^{+}_{n}(t) \right]
  \ge 1- \levelv - C\sqrt{\frac{1+t}{k}},
\end{align*}
and under Assumption~\ref{assumption:continuity-scores-v},
\begin{align*}
  \Pv f^{-}_{\what q_\delta}  = \Pv \left[
    \wcoverage(C^{(\what q_\delta)}, \mc{R}_v, \delta^{-}_n(t); \, Q_0)
    \le 1- \alpha^{-}_{n}(t)
       \right]
  \stackrel{(\textup{\ref{assumption:continuity-scores-v}})}{\ge}
  \levelv - \frac{1}{k} - C\sqrt{\frac{1+t}{k}}.
\end{align*}

\subsection{Proof of Lemma~\ref{lemma:equivalence-rho-threshold}}
\label{sec:proof-equivalence-rho-threshold}

Let $\scorerv_i = \score(X_i, Y_i)$ for shorthand, and assume
w.l.o.g.\ that $\scorerv_1 \le \cdots \le \scorerv_n$.
We will show that if $\what{q} \in \{\scorerv_i\}_{i \ge
  \ceil{n(1 - \alpha)}}$, then if
\begin{equation}
  \label{eqn:equivalence-rho-threshold}
  \what{\rho} = \rho_{f,\alpha}(\what{q}; \empP)
  ~~ \mbox{then} ~~
  \what{q} = \WCQuantile_{f,\what{\rho}}(1 - \alpha; \empP).
\end{equation}
Evidently this implies that $C^{(\what{q})}(x) = C_{f,\what{\rho}}(x;
\empP)$ for all $x \in \mc{X}$, giving the lemma, so for the remainder, we
show the equivalence~\eqref{eqn:equivalence-rho-threshold}.

Recall the definition $g_{f,\rho}^{-1}(\tau) = \sup\{\beta \in [\tau, 1]
\mid \beta f(\frac{\tau}{\beta}) + (1 - \beta) f(\frac{1 - \tau}{1 -
  \beta}) \le \rho\}$ in the discussion following
Proposition~\ref{proposition:rho-vs-alpha}.  Suppose that $\what{q} =
\scorerv_j$, where $j \in [n]$, which immediately implies that, for all
$(j - 1) / n < \beta \le j / n$, $\what{q} = \Quantile(\beta; \empP)$. By
Proposition~\ref{proposition:rho-vs-alpha}, we therefore see that if $\rho
\ge 0$ satisfies $(j - 1) / n <
g_{f,\rho}^{-1}(1-\alpha) \le j / n$,
then
\begin{align*}
  \WCQuantile_{f, \rho}(1-\alpha; \hat P_n) = \what q.
\end{align*}
In addition, as the scores $\scorerv_i$ are all distinct, $\WCQuantile_{f,
  \rho}(1-\alpha; \hat P_n) > \what{q}$ if $g_{f,\rho}^{-1}(1-\alpha) >
j / n$, making $\rho_{f,\alpha}$ in this case equal to
\begin{align*}
  \rho_{f,\alpha}(\what{q} ; \empP)
  = \sup\{ \rho \ge 0 \mid g_{f,\rho}^{-1}(1-\alpha) \le j/n \}.
\end{align*}

The mapping $\rho \mapsto g_{f,\rho}^{-1}(\tau)$ is concave and
nonnegative.  As $f$ is $1$-coercive by assumption, we also have that it
is defined on $\R_+$, and it is continuous strictly increasing on
$\R_{++}$.  Its inverse (as a function of $\rho$) is therefore continuous,
which implies in particular that $g_{f,\rho_{f,\alpha}(\what{q}
  ;\empP)}^{-1}(1-\alpha) = j/n$, and hence
equality~\eqref{eqn:equivalence-rho-threshold} holds.


\section{Proofs related to finding worst shift directions}
\label{sec:worst-shift-proofs}

\subsection{Proofs on worst direction recovery}

\subsubsection{Proof of Lemma~\ref{lemma:stochastic-domination-direction}}
\label{sec:proof-stochastic-domination-direction}

Fix $q \in \R$, $\delta \in (0,1)$, and consider $v \in \mc{V},  t \in \R$ such that $\P( v(X) \ge t) \ge \delta$, i.e $F_v^-(t) \le 1-\delta$.  We then have
\begin{align*}
\P( \score(X,Y) > q \mid v(X) \ge t )  &= \dfrac{ \P( \score(X,Y) > q,   F_v^-(v(X)) \ge F_v^-(t))}{1 - F_v^-(t)} \\
&\overset{(i)}{\le} \dfrac{ \P( \score(X,Y) > q,   F_{v\opt}(v\opt(X)) \ge F_v^-(t))}{1 - F_v^-(t)} \\
&\overset{(ii)}{=}  \P( \score(X,Y) > q \mid  F_{v\opt}(v\opt(X)) \ge F_v^-(t))) \\
&= \P( \score(X,Y) > q \mid  v\opt(X)) \ge {F_{v\opt}}^{-1}(F_v^-(t)))
\end{align*}
where $(i)$ and $(ii)$ comes from Assumption~\ref{assumption:stochastic-dominance},  and from the continuity of the distribution of $v\opt(X)$,  which guarantees that $  F_{v\opt}^-(v\opt(X)) = F_{v\opt}(v\opt(X)) \sim \uniform[0,1]$.

Since $\P(v\opt(X)) \ge {F_{v\opt}}^{-1}(F_v^-(t))) =  \P( F_{v\opt}(v\opt(X)) \ge F_v^-(t)) = 1 - F_v^-(t) \ge \delta$, this implies that
\begin{align*}
  \wc(C^{(q)}, \mc{R}_{v\opt}, \delta; \, Q_0) \le
  \P( \score(X,Y) \le q \mid v(X) \ge t ).
\end{align*}
The result follows by taking the infimum over all $(v, t) \in \mc{V}
\times \R$ such that $\P( v(X) \ge t) \ge \delta$.

\subsubsection{Proof of Lemma~\ref{lemma:heterogeneous-regression-verifies}}
\label{sec:proof-heterogeneous-regression-verifies}

Let $(t,u) \in \R^2$,  and assume for simplicity that $\score(x,y) = \left| y - \mu^\star(x) \right|$ (the squared error case is similar).
We then have for all $v \in \mc{V} = \R^d \setminus \{0\}$, 
\begin{align*}
\P\left[ \score(X,Y) \ge t, F_v^-(v^T X) \ge u \right]
&= \E \left[  \P \left( X^T \vvar \ge h^{-1}(t/|\varepsilon|), F_v^-(X^T v) \ge u \mid \varepsilon \right) \right]  \\
&\overset{(i)}{\le} \E \left[ \min\left(  \P \left( X^T \vvar \ge h^{-1}(t/|\varepsilon|) \mid \varepsilon \right),  \P( F_v^-(X^T v) \ge u) \right) \right] \\
&\overset{(ii)}{=}  \E \left[ \min\left(  \P \left( X^T \vvar \ge h^{-1}(t/|\varepsilon|) \mid \varepsilon \right),  \P( F_{\vvar}^-(X^T \vvar) \ge u) \right) \right] \\
&= \E \left[ 
	\P\left( X^T \vvar \ge \max(  h^{-1}(t/|\varepsilon|),  F_{\vvar}^{-1}(u)) \mid \varepsilon \right) 
\right] \\
&= \E \left[ 
	\P\left( |\varepsilon| h(X^T \vvar) \ge t, F_{\vvar}^-(X^T \vvar) \ge u \mid \varepsilon \right) 
\right] \\
&= \P\left( \score(X,Y) \ge t, F_{\vvar}^-(X^T \vvar) \ge u \right).
\end{align*} 
Inequality $(i)$ is simply a restatement of the elementary fact $\P( A \cap B) \le \min(\P(A), \P(B))$, while equality $(ii)$ is due to the fact that, since every linear combination $X^T v$ has a continuous distribution for all $v \neq 0$,  $F_v^-(X^T v)$ has an uniform distribution on $[0,1]$.

\subsubsection{Proof of Lemma~\ref{lemma:worst-direction-order-consistency}}
\label{sec:proof-worst-direction-order-consistency}
\newcommand{\orthantdominates}{\succeq_{\mathsf{uo}}}

Assumption~\ref{assumption:stochastic-dominance} ensures the following upper orthant stochastic order:
\begin{align*}
(\score(X,Y),  F_{v\opt}(v\opt(X)) \stocuo (\score(X,Y),  F_v^-(v(X))) \text{ for all } v \in \mc{V}.
\end{align*}
Letting $F_S(t) \defeq \P( \scorerv \le t)$, we first observe by conditioning on $(X_1, Y_1)$ that
\begin{align*}
&\P\left[ \score(X_1,  Y_1) > \score(X_2, Y_2), \, v(X_1) > v(X_3) \right] \\
&\quad =
\E\left[ \P(\score(X_1,  Y_1) > \score(X_2, Y_2) \mid X_1, Y_1) \P(v(X_1) > v(X_3) \mid X_1) \right] \\
&\quad = \E\left[ F_S^-(\score(X_1, Y_1)) F_v^-(v(X_1)) \right].
\end{align*}
We then have the following lemma on upper orthant ordering.
\begin{lemma}
  \label{lemma:orders}
  Let $U, V \in \R^2$. Then $U \orthantdominates V$ if and only if for all
  non-negative and non-decreasing functions $f, g$,
  \begin{equation}
    \label{eqn:domination-by-fg}
    \E[f(V_1) g(V_2)]
    \le \E[f(U_1) g(U_2)].
  \end{equation}
  If additionally  $U_1 \eqdist V_1$ and $\E[|f(V_1) g(V_2)|]$
  and $\E[|f(U_1) g(U_2)|] < \infty$,
  then $\E[f(V_1) g(V_2)] \le \E[f(U_1) g(U_2)]$ for all
  non-negative and non-decreasing $f$ and non-decreasing $g$.
\end{lemma}
\begin{proof}
  The equivalence of inequality~\eqref{eqn:domination-by-fg} and
  $U \orthantdominates V$ is~\cite[Eq.~(6.B.4)]{ShakedSh07}.
  For the second result, consider the sequence
  $g_m(x) \defeq \hinge{g(x) + m} - m$
  for $m = 1, 2, \ldots$. Then $g_m \downarrow g$, while
  \begin{equation*}
    \E[f(U_1) g_m(U_2)]
    \ge \E[f(V_1) \hinge{g(V_2) + m}] - m \E[g(V_1)]
    = \E[f(V_1) g_m(V_2)].
  \end{equation*}
  Dominated convergence gives the result.
\end{proof}

Applying Lemma~\ref{lemma:orders} with the non-decreasing functions $f=F_S^-$ and $g=\text{id}$, we obtain
\begin{align*}
\E\left[ F_S^-(\score(X_1, Y_1)) F_v^-(v(X_1)) \right] \le \E\left[ F_S^-(\score(X_1, Y_1)) F_{v\opt}^-(v\opt(X_1)) \right],
\end{align*}
which is equivalent to
\begin{align*}
\P\left[ \score(X_1,  Y_1) > \score(X_2, Y_2), \, v(X_1) > v(X_3) \right] \le \P\left[ \score(X_1,  Y_1) > \score(X_2, Y_2), \, v\opt(X_1) > v\opt(X_3) \right]
\end{align*}
The same argument with $f=F_S$ also proves that:
\begin{align*}
 \P\left[ \scorerv_1 \ge \scorerv_2,  v(X_1) > v(X_3) \right] \le  \P\left[ \score(X_1,  Y_1) \ge \score(X_2, Y_2), \,  v\opt(X_1) >  v\opt(X_3) \right],
\end{align*}
which allows us to conclude that
\begin{align*}
v\opt \in \argmax_{v \in \mc{V}} \left\{ \P\left( \scorerv_1 > \scorerv_2,  v(X_1) > v(X_3) \right) + \P\left( \scorerv_1 \ge \scorerv_2,  v(X_1) > v(X_3) \right) \right\}.
\end{align*}

\subsubsection{Proof of Lemma~\ref{lemma:penalized-linear-loss-expression}}
\label{sec:proof-penalized-linear-loss-expression}

By definition of $\eta_S$, we have
\begin{align*}
\E \left[ \sign(S_1 - S_2) \mid X_1 \right] &= \P( S_1 > S_2 \mid X_1) -  \P(S_1 < S_2 \mid X_1) \\
&= 2 \P(S_1 > S_2 \mid X_1) - 1 + \P(S_1 = S_2 \mid X_1) \\
&= 2 \eta_S(X_1) - 1,
\end{align*}
which shows that for all $v \in \mc{V}$, 
\begin{align*}
\E \left[ (v(X_1) - \sign(\scorerv_1 - \scorerv_2))^2 \right] = \E\left[ \left( v(X_1) - (2\eta_S(X_1) -1) \right)^2 \right] + \E\left[ \var(\sign(\scorerv_1 - \scorerv_2) \mid X_1) \right],
\end{align*}
and proves our first result.

Additionally,  for any measurable function $f : \mc{X} \to \R$,  let $F_f$ be the c.d.f.\ of $f(X)$  which satisfies $F_f(X) \succeq U \sim \uniform[0,1] \succeq F_f^{-}(X)$, where the latter is the left-continuous version.
By conditioning respectively, and in order,
on $(X_1,Y_1)$ and $X_3$, and then only on $(X_1, Y_1)$,  we see that
\begin{align*}
\P\left( \scorerv_1 > \scorerv_2 , f(X_1) > f(X_3) \right) &+ \half \P\left( \scorerv_1 = \scorerv_2 , f(X_1) > f(X_3) \right) \\
&= \E\left[ \eta_S(X_1) \indic{ f(X_1) > f(X_3) } \right] \\
&= \E\left[ \eta_S(X_1) F_f^{-}(f(X_1)) \right]  \\
&= \int_{(0,1)^2} \P\left[ \eta_S(X_1) \ge u,  F_f^{-}(f(X_1)) \ge v \right] du dv \\
&\le \int_{(0,1)^2}  \min \left\{ \P(\eta_S(X_1) \ge u),  \P(F_f^{-}(f(X_1)) \ge v) \right\} du dv \\
&\overset{(i)}{\le} \int_{(0,1)^2}  \min \left\{ \P(\eta_S(X_1) \ge u),  \P(F_{\eta_S}^{-}(\eta_S(X_1)) \ge v) \right\} du dv \\
&\overset{(ii)}{=} \E\left[ \eta_S(X_1) F_{\eta_S}^{-}(\eta_S(X_1)) \right] \\
&= \P\left( \scorerv_1 > \scorerv_2 , \eta_S(X_1) > \eta_S(X_3) \right) +  \half \P\left( \scorerv_1 = \scorerv_2 , \eta_S(X_1) > \eta_S(X_3) \right).
\end{align*}
Equality $(i)$ comes from the fact that for any measurable function $f$, the
function $v \mapsto \P(F_f^{-}(f(X_1)) \ge v)$ is less than $1-v =
\P(F_{\eta_S}^{-}(\eta_S(X_1)) \ge v)$, as by assumption, $\eta_S(X)$ has a
continuous distribution, which entails that $F_{\eta_S}^{-}(\eta_S(X))
\overset{d}{=} \uniform[0,1]$.  Equality $(ii)$ uses that $F_{\eta_S} =
F_{\eta_S}^-$ is non-decreasing, so
\begin{align*}
\min \left\{ \P(\eta_S(X_1) \ge u),  \P(F_{\eta_S}^{-}(\eta_S(X_1)) \ge v) \right\} =
\P\left[ \eta_S(X_1) \ge u,  F_{\eta_S}^{-}(\eta_S(X_1)) \ge v \right]
\end{align*}
for all $(u,v) \in (0,1)^2$.
 
\subsubsection{Proof of Proposition~\ref{prop:consistency-worst-direction-rkhs}}
\label{sec:proof-consistency-worst-direction-rkhs}
For each $i \in [n]$,  define
\begin{align*}
\tilde{Y}_i =  \half \E\left[ \sign\left( \score(X_i,Y_i) - \score(X', Y') \right) \mid X_i,Y_i \right] \in \left[-\half, \half \right],
\end{align*}
and define the ``theoretical'' estimator that $\what v_{\text{pen},
  \lambda_n}$ approximates,
\begin{align*}
  \tilde{v}_{\text{pen}, \lambda_n} \defeq \argmin_{v \in \mc{V}} \left\{ \frac{1}{n}\sum_{i=1}^n \left( v(X_i) - \tilde{Y_i} \right)^2 + \lambda_n \norm{v}_\mc{V}^2 \right\}.
\end{align*}
A direct application of Theorem 9.1 in~\citet{SteinwartCh08} to the dense separable RKHS $\mc{V}$ with bounded measurable kernel $k$ shows that
\begin{align}
\label{eqn:rkhs-erm-convergence}
\int_{x \in \mc{X}} \left( \tilde{v}_{\text{pen}, \lambda_n}(x) - \E\left[ \tilde{Y} \mid X=x \right] \right)^2 dP_X(x) = o_p(1),
\end{align}
where additionally $\E[ \tilde{Y} \mid X=x] = \eta_S(x) - \half$.

It remains to compare the finite sample estimators $\hat{v}_{\text{pen},
  \lambda_n}$ and $\tilde{v}_{\text{pen}, \lambda_n}$.  The key is to notice
that, if we let $\bar{Y}_i^n \defeq \frac{1}{2(n-1)}\sum_{j \neq i}
\sign(\scorerv_i - \scorerv_j)$, then
\begin{align*}
\hat{v}_{\text{pen}, \lambda_n}  \defeq \argmin_{v \in \mc{V}} \left\{ \frac{1}{n}\sum_{i=1}^n \left( v(X_i) - \bar{Y}_i^n \right)^2 + \lambda_n \norm{v}_\mc{V}^2 \right\}, 
\end{align*}
and we expect $\{ \bar{Y}_i^n \}_{i=1}^n$ and $\{ \tilde{Y}_i \}_{i=1}^n$ to
be uniformly close.  Indeed, we have $\tilde{Y}_i = f(\scorerv_i)$, where
$f(s) \defeq \half \E[ \sign(s - \scorerv)]$, and $\bar{Y}_i ^n =
f_n(\scorerv_i)$, with $f_n(s) \defeq \frac{1}{2(n-1)} \sum_{j=1}^n \sign(s
- \scorerv_j)$.

Let $E_n \defeq \max_{1 \le i \le n} \left| \tilde{Y}_i - \tilde{Y}_i^n
\right| \le 1$.

As the class of sign thresholds $\{ x \mapsto \sign(s-x)\}_{s \in \R}$ is
uniformly bounded by $1$ and has VC-dimension at most $2$, Donsker's
theorem implies that
\begin{align}
  \label{eqn:cdf-uniform-conv}
  n^{1/2} \sup_{s \in \R} \left| f_n(s) - f(s) \right| =O_p(1),~ \text{thus} ~
  E_n  =O_p(n^{-1/2}).
\end{align}

To conclude the proof of the first result, define 
\begin{align*}
  R_n(v) \defeq
  \left\{ \frac{1}{n}\sum_{i=1}^n \left( v(X_i) - \bar{Y}_i^n \right)^2 + \lambda_n \norm{v}_\mc{V}^2 \right\}
\end{align*}
and $\tilde{R}_n$ similarly with each $\tilde{Y}_i$ in lieu of $\bar{Y}_i^n$.
The convergence~\eqref{eqn:cdf-uniform-conv} directly implies that uniformly over $v \in \mc{V}$, we have
\begin{align}
\label{eqn:risk-uniform-approx}
\left| R_n(v) - \tilde{R}_n(v) \right| \le 2 E_n \left\{1+ \frac{1}{n} \sum_{i=1}^n \left| v(X_i) - \tilde{Y}_i \right| \right\} = O(E_n) \left( \norm{v}_\mc{V} + 1 \right),
\end{align}
as the kernel $k$ is bounded, so there exists $C_k \defeq \sup_{x \in
  \mc{X}} k(x,x)^{1/2}$ such that $|v(x)| = |\langle k(x,\cdot), v
\rangle_{\mc{V}}| \le C_k \norm{v}_\mc{V}$ for all $x \in \mc{X}$.  The
inequality~\eqref{eqn:risk-uniform-approx}, along with the fact that
$\lambda_n \norm{\hat v_{\text{pen}, \lambda_n}}_\mc{V}^2 \le R_n(0) \le 1$
(and similarly for $\tilde{v}_{\text{pen}, \lambda_n}$), leads us to
\begin{align*}
&R_n(\tilde{v}_{\text{pen}, \lambda_n}) - R_n(\hat v_{\text{pen}, \lambda_n})\\ &= 
\left( R_n(\tilde{v}_{\text{pen}, \lambda_n}) - \tilde{R}_n(\tilde{v}_{\text{pen}, \lambda_n}) \right) +  \left( \tilde{R}_n(\tilde{v}_{\text{pen}, \lambda_n}) - \tilde{R}_n(\hat v_{\text{pen}, \lambda_n}) \right) + 
\left(\tilde{R}_n(\hat v_{\text{pen}, \lambda_n}) - R_n(\hat v_{\text{pen}, \lambda_n}) \right) \\
 &\le 2 \sup_{v \in \mc{V}: \norm{v}_\mc{V} \le \lambda_n^{-1/2}} \left| R_n(v) - \tilde{R}_n(v) \right| = O(\lambda_n^{-1/2} E_n).
\end{align*}

By the strong convexity of $R_n$
(via the regularization term $\lambda \norm{v}_{\mc{V}}^2$),
and as $\hat v_{\text{pen}, \lambda_n}$ is its minimizer, we must have 
\begin{align*}
R_n(v) - R_n(\hat v_{\text{pen}, \lambda_n}) \ge \lambda_n \norm{v - \hat v_{\text{pen}, \lambda_n}}_\mc{V}^2 \text{ for all } v\in \mc{V},
\end{align*}
which, combining the last two inequalities and
substituting $v = \tilde{v}_{\text{pen},\lambda_n}$, implies that
\begin{align*}
\norm{\tilde{v}_{\text{pen}, \lambda_n} - \hat v_{\text{pen}, \lambda_n}}_\mc{V}^2 \le O\left(\lambda_n^{-3/2} E_n \right).
\end{align*}
As the kernel $k$ is bounded, we then have
\begin{align*}
\norm{\tilde{v}_{\text{pen}, \lambda_n} - \hat v_{\text{pen}, \lambda_n}}_{L^2(P_X)} &\le \left(\int_{x} k(x,x)dP_X(x)\right)^{1/2} \norm{\tilde{v}_{\text{pen}, \lambda_n} - \hat v_{\text{pen}, \lambda_n}}_{\mc{V}} \\
&= O(\lambda_n^{-3/4} E_n^{1/2}) = O_p(n^{- 1/16}),
\end{align*}
Recalling equation~\eqref{eqn:rkhs-erm-convergence}, this yields the desired result.

For the second claim,  observe that our first result also entails that
\begin{align}
\label{eqn:v-dense-in-l2}
\inf_{v \in \mc{V}} \int_{x \in \mc{X}} \left( v(x) + \half - \eta_S(x) \right)^2 dP_X(x) = 0.
\end{align}
We claim that a consequence of this fact is that
\begin{align}
\label{eqn:vopt-and-etas-are-same}
\E\left[ \eta_S(X) F_{\eta_S}(\eta_S(X)) \right] = \E \left[ \eta_S(X) F_{v\opt}(v\opt(X)) \right].
\end{align}
Before proving claim~\eqref{eqn:vopt-and-etas-are-same}, we see how this implies that $v\opt$ must be a function of $\eta_S$. 
Observe that we can rewrite the latter equality as
\begin{align*}
\int \P\left( \eta_S(X) \ge u_1,  F_{\eta_S}(\eta_S(X)) \ge u_2 \right) du_1 du_2 = \int \P\left( \eta_S(X) \ge u_1,  F_{v\opt}(v\opt(X)) \ge u_2 \right) du_1 du_2.
\end{align*}
On the other hand, it is straightforward to check that, as the distribution
of $\eta_S(X)$ is continuous by assumption, we have $(\eta_S(X),
F_{\eta_S}(\eta_S(X))) \stocuo (\eta_S(X), F_{v\opt}(v\opt(X))$, and so both
integrands must be identical up to a measure 0 set, as the left is
always larger than the right while the integrals are equal.  By
left-continuity of both functions, the equality must extend to the entire
square $[0,1]^2$, so for all $(u_1,u_2) \in [0,1]^2$ we have
\begin{align*}
 \P\left( \eta_S(X) \ge u_1,  F_{\eta_S}(\eta_S(X)) \ge u_2 \right)  =  \P\left( \eta_S(X) \ge u_1,  F_{v\opt}(v\opt(X)) \ge u_2 \right).
\end{align*}
Taking $u_2 = F_{\eta_S}(u_1)$,  this directly gives
\begin{align*}
 \P\left( \eta_S(X) \ge u \right) =  \P\left( \eta_S(X) \ge u,  F_{v\opt}(v\opt(X)) \ge F_{\eta_S}(u) \right),
\end{align*}
for all $u \in [0,1]$, which in turn implies
\begin{align*}
 \P\left[ F_{\eta_{S}}(\eta_S(X)) < F_{\eta_{S}}(u),  F_{v\opt}(v\opt(X)) \ge F_{\eta_S}(u) \right] = \P\left[ \eta_S(X) < u,  F_{v\opt}(v\opt(X)) \ge F_{\eta_S}(u) \right] = 0.
\end{align*}
This equality holds for any $u \in [0,1]$, so we must have
$F_{v\opt}(v\opt(X)) = F_{\eta_{S}}(\eta_S(X))$ almost surely, which
concludes the proof of the second part of
Proposition~\ref{prop:consistency-worst-direction-rkhs}.

Coming back to the claim~\eqref{eqn:vopt-and-etas-are-same},  we first observe that Lemma~\ref{lemma:penalized-linear-loss-expression} ensures that
\begin{align*}
\E\left[ \eta_S(X) F_{\eta_S}(\eta_S(X)) \right] = \inf_{f: \mc{X} \to \R ~ \text{measurable}} \left[ \eta_S(X) F_f^-(f(X)) \right],
\end{align*}
which immediately yields the inequality $\E\left[ \eta_S(X) F_{v\opt}(v\opt(X))
  \right] \ge \E\left[ \eta_S(X) F_{\eta_S}(\eta_S(X)) \right]$, because
$\mc{V} \subset \{f: \mc{X} \to \R \text{ measurable} \}$ and $F_{v\opt} \ge
F_{v\opt}^-$.


For the reverse inequality,  consider a sequence $v_n \in \mc{V}$ such that $\norm{v_n + \half - \eta_S}_{L^2(P_X)} \to 0$, which is possible from the infimum~\eqref{eqn:v-dense-in-l2}. 
Since $\eta_S(X)$ has a continuous distribution, we must have $F_{v_n + \half} \to F_{\eta_S}$ pointwise, and hence uniformly as they are non-decreasing functions.
This, plus the fact that $v_n(X)+\half \cp \eta_S(X)$,  implies by continuous mapping that
\begin{align*}
F_{v_n}(v_n(X))  = F_{v_n+\half}\left(v_n(X)+\half\right) \cp F_{\eta_S}(\eta_S(X)),
\end{align*} and, since the sequence $\{ \eta_S(X) F_{v_n}(v_n(X)) \}_{n \ge 1}$ is uniformly bounded (by 1) that
\begin{align*}
\E\left[ \eta_S(X) F_{v_n}(v_n(X)) \right] \underset{n \to \infty}{\rightarrow} \E\left[ \eta_S(X) F_{\eta_S}(\eta_S(X)) \right],
\end{align*}
which eventually yields that
\begin{align*}
\inf_{v \in \mc{V}} \left[ \eta_S(X) F_v^-(v(X)) \right] \le \E\left[ \eta_S(X) F_{\eta_S}(\eta_S(X)) \right]
\end{align*}
and concludes the proof, as Lemma~\ref{lemma:worst-direction-order-consistency} ensures that
\begin{align*}
\E\left[ \eta_S(X) F_{v\opt}(v\opt(X)) \right] = \inf_{v \in \mc{V}} \left[ \eta_S(X) F_v^-(v(X)) \right].
\end{align*}

\subsubsection{Proof of Proposition~\ref{proposition:vopt-least-squares}}
\label{sec:proof-vopt-least-squares}

We now show that
\begin{equation}
  \label{eqn:upper-dominance-vopt}
  (\score(X, Y), X^T v\opt) \orthantdominates (\score(X, Y), X^T u)
\end{equation}
for any vector $u$ satisfying $\ltwos{\Sigma^{1/2} v\opt} =
\ltwos{\Sigma^{1/2} u}$.  Without loss of generality, we assume
$\ltwos{\Sigma^{1/2} v\opt} = 1$.  Then for all $q \in \R$ and $t \in \R$,
\begin{align*}
  \P( \score(X,Y) \ge q, X^T v\opt \ge t)
  & = \P( \score(X,Y) \ge q \mid X^T v\opt \ge t) \P(X^T v\opt \ge t) \\
  &\stackrel{(\star)}{\ge}
  \P( \score(X,Y) \ge q \mid  X^T u \ge t) \P(X^T u \ge t) \\
  & = \P( \score(X,Y) \ge q, X^T u \ge t ),
\end{align*}
where inequality $(\star)$ uses
Assumption~\ref{assumption:stochastic-dominance} and that $\tilde{X} \defeq
\Sigma^{-1/2} X$ has an isotropic disitribution, so that $\P(X^T u \ge t) =
\P( \tilde{X}^T \Sigma^{1/2} u \ge t) = \P( \tilde{X}^T \Sigma^{1/2} v\opt
\ge t) = \P( X^T v\opt \ge t)$ and
$X^T u \eqdist X^T v\opt$.
In particular, Lemma~\ref{lemma:orders} yields
\begin{equation*}
  \E \left[ \score(X,Y) X^T u \right] \le
  \E \left[ \score(X,Y) X^T v\opt \right]
\end{equation*}
for all $u \in \R^d$ such that $\ltwos{\Sigma^{1/2}u} =
\ltwos{\Sigma^{1/2}v\opt}$, because
$\E[|\score(X, Y) X^T u|] < \infty$ by Cauchy-Schwarz.
As a result, using the assumption in the proposition that
$\E[\score(X, Y) X] \neq 0$, we have the fixed point
\begin{align*}
  v\opt = \argmin_{u \in \R^d} \left\{ \E \left[ \score(X,Y) X^T u \right]
  \mid \ltwos{\Sigma^{1/2} u} = \ltwos{\Sigma^{1/2} v\opt} \right\}.
\end{align*}
By a direct change of variables via $\tilde{X} = \Sigma^{-1/2} X$, this is
equivalent to
\begin{align*}
  \Sigma^{1/2} v\opt = \argmin_{\tilde{u} \in \R^d} \left\{ \tilde{u}^T \E \left[ \score(X,Y) \tilde{X} \right] \mid \ltwo{\tilde{u}} = \ltwo{\Sigma^{1/2} v\opt} \right\}.
\end{align*}
Rewriting, we obtain
\begin{align*}
  v\opt \propto \Sigma^{-1} \E \left[X \score(X,Y)\right]
  = \E\left[XX^T\right]^{-1} \E \left[ X \score(X,Y) \right]
  = \argmin_u \E[(\score(X, Y) - X^T u)^2].
\end{align*}

\subsection{Proof of Theorem~\ref{theorem:uniform-asymptotic-coverage}}
\label{sec:proof-uniform-asymptotic-coverage}

The proof of the theorem is technical, so we
state and prove several lemmas on worst coverage regularity and convergence
(Section~\ref{sec:lemmes-worst-coverage-and-direction-recovery}), combining
all the pieces in Section~\ref{sec:finalize-uniform-proof}.

\subsubsection{Lemmas on worst coverage estimation}
\label{sec:lemmes-worst-coverage-and-direction-recovery}



\begin{lemma}
  \label{lem:continuous-distribution-vs-worst-coverage-continuity}
  Let Assumption~\ref{assumption:continuity-worst-coverage-1} hold.  Then
  the function $(q, v, \delta) \mapsto \wc( C^{(q,\score)}, \mc{R}_v,
  \delta; \, Q_0) $ is continuous at any tuple $(q,v\opt, \delta)$,  considering $\mc{V}$ as a subset of the Banach space $L^2(P_X)$.
\end{lemma}

\begin{proof}
  We use $C^{(q)}$ as shorthand for $C^{(q,\score)}$, and we consider a
  sequence $\{ (q_n , v_n, \delta_n) \}_{n \ge 1} \to (q, v\opt, \delta) \in \R
  \times \mc{V} \times (0,1)$.  We will show that $\{ \wc(C^{(q_n)},
  \mc{R}_{v_n}, \delta_n; \, Q_0) \}_{n \ge 1}$ converges by proving that
  the sequence has a unique accumulation point.  We therefore assume without
  loss of generality that
  \begin{align}
    \label{eqn:limit-of-wcs}
    \wc(C^{(q_n)}, \mc{R}_{v_n}, \delta_n; \, Q_0)
    \underset{n \to \infty}{\longrightarrow} \ell \in [0,1],
  \end{align}
  and we successively prove that $\ell \le \wc(C^{(q)}, \mc{R}_{v}, \delta;
  \, Q_0)$ and $ \wc(C^{(q)}, \mc{R}_{v}, \delta; \, Q_0) \le \ell$.
  Combining claims~\ref{claim:lower-limit-wcs}
  and~\ref{claim:upper-limit-wcs} immediately gives the continuity claim in
  Lemma~\ref{lem:continuous-distribution-vs-worst-coverage-continuity}.

  \begin{claim}
    \label{claim:lower-limit-wcs}
    The limit $\ell$ in Eq.~\eqref{eqn:limit-of-wcs} satisfies
    $\ell \le \wc(C^{(q)}, \mc{R}_v, \delta; Q_0)$.
  \end{claim}
  \begin{proof}
    Let $\varepsilon > 0$, and consider $t \in \R$ such that $\P( v\opt(X) \ge
    t) \in (\delta,1)$ and
    \begin{align*}
      Q_0 \left(  \score(X,Y) \le q \mid v\opt(X) \ge t \right)
      \le \wc(C^{(q)}, \mc{R}_{v\opt}, \delta; \, Q_0) + \varepsilon.
    \end{align*}
    Next, consider $t_n \in \R$ such that $Q_0( v_n(X) \ge t_n) \ge \delta_n$ and $Q_0( v_n(X) \ge t_n) \to Q_0( v\opt(X) \ge t) \} $. 
    As we may consider a
    subsequence, we assume without loss of generality that $\{ t_n \}_{n \ge
      1}$ converges to $\tilde{t} \in \left[ - \infty, \infty \right]$.  Then
    we have by Slutsky's lemma that $v_n(X) - t_n \cd v\opt(X) -
    \tilde{t}$ (since $v_n(X) \cd v(X)$), and thus
    \begin{align*}
      Q_0( v\opt (X) \ge \tilde{t}) = \lim_{n \to \infty} Q_0(v_n(X) \ge t_n) = Q_0( v\opt(X) \ge t) \ge \delta,
    \end{align*}
    as $v\opt(X)$ has a continuous distribution (the above relation also proves
    that $\tilde{t} \in \R$, since $0 < Q_0( v\opt(X) \ge t) < 1 )$.  Since we
    either have $\{ v\opt(X) \ge \tilde{t} \} \subset \{ v\opt(X) \ge t \}$ or $\{
   v\opt(X) \ge t \} \subset \{ v\opt(X) \ge \tilde{t} \}$, the above relation
    also shows that
    \begin{align*}
      Q_0(  \score(X,Y) \le q \mid v\opt(X) \ge \tilde{t}) = Q_0(  \score(X,Y) \le q \mid v\opt(X) \ge t) \le  \wc(C^{(q)}, \mc{R}_{v}, \delta; \, Q_0) + \varepsilon.
    \end{align*}
    Finally, if we define $\Delta_{n,v} \defeq v_n(X) - v\opt(X) - t_n
    +\tilde{t} \cp 0$, we have
    \begin{align}
      \label{eqn:cvg-prob-xtvscores}
      \begin{split}
        &\big| Q_0( \score(X,Y) \le q_n,  v_n(X)  \ge t_n) - Q_0( \score(X,Y) \le q,  v(X) \ge \tilde{t}) \big| \\
        &\le Q_0\left( |\score(X,Y) - q| \le |q_n - q| \right) + Q_0(\tilde{t} - \Delta_{n,v} \le v\opt(X)  < \tilde{t}) + Q_0( \tilde{t}  \le v\opt(X)  < \tilde{t} - \Delta_{n,v}) \\
        &\underset{n \to \infty}{\longrightarrow} 0,
      \end{split}
    \end{align}
    where the first (resp.\ second and third) term converges to $0$ as the
    distribution of $\score(X,Y)$ (resp. $v\opt(X)$) is continuous under $Q_0$.
    This proves that
    \begin{align*}
      Q_0( \score(X,Y) \le q_n \mid  v_n(X) \ge t_n) \underset{n \to \infty}{\longrightarrow} Q_0( \score(X,Y) \le q \mid v(X) \ge \tilde{t}) \le \wc(C^{(q)}, \mc{R}_{v}, \delta; \, Q_0) + \varepsilon,
    \end{align*}
    and thus $\ell \le \wc(C^{(q)}, \mc{R}_{v}, \delta; \, Q_0) + \varepsilon$.
    As $\varepsilon > 0$ was arbitrary, we have the claim.
  \end{proof}

  \begin{claim}
    \label{claim:upper-limit-wcs}
    The limit $\ell$ in Eq.~\eqref{eqn:limit-of-wcs} satisfies
    $\ell \le \wc(C^{(q)}, \mc{R}_v, \delta; Q_0)$.
  \end{claim}
  \begin{proof}
    By definition of the worst-coverage, we can find $\{ t_n \}_{n \ge 1}$
    such that $Q_0 (v_n(X) \ge t_n) \ge \delta_n$ for all $n \ge 1$, and
    \begin{align*}
      Q_0( \score(X,Y) \le q_n \mid v_n(X) \ge t_n) \underset{n \to \infty}{\longrightarrow} \ell.
    \end{align*}
    As we may always consider a subsequence, we again assume that $t_n
    \to t \in [ -\infty, \infty]$.  Next, observe that, by Slutsky's lemma, $v_n(X) - t_n \cd v\opt(X) - t $
    (where the limit distribution is continuous but potentially infinite if
    $t \in \{ -\infty, \infty\}$), so
    \begin{align}
      \label{eqn:cvg-prob-xtv-lower-semicontinuity}
      Q_0( v\opt(X) \ge t) =  \lim_n Q_0( v_n(X) \ge t_n) \ge  \delta
    \end{align}
    by the Portmanteau theorem,
    which also proves that $t < \infty$.

    If $t = - \infty$, then $Q_0( v_n(X) \ge t_n) \to 1$.  As the
    distribution of $\score(X,Y)$ is continuous under $Q_0$, this ensures
    that
    \begin{align*}
      Q_0( \score(X,Y) \le q_n \mid v_n(X) \ge t_n )\to Q_0( \score(X,Y) \le q) \ge \wc(C^{(q)}, \mc{R}_{v\opt}, \delta; \, Q_0),
    \end{align*}
    and proves that $\ell \ge \wc(C^{(q)}, \mc{R}_{v\opt}, \delta; \, Q_0)$.  If
    $t \in \R$, then with derivation \emph{mutatis mutandis} identical to
    that to develop the convergence~\eqref{eqn:cvg-prob-xtvscores}, we
    obtain that
    \begin{align*}
      Q_0( \score(X,Y) \le q_n,  v_n(X) \ge t_n)  - Q_0( \score(X,Y) \le q,  v\opt(X) \ge t) \underset{n \to \infty}{\rightarrow} 0.
    \end{align*}
    With equation~\eqref{eqn:cvg-prob-xtv-lower-semicontinuity}, this directly
    shows that
    \begin{align*}
      \wc(C^{(q)}, \mc{R}_{v\opt}, \delta; \, Q_0)
      &\le Q_0( \score(X,Y) \le q \mid v\opt(X) \ge t)  \\
      &= \lim_{n \to \infty} Q_0( \score(X,Y) \le q \mid v_n(X) \ge t_n) = \ell
    \end{align*}
    as desired.
  \end{proof}

\end{proof}

\begin{lemma}
  \label{lem:consistency-of-empirical-worst-coverage-for-alpha}
  As $n \to \infty \; \; (n_1,n_2 \to \infty)$, the confidence set mapping
  $\what C_n$ from Alg.~\ref{alg:worst-direction-validation} satisfies
  \begin{equation*}
    1- \alpha \le \wc( \what C_n, \mc{R}_{\hat v}, \delta; \, \what Q_{n_2})
    \le 1 - \alpha + u_n,
  \end{equation*}
  where $u_n \in [0,\alpha]$, and $u_n \cp 0$ if
  Assumption~\ref{assumption:estimated-scores-as-distinct} holds.
\end{lemma}
\begin{proof}
  The lower bound is immediate by definition of $\what C_n$.
  For the upper bound, we have
  \begin{equation*}
    \wc( \what C_n, \mc{R}_{\hat v}, \delta; \, \what Q_{n_2}) \le 1- \alpha + \frac{1}{n_2 \delta}
  \end{equation*}
  whenever the scores $\{ \scoreest(X_i,Y_i) \}_{i=n_1+1}^n$ are all distinct, which occurs eventually with high probability under Assumption~\ref{assumption:estimated-scores-as-distinct}.
\end{proof}

\begin{lemma}
  \label{lem:consistency-of-empirical-worst-coverages}
  Let
  Assumption~\ref{assumption:continuity-worst-coverage-1} hold.
  Then as $n \to \infty$,
  the worst coverages under $\what Q_{n_2}$ and $Q_0$ satisfy
  the Glivenko-Cantelli result
  \begin{align*}
    \sup_{q\in \R} \left| 
    \wc( C^{(q,\scoreest)}, \mc{R}_{\hat v}, \delta; \, \what Q_{n_2})-\wc( {C}^{(q, \scoreest)}, \mc{R}_{\hat v}, \delta; \, Q_0) \right| 
    \cas 0.
  \end{align*}
\end{lemma}
\begin{proof}
  Let $\varepsilon > 0$ be arbitrary.  Recalling
  equations~\eqref{eqn:uniform-robust-estimation-cdf}
  and~\eqref{eqn:uniform-probability-events} in the proof of
  Theorem~\ref{theorem:high-probability-coverage}, there exists a universal
  constant $c < \infty$ such that conditionally on $\scoreest$ and the first
  half of the validation set (hence $\hat v$), we have with probability at
  least $1- \varepsilon$ over $\{ (X_i, Y_i) \}_{i=n_1+1}^{n}$ that
  \begin{align*}
    \begin{split}
      \sup_{q \in \R} \Big| \inf_{\substack{R \in \mc{R}_{\hat v} \\ Q_{n_2} (X\in R) \ge \delta }} Q_{n_2}\left( \scoreest(X,Y) \le q \mid X \in R \right) &-  
      \inf_{\substack{R \in \mc{R}_{\hat v} \\ Q_{n_2}(X\in R) \ge \delta }} Q_0\left( \scoreest(X,Y) \le q \mid X \in R \right) \Big|  \\
      &\le c \sqrt{\frac{\log(n_2/\varepsilon)}{n_2\delta}}
    \end{split}
  \end{align*}
  and
  \begin{align*}
    \sup_{q \in \R, R \in \mc{R}_{\hat v}} \lvert Q_{n_2}( X \in R) - Q_0(X \in R) \rvert \le  c \sqrt{\frac{\log(n_2/\varepsilon)}{n_2}}.
  \end{align*}
  Setting $\delta_n^\pm \defeq \delta \pm c
  \sqrt{\frac{\log(n_2/\varepsilon)}{n_2}}$, these two statements ensure
  that with probability $1-\varepsilon$ over $\{ (X_i, Y_i)
  \}_{i=n_1+1}^{n}$, simultaneously for all $q \in \R$,
  \begin{align*}
    \wc( C^{(q,\scoreest)}, \mc{R}_{\hat v}, \delta_n^-; \, Q_0)
    -  c \sqrt{\frac{\log(n_2 /\varepsilon)}{n_2 \delta}} &\le 
    \wc( C^{(q,\scoreest)}, \mc{R}_{\hat v}, \delta; \, \what Q_{n_2}) \\
    &\le \wc( C^{(q,\scoreest)}, \mc{R}_{\hat v}, \delta_n^+; \, Q_0) +
    c \sqrt{\frac{\log(n_2/\varepsilon)}{n_2\delta}}.
  \end{align*}

  To conclude, we claim that for all $q \in \R$, $v \in \mc{V}$ for which $v(X)$ has a continuous distribution, scores
  $\score$, and $0<
  \delta_0 < \delta_1 < 1$, we have
  \begin{align}
    \label{eqn:worst-coverage-lipschitz-in-delta}
    \wc( C^{(q,\score)}, \mc{R}_{v}, \delta_1; \, Q_0) - \frac{\delta_1 -
      \delta_0}{\delta_1}\le \wc( C^{(q,\score)}, \mc{R}_{v}, \delta_0; \,
    Q_0).
  \end{align}
  Temporarily deferring the proof of
  inequality~\eqref{eqn:worst-coverage-lipschitz-in-delta},
  this shows in particular that for all $q \in \R$, so long as $\hat v(X)$ has a continuous distribution (which occurs with probability going to $1$ from Assumption~\ref{assumption:continuity-worst-coverage-1}),
  \begin{align*}
    \wc( C^{(q,\scoreest)}, \mc{R}_{\hat v}, \delta_n^-; \, Q_0) \ge \wc( C^{(q,\scoreest)}, \mc{R}_{\hat v}, \delta; \, Q_0) - \frac{ \delta-\delta_n^-}{\delta},
  \end{align*}
  and that
  \begin{align*}
    \wc( C^{(q,\scoreest)}, \mc{R}_{\hat v}, \delta_n^+; \, Q_0) \le \wc( C^{(q,\scoreest)}, \mc{R}_{\hat v}, \delta; \, Q_0) + \frac{\delta_n^+ - \delta}{\delta}.
  \end{align*}
  We thus have, conditionally on $\scoreest$ and $\hat v$,
  which are
  independent of the sample $\{ (X_i, Y_i) \}_{i=n_1+1}^{n}$,
  that with probability at least $1 - \varepsilon$
  \begin{align*}
    \sup_{q \in \R} \lvert \wc( C^{(q,\scoreest)}, \mc{R}_{\hat v}, \delta; \, \what Q_{n_2}) 
    - \wc( C^{(q,\scoreest)}, \mc{R}_{\hat v}, \delta; \, Q_0 )   \rvert
    & \le
    c \sqrt{\frac{\log(n_2/\varepsilon)}{n_2}}\left(\delta^{-1/2} + \delta^{-1} \right).
  \end{align*}
  The Borel-Cantelli lemma then gives the almost sure convergence.
  
  We return to demonstrate the
  claim~\eqref{eqn:worst-coverage-lipschitz-in-delta}.
  We have by definition that
  \begin{equation*}
    \wc(C^{(q,\score)}, \mc{R}_v, \delta_0; Q_0)
    = \min\left\{
    \begin{array}{l}
      \inf_{R \in \mc{R}_v} \{Q_0(\score(X, Y) \le q \mid X \in R) :
      \delta_1 \le Q_0(X \in R)\}, \\
      \inf_{R \in \mc{R}_v} \{Q_0(\score(X, Y) \le q \mid X \in R) :
      \delta_0 \le Q_0(X \in R) < \delta_1 \}
    \end{array}
    \right\}.
  \end{equation*}
  If the topmost term achieves the minimum, the
  claim~\eqref{eqn:worst-coverage-lipschitz-in-delta} is immediate,
  so we may instead assume that the bottom term achieves it.  
 The fact that $v(X)$ is continuous ensures the
  existence of $a_1 \in \R$ such that $Q_0( v(X)  \ge a_1) = \delta_1$
  satisfying
  \begin{equation*}
    \wc( C^{(q,\score)}, \mc{R}_{v}, \delta_1; \, Q_0)
    \le Q_0( \score(X,Y) \le q \mid v(X) \ge a_1)
  \end{equation*}
  as $\wc$ is an infimum over all such shifts.
  Then for any $a_0 \ge a_1$ such that
  $Q_0( v(X) \ge a_0) \ge \delta_0$, we in turn have
  \begin{align*}
    Q_0( \score(X,Y) \le q \mid v(X) \ge a_1)
    &= \delta_1^{-1} Q_0( \score(X,Y) \le q,  v(X) \ge a_1) \\
    &\le \delta_1^{-1} \left( Q_0(\score(X,Y) \le q,  v(X) \ge a_0) +  Q_0(a_1 \le  v(X) < a_0) \right) \\
    & \le Q_0( \score(X,Y) \le q \mid v(X)  \ge a_0) + \frac{\delta_1 - \delta_0}{\delta_1},
  \end{align*}
  where we have used that $Q_0( v(X) \ge a_0) \le \delta_1$.
  Taking an infimum over all such $a_0$
  gives
  the statement~\eqref{eqn:worst-coverage-lipschitz-in-delta} above.
\end{proof}

\begin{lemma}
  \label{lem:uniform-convergence-over-v-for-scores}
  Let Assumptions~\ref{assumption:consistency-of-scores-and-vhat}
  and~\ref{assumption:continuity-worst-coverage-1} hold. Then the score
  functions $\scoreest$ and $\score$ offer uniformly close worst coverage in
  the sense that
  \begin{equation*}
    \sup_{q,v}
    \left\{ \left| \wc( C^{(q, \scoreest)}, \mc{R}_{ v}, \delta; \,  Q_{0})
    -\wc( {C}^{(q,  \score)}, \mc{R}_{v}, \delta; \, Q_0) \right|
    \mid q \in \R, v \in \mc{V} \right\}
    = o_P(1). 
  \end{equation*}
\end{lemma}
\begin{proof}
  We need to show
  \begin{align*}
    \sup_{q,v} \left| \inf_{a : Q_{0}(X \in R_{ v,a}) \geq \delta}P_{0}( \scoreest \leq q \mid X \in R_{ v,a})- \inf_{a : Q_{0}(X \in R_{ v,a}) \geq \delta}P_{0}({\scorerv} \leq  q \mid X \in R_{ v,a})
    \right| = o_P(1),
  \end{align*}
  for which it is sufficient to prove that
  \begin{align*}
    \label{eqn:intersect-prob-close-v-2}
    \sup_a \left\{ \lvert Q_{0}(\scoreest(X,Y) \leq q, X \in R_{{v},a})
    -Q_0( \score(X,Y) \leq q  ,X \in R_{{v},a})\rvert   
    \mid Q_0\left( v(X) \ge a\right) \ge \delta  \right\}
    \cp 0.
  \end{align*}

  Fix $\varepsilon > 0$.  Under
  Assumption~\ref{assumption:continuity-worst-coverage-1}, the distribution
  of $\scorerv$ is continuous, so that $q \mapsto P_0( \scorerv \le q)$ is
  continuous, monotone, and has finite limits in $\pm \infty$, so that it is
  uniformly continuous. Thus, there exists $\eta = \eta(\varepsilon) >0$
  such that
  \begin{align*}
    \sup_{q \in \R} P_{0}(q < \scorerv\le q + \eta ) \le \varepsilon.
  \end{align*}
  Now, define
  \begin{equation*}
    B_{n} \defeq \{(x,y) \in \mc{X} \times \mc{Y} \mid |\scoreest(x,y)-\score(x,y)| \ge \eta \},
  \end{equation*}
  and observe that for all $q \in \R$, $v \in \mc{V}$ and $a \in \R$, we have
  \begin{align*}
    Q_0( \scoreest(X,Y) \le q,  v(X) \ge a) &\le Q_0(B_{n}) + Q_0( \score(X,Y) \le q+\eta,  v(X) \ge a) \\
    &\le Q_0(B_{n}) + Q_0( \score(X,Y) \le q+\eta,  v(X) \ge a) \\
    &\le Q_0(B_{n}) +  Q_0( \score(X,Y) \le q,  v(X) \ge a) + \varepsilon,
  \end{align*}
  and similarly
  \begin{align*}
    Q_0( \score(X,Y) \le q,  v(X) \ge a) \le Q_0(B_{n}) +  Q_0( \scoreest(X,Y) \le q,  v(X) \ge a) + \varepsilon.
  \end{align*}
  These imply that
  \begin{align*}
    \sup_a \Big\{ \left| Q_0( \scoreest(X,Y) \le q, v(X) \ge a) - Q_0( \score(X,Y) \le q, v(X) \ge a) \right|
    & \mid 
    Q_0\left( v(X) \ge a\right) \ge \delta
    \Big\} \\
    &\le \varepsilon + Q_0(B_{n}),
  \end{align*}
  and we conclude using Markov's inequality and
  Assumption~\ref{assumption:consistency-of-scores-and-vhat} that
  \begin{align*}
    Q_0(B_n) \le \frac{\ltwopx{\scoreest-\score}^2}{\eta^2} \cp 0,
  \end{align*}
  which gives the result.
\end{proof}

\begin{lemma}
  \label{lem:consistency-of-empirical-worst-direction}
  Let Assumptions~\ref{assumption:consistency-of-scores-and-vhat}
  and~\ref{assumption:continuity-worst-coverage-1} hold.  Then as $n_1 \to
  \infty$,
  \begin{align*}
    \sup_q \lvert \wc( C^{(q,\score)}, \mc{R}_{\hat v}, \delta; \, Q_0) 
    -
    \wc( {C}^{(q, \score)}, \mc{R}_{v\opt}, \delta; \, Q_0) \rvert = o_p(1).
  \end{align*}
\end{lemma}
\begin{proof}
  
  Let $\varepsilon > 0$.
For each $v \in \mc{V}$,  the function $q \in \R  \mapsto \wc( C^{(q,\score)}, \mc{R}_{v\opt}, \delta; \, Q_0)$ is bounded non-decreasing, hence there exists a $\{ q_i \}_{i=1}^N \subset \R$ a non-decreasing sequence so that 
\begin{align*}
\sup_{0\le i \le N} \left|  \wc( C^{(q_{i+1},\score)}, \mc{R}_{v\opt}, \delta; \, Q_0) -  \wc( C^{(q_i,\score)}, \mc{R}_{v\opt}, \delta; \, Q_0) \right| \le \varepsilon,
\end{align*}
with the convention that $q_0 = -\infty$ and $q_{N+1} = \infty$.

For each fixed $q \in \R$, $v \in \mc{V} \mapsto \wc( C^{(q,\score)}, \mc{R}_{v}, \delta; \, Q_0) $ is continuous,  which implies by continuous mapping (since $\norm{\hat{v} - v}_{L^2(P_X)} \cp 0$) that
\begin{align*}
\sup_{0\le i \le N+1} \left|  \wc( C^{(q_{i},\score)}, \mc{R}_{\hat{v}}, \delta; \, Q_0) -  \wc( C^{(q_i,\score)}, \mc{R}_{v\opt}, \delta; \, Q_0) \right| = o_P(1).
\end{align*}

Finally,  we can use the fact that $q \in \R  \mapsto \wc( C^{(q,\score)}, \mc{R}_{\hat{v}}, \delta; \, Q_0)$ is also non-decreasing to conclude that
\begin{align*}
&\sup_{q \in \R} \left| \wc( C^{(q,\score)}, \mc{R}_{\hat{v}}, \delta; \, Q_0)
- \wc( C^{(q,\score)}, \mc{R}_{v\opt}, \delta; \, Q_0) \right|
\le
\\
 & \sup_{0\le i \le N+1} \left|  \wc( C^{(q_{i},\score)}, \mc{R}_{\hat{v}}, \delta; \, Q_0) -  \wc( C^{(q_i,\score)}, \mc{R}_{v\opt}, \delta; \, Q_0) \right| + \varepsilon,
\end{align*} 
which eventually yields the desired result as $\varepsilon$ is arbitrary.
\end{proof}

\subsubsection{Finalizing the proof of
  Theorem~\ref{theorem:uniform-asymptotic-coverage}}
\label{sec:finalize-uniform-proof}

Lemma~\ref{lem:uniform-convergence-over-v-for-scores} shows that $\what{C}_n
= C^{(\hat q_{\delta}, \scoreest)}$ satisfies
\begin{align*}
  \sup_ {v \in \mc{V}} |\wc( \what{C}_n, \mc{R}_{v}, \delta; \, Q_0)-\wc( {C}^{(\hat q_{\delta},  \score)}, \mc{R}_{v}, \delta; \, Q_0)| = o_p(1),
\end{align*}
which implies
\begin{equation}
  \label{eqn:peanut-butter}
  |\wc( \what{C}_n, \mc{R}, \delta; \, Q_0) -\wc( {C}^{( \hat q_{\delta},  \score)}, \mc{R}, \delta; \, Q_0)| = o_p(1).
\end{equation}
Combining
Lemmas~\ref{lem:consistency-of-empirical-worst-coverages},
\ref{lem:uniform-convergence-over-v-for-scores}
and~\ref{lem:consistency-of-empirical-worst-direction}, we
additionally see that
\begin{align*}
\wc( {C}^{(\hat q_{\delta},  \score)}, \mc{R}_{v\opt}, \delta; \, Q_0)
&\overset{\ref{lem:consistency-of-empirical-worst-direction}}{=}
 \wc( {C}^{(\hat q_{\delta},  \score)}, \mc{R}_{\hat v}, \delta; \, Q_0) +o_{P}(1) \\
&\overset{\ref{lem:uniform-convergence-over-v-for-scores} }{=}
 \wc( {C}^{(\hat q_{\delta},  \scoreest)}, \mc{R}_{\hat v}, \delta; \, Q_0) +o_{P}(1) \\
&\overset{\ref{lem:consistency-of-empirical-worst-coverages}}{=} 
\wc( {C}^{(\hat q_{\delta},  \scoreest)}, \mc{R}_{\hat v}, \delta; \, \what{Q}_{n_2}) +o_{P}(1).
\end{align*}
As $\wc( {C}^{(\hat q_{\delta}, \scoreest)}, \mc{R}_{\hat v}, \delta; \,
\what{Q}_{n_2}) = 1- \alpha + u_n$ for some $u_n \ge 0$ by
Lemma~\ref{lem:consistency-of-empirical-worst-coverage-for-alpha}, where
$u_n \cp 0$ under Assumption~\ref{assumption:estimated-scores-as-distinct},
we have
\begin{align}
  \label{eqn:more-peanut-butter}
  \wc( {C}^{(\hat q_{\delta}, \score)}, \mc{R}_{v\opt}, \delta; \, Q_0)
  = 1 - \alpha + u_n + o_P(1).
\end{align}
With Lemma~\ref{lemma:stochastic-domination-direction},
Assumption~\ref{assumption:stochastic-dominance} ensures that
$\wc( {C}^{(\hat q_{\delta},  \score)}, \mc{R}_{v\opt}, \delta; \, Q_0)= \wc( {C}^{(\hat q_{\delta},  \score)}, \mc{R}, \delta; \, Q_0)$,
so we can conclude that
\begin{eqnarray*}
  \wc( \what{C}_n , \mc{R}, \delta; \, Q_0)
  & \stackrel{\eqref{eqn:peanut-butter}}{=} &
  \wc( {C}^{( \hat q_{\delta},  \score)}, \mc{R}, \delta; \, Q_0) + o_p(1) \\
  & \stackrel{\textup{Lem.~\ref{lemma:stochastic-domination-direction}}}{=} &
  \wc( {C}^{( \hat q_{\delta},  \score)}, \mc{R}_{v\opt}, \delta; \, Q_0) + o_p(1)
  \stackrel{\eqref{eqn:more-peanut-butter}}{=} 1 - \alpha + u_n + o_p(1).
\end{eqnarray*}

\section{Proof of Theorem~\ref{thm:unif-conv-sens}}
\label{proof-thm-unif-conv-sens}

Recall that our goal is to prove that
there exists a Gaussian process $\mathbb{G}$ such that
for every compact $K \subset \R_+$, we have
\begin{equation}
  \label{goal:uniform-convergence}
  \{ \sqrt{n}(\what{\SF}^{(\predsetthresh)}_n(\rho)- \SFcov(\rho)) \}_{\rho \in K}
  \cd \{\mathbb{G}(\rho)\}_{\rho \in K}.
\end{equation}
as elements in $L^\infty(K)$. Fix a compact set $K \subset \R_+$.  We first
set notation.  For simplicity, we omit the threshold superscripts $t$ on
$\MC^{(\predsetthresh)}$, $\qfunc^{(\predsetthresh)}$ and
$h^{(\predsetthresh)}$ as the threshold $\predsetthresh$ remains fixed
throughout.  For shorthand, let $\mc{X} \defeq \R^I$, and for any functions
$m : \mc{X} \to [0,1]$ and $q : K \to [0,1]$ and scalar $\rho>0$, we define
the integrand (recall the expansion~\eqref{eqn:intuition-for-augmentation})
\begin{align*}
  \Phi_{m, q, \rho} (x,  s) \defeq e^\rho \hinge{m(x) - q(\rho)}
  + q(\rho) + e^\rho \indic{ m(x) > q(\rho)} \left[ \indic{s > \predsetthresh}
    - m(x) \right].
\end{align*}
For any $P_{0,I}$-integrable function $f : \mc{X} \times \R \to \R$, we
define the empirical process shorthands
\begin{align*}
  P_n f \defeq \frac{1}{n} \sum_{i=1}^n f(X_{I,i}, S_i), ~ P f = \E_{(X,S) \sim P_{0,I}} \left[ f(X,S) \right],  ~ \text{ and } ~ \mathbb{G}_n f \defeq \sqrt{n} \left( P_n - P \right) f.
\end{align*}
Additionally, for every $\batch \in [\nBatch]$, we
define the subsampled quantities
\begin{align*}
  P_{n,\batch} f \defeq \frac{1}{n/\nBatch} \sum_{i \in \mc{I}_\batch} f(X_{I,i}, S_i) ~ \text{ and } \mathbb{G}_{n,\batch}f \defeq \sqrt{n/\nBatch} \left(P_{n,\batch} - P \right) f.
\end{align*}

By definition of $\what{\SF}^{(\predsetthresh)}_n(\rho)$ and $
\SFcov(\rho)$, we have
\begin{align*}
  \what{\SF}^{(\predsetthresh)}_n(\rho) = \frac{1}{\nBatch} \sum_{\batch=1}^\nBatch P_{n,\batch} \Phi_{\what \MC_\batch, \what \qfunc_\batch, \rho} ~ \text{ and } \SFcov(\rho) = P \Phi_{\MC, \qfunc, \rho},
\end{align*}
so if we define the remainder 
\begin{align*}
  R_{n, \rho} \defeq \frac{1}{\sqrt{\nBatch}}\sum_{\batch=1}^\nBatch \sqrt{n/\nBatch} \left( P_{n,\batch} \Phi_{\what \MC_\batch, \what \qfunc_\batch, \rho} - P_{n,\batch} \Phi_{\MC,  \qfunc, \rho} \right),
\end{align*}
then our empirical process is
\begin{align*}
  \sqrt{n} \left( \what{\SF}^{(\predsetthresh)}_n(\rho) - \SFcov(\rho) \right)
  = \mathbb{G}_n \Phi_{\MC, \qfunc, \rho} + R_{n, \rho}.
\end{align*}
By Slutsky's lemma, it thus suffices to prove that that the collection
$\mc{F}_{\MC, \qfunc}^K \defeq \{ \Phi_{\MC, \qfunc, \rho} \}_{\rho \in K}$
is Donsker, i.e., that there exists a Gaussian process $\mathbb{G}_K$ on
$L^\infty(K)$ such that
\begin{align*}
  \{\mathbb{G}_n \Phi_{\MC, \qfunc, \rho} \}_{\rho \in K}
  \cd \mathbb{G}_K
  ~~ \mbox{in~} L^\infty(K)
\end{align*}
and that the remainder is uniformly negligible, satisfying $\sup_{\rho \in
  K} |R_{n,\rho}| = o_P(1)$.
We now argue that each of these hold.

\paragraph{Donsker properties of $\mc{F}$}:
We first show that $\mc{P}_{\MC, \qfunc}^K$ is Donsker, an immediate
consequence of the following lemma and~\citet[Thm 19.14]{VanDerVaart98}.  In
the statement of the lemma, recall that for a collection of functions
$\mc{F}$, the $L^2(Q)$-covering number $N(\epsilon, \mc{F},
\norm{\cdot}_{L^2(Q)})$ is the size of the smallest $\epsilon$-cover for
$\mc{F}$ in $L^2(Q)$ norm, that is, the smallest $N$ for which there exist
$h_1, \ldots, h_N$ satisfying $\min_{i \le N} \norm{f - h_i}_{L^2(Q)} \le
\epsilon$ for all $f \in \mc{F}$.
\begin{lemma}
  \label{lem:uniform-covering-number}
  Let $m : \mc{X} \to [0,1]$ be measurable and $q : K \to [0,1]$
  non-decreasing. Define
  \begin{align*}
    \mc{F}_{m, q}^K \defeq \left\{ \Phi_{m,  q, \rho} \mid \rho \in K \right\}.
  \end{align*}
  Then there exists a constant $c_K \lesssim 1 + \textup{diam}(K)$ such
  that, for $0<\varepsilon \le 1$, we have
  \begin{align*}
    \sup_Q \log N\left(\varepsilon \sup_{\rho \in K} e^{\rho},  \mc{F}_{m,q}^K, \norm{\cdot}_{L^2(Q)} \right) \le \log(c_K/\varepsilon^2).
  \end{align*}
\end{lemma}
\begin{proof}
  Let $Q$ be a distribution for $X$, and set $a_K \defeq \sup_{\rho \in K}
  e^\rho$, and let $F_{Q}$ be the c.d.f. of $m(X)$ under $Q$.  For any
  $\rho_1 < \rho_2 \in K$, we have
  \begin{align*}
    \big| \Phi_{m,  q, \rho_1}(x,s) &- \Phi_{m,  q, \rho_2}(x,s)\big| \\
    &\le 2 a_K \left( \left| \rho_2 - \rho_1 \right| + \left| q(\rho_2) - q(\rho_1) \right|  + \indic{ q(\rho_1) < m(X) \le q(\rho_2)}\right),
  \end{align*}
  implying that, for some universal constant $C$, 
  \begin{align*}
    \norm{\Phi_{m,  q, \rho_1} - \Phi_{m,  q, \rho_2}}_{L^2(Q)}
    \le C a_K \left( \rho_2 - \rho_1 + q(\rho_2) - q(\rho_1) + \sqrt{F_{Q}(q(\rho_2)) - F_{Q}(q(\rho_1))} \right),
  \end{align*}
  where we used the bound $(a+b+c)^2 \le 3(a^2 + b^2 + c^2)$ for all $a,b,c
  \in \R$.


We can then construct a $3C a_K \varepsilon$-cover of $\mc{F}_{m, q}^K$ by choosing $\rho_1 \defeq \inf K \le \dots \le \rho_N$ such that for each $i \in \{1, \dots N-1\}$ we have
\begin{align*}
\rho_{i+1} = \inf \biggr\{ \rho \in K ~ \text{s.t} ~  \rho - \rho_{i} \ge \varepsilon ~\text{or} ~ q(\rho) - q(\rho_i)\ge \varepsilon  ~ \text{ or } F_{Q}(q(\rho)) - F_{Q}(q(\rho_i)) \ge \varepsilon^2 \biggr\}.
\end{align*}
By convention, if $\rho_{i+1} = \rho_i$,  meaning that $\lim_{\rho \downarrow \rho_i} q(\rho) > q(\rho_i)$,  we choose instead any $\rho_{i+1} > \rho_i$ such that
$
q(\rho_{i+1})\le \lim_{\rho \downarrow \rho_i} q(\rho) + \varepsilon 
~\text{and} ~
F_{Q}(q(\rho_{i+1})) \le \lim_{\rho \downarrow \rho_i} F_{Q}(q(\rho)) + \varepsilon^2,
$
which exists as $F_Q$ is right-continuous and $q$ is non-decreasing.

This cover contains at most $1+ \frac{2+ \textrm{diam}(K)}{ \varepsilon^2}$ such elements,  
since
\begin{align*}
2+ \textrm{diam}(K) &\ge \rho_N - \rho_1 + q(\rho_{N}) - q(\rho_1) + F_{Q}(q(\rho_{N})) - F_{Q,0}(q(\rho_1))  \\
&\ge 
\sum_{i=1}^{N-1} \big\{
\rho_{i+1} - \rho_i + q(\rho_{i+1}) - q(\rho_i) + F_{Q}(q(\rho_{i+1})) - F_{Q}(q(\rho_i)) \big\} \\
&\ge (N-1)(\varepsilon \wedge \varepsilon^2) = (N-1)\varepsilon^2,
\end{align*}
 which then implies that
\begin{align*}
N \left(3Ca_K \varepsilon,  \mc{F}_{m,\eta}^K, \norm{\cdot}_{L^2(Q)} \right) \le 1+\frac{2+ \textrm{diam}(K)}{ \varepsilon^2},
\end{align*}
and concludes the proof.
\end{proof}

It remains to bound the remainder term $R_{n,\rho}$.  To that end,  observe that, for each $b \in \nBatch$, we have
\begin{align*}
\begin{split}
\sqrt{n/\nBatch} \left( P_{n,\batch} \Phi_{\what \MC_\batch, \what \qfunc_\batch, \rho} - P_{n,\batch} \Phi_{\MC,  \qfunc, \rho} \right) 
= \mathbb{G}_{n,\batch}& \left( \Phi_{\what \MC_\batch, \what \qfunc_\batch, \rho} - \Phi_{\MC,  \qfunc, \rho} \right) \\ +& \sqrt{n/\nBatch}\left(P\Phi_{\MC,  \qfunc, \rho} - P \Phi_{\what \MC_\batch, \what \qfunc_\batch, \rho} \right),
\end{split} 
\end{align*}
which motivates Lemmas~\ref{lem:uniform-empiricalprocess-remainder} and~\ref{lem:uniform-expectation-remainder} below. 
The proof of these two lemmas is quite technical, which is why we defer them to Appendix~\ref{sec:technical-lemmas}. 
\begin{lemma}
\label{lem:uniform-empiricalprocess-remainder}
Let $\mc{F}_{n,-\batch} \defeq \sigma \left\{ (\scorerv_i , X_{I,i})_{i \in [n] \setminus \mc{I}_\batch} \right\} $. 
For each $\batch \in [\nBatch]$, we have
\begin{align*}
\E\left[ \sup_{\rho \in K} \left|  \mathbb{G}_{n,\batch}\left( \Phi_{\what \MC_\batch, \what \qfunc_\batch, \rho} - \Phi_{\MC,  \qfunc, \rho} \right) \right| \mid \mc{F}_{n,-\batch} \right] = o_p(1).
\end{align*}
\end{lemma}

\begin{lemma}
\label{lem:uniform-expectation-remainder}
For each $\batch \in [\nBatch]$, we have
\begin{align*}
\sup_{\rho \in K} \left|P(\Phi_{\what \MC_\batch, \what \qfunc_\batch, \rho} - \Phi_{\MC,  \qfunc, \rho}) \right| = o_p(n^{-1/2}).
\end{align*}
\end{lemma}
Lemma~\ref{lem:uniform-empiricalprocess-remainder} provides a bound, conditionally on $(\scorerv_i , X_{I,i})_{i \in [n] \setminus \mc{I}_\batch}$, on the supremum of the empirical process $\left\{ \mathbb{G}_{n,\batch}\left( \Phi_{\what \MC_\batch, \what \qfunc_\batch, \rho} - \Phi_{\MC,  \qfunc, \rho} \right) \right\}_{\rho \in K}$. 
Since conditional convergence in probability implies convergence in probability (see e.g.\ \citet[Lemma 6.1]{ChernozhukovChDeDuHaNeRo16}), an immediate consequence of this lemma is
\begin{align*}
 \sup_{\rho \in K} \left|  \mathbb{G}_{n,\batch}\left( \Phi_{\what \MC_\batch, \what \qfunc_\batch, \rho} - \Phi_{\MC,  \qfunc, \rho} \right) \right|  = o_p(1).
\end{align*}
Combined with Lemma~\ref{lem:uniform-expectation-remainder}, which uniformly controls the difference between the expectations under $P$, this concludes the proof of the theorem since 
\begin{align*}
\sup_{\rho \in K} |R_n(\rho)| &\le B^{-1/2}\sum_{b \in \nBatch} \left[  \sup_{\rho \in K} \left|  \mathbb{G}_{n,\batch}\left( \Phi_{\what \MC_\batch, \what \qfunc_\batch, \rho} - \Phi_{\MC,  \qfunc, \rho} \right) \right| 
+ \sqrt{n/\nBatch}\sup_{\rho \in K} \left|P(\Phi_{\what \MC_\batch, \what \qfunc_\batch, \rho} - \Phi_{\MC,  \qfunc, \rho}) \right|  \right] \\
&= o_p(1).
\end{align*}

\subsection{Proof of technical lemmas}
\label{sec:technical-lemmas}
Before proving Lemmas~\ref{lem:uniform-empiricalprocess-remainder} and~\ref{lem:uniform-expectation-remainder}, we first need to introduce two auxiliary lemmas. 
\begin{lemma}
\label{lem:quantile-infty-norm}
Let $X$ and $Y$ be two bounded random variables on the same probability space, and, for any $\alpha \in [0,1]$, let $\mc{Q}_{\alpha}(X)$ and $\mc{Q}_{\alpha}(Y)$ be their respective $1-\alpha$-quantiles.  We have, for any $\alpha \in [0,1]$,
\begin{align*}
\left| \mc{Q}_{\alpha}(X) - \mc{Q}_{\alpha}(Y)\right| \le \norm{X-Y}_\infty.
\end{align*}
\end{lemma}
\begin{proof}
This is an immediate consequence of the fact that, for any $t \in \R$ we have
\begin{align*}
\P( X \le t - \norm{X-Y}_\infty ) \le \P(Y \le t) \le \P(X \le t + \norm{X-Y}_\infty),
\end{align*}
the left inequality implying that
$
\mc{Q}_\alpha(Y) - \norm{X-Y}_\infty \ge \mc{Q}_\alpha(X) 
$
and the right one that
$
\mc{Q}_\alpha(X) \le \mc{Q}_\alpha(Y) + \norm{X-Y}_\infty.
$
\end{proof}

In particular, this simple lemma yields the following result.
For each $m : \mc{X} \to \R$, $q: K \to [0,1]$ and $\rho \in K$, let $h_{m,q, \rho}(x) \defeq \indic{ m(x) > q(\rho)}$. 
\begin{lemma}
\label{lem:bound-h-indicator}
For any $\batch \in [\nBatch]$, we have
\begin{align*}
\norm{ h_{\what \MC_\batch, \what \qfunc_\batch, \rho} - h_{\MC, \qfunc, \rho}}_{L^1(P_{0,I})} \le O(1)\norm{f_{\MC}}_\infty\left(  \left| \what \qfunc_\batch(\rho) - \qfunc_0(\what{\MC}_\batch, \rho) \right| + \norm{\what{\MC}_\batch - \MC}_{L^\infty(P_{0,I})} \right).
\end{align*}
\end{lemma}
\begin{proof} 
A direction computation shows that
\begin{align*}
\begin{split}
\norm{ h_{\what \MC_\batch, \what \qfunc_b, \rho} - h_{\MC, \eta, \rho}}_{L^1(P_{0,1})} = \P_X&\left[ \MC(X) > \qfunc(\rho), \what \MC_\batch(X) \le \what \qfunc_\batch(\rho) \right] \\
&+ \P_X\left[ \MC(X) \le \qfunc(\rho), \what \MC_\batch(X) > \what \qfunc_\batch(\rho) \right].
\end{split}
\end{align*}
We show how to bound the first term, as the second is similar.
For every $c \ge 0$, we have
\begin{align*}
\P_X &\left[ \MC(X) > \eta(\rho), \what \MC_\batch(X) \le \what \qfunc_\batch(\rho) \right] \\
&\le
 \P_{X} \left[ \qfunc(\rho) < \MC(X) \le \qfunc(\rho) + c \right]
 +  \P_{X} \left[ \MC(X) > \qfunc(\rho) + c, 
\what \MC_\batch(X) \le \what \qfunc_\batch(\rho) \right] 
\\
&\le \norm{f_{\MC}}_\infty c +  \P_{X}\left[ \MC(X) - \what \MC_\batch(X) > \qfunc(\rho) - \what \qfunc_\batch(\rho) + c \right].
\end{align*}
Consider then $c \defeq (\what \qfunc_\batch(\rho) - \qfunc(\rho))_+ +  \norm{ \what \MC_\batch - \MC}_{L^\infty(P_{0,I})}$: the second term becomes $0$,  thus
\begin{align*}
\P_{X} &\left( \MC(X) > \qfunc(\rho),  \what \MC_\batch(X) \le \what \qfunc_\batch(\rho) \right]  \\
&\le \norm{f_{\MC}}_\infty \left[  \norm{ \what \MC_\batch - \MC}_{\infty} + \left|\what \qfunc_\batch(\rho) - \qfunc(\rho) \right| \right) \\
&\le \norm{f_{\MC}}_\infty \left( 2\norm{ \what \MC_\batch - \MC}_{\infty} + \left|\what \qfunc_\batch(\rho) - \qfunc_0(\what \MC_\batch, \rho) \right| \right),
\end{align*}
where the last inequality comes from an application of Lemma~\ref{lem:quantile-infty-norm},  which ensures that for every $\rho > 0$,
\begin{align*}
\left|\what \qfunc_\batch(\rho) - \qfunc(\rho) \right| 
&\le \left|\what \qfunc_\batch(\rho) - \qfunc_0(\what \MC_\batch, \rho) \right| + \left| \qfunc_0( \what \MC_\batch, \rho ) - \qfunc(\rho) \right| \\
&=\left|\what \qfunc_\batch(\rho) - \qfunc_0(\what \MC_\batch, \rho) \right| + \left| \qfunc_0( \what \MC_\batch, \rho ) - \qfunc_0(\MC, \rho) \right| \\ &\le \left|\what \qfunc_\batch(\rho) - \qfunc_0(\what \MC_\batch, \rho) \right| + \norm{ \what \MC_\batch - \MC}_{L^\infty(P_{0,I})},
\end{align*}
as $\qfunc_0(m, \rho )$ is the $1-e^{-\rho}$ population quantile of $m(X)$ for any function $m : \mc{X} \to \R$.
\end{proof}

\subsubsection{Proof of Lemma~\ref{lem:uniform-empiricalprocess-remainder}}

We first need to bound the second moment of each $\Phi_{\what \MC_m, \what \qfunc_\batch, \rho} - \Phi_{\MC,  \qfunc, \rho}$ individually, which is what the following lemma does.
Let $\what \sigma_\batch^2(\rho) \defeq  P\left[ \big( \Phi_{\what \MC_\batch, \what \qfunc_\batch, \rho} - \Phi_{\MC,  \qfunc, \rho} \big)^2 \right]$. 
\begin{lemma}
\label{lem:uniform-second-moment}
We have
\begin{align*}
\what \sigma_{\batch,K}^2 \defeq \max\left( \sup_{\rho \in K} \what \sigma^2_\batch(\rho), n^{-1/2} \right) = o_p(n^{-1/4}).
\end{align*}
\end{lemma}
\begin{proof}
For any $(x, s) \in \mc{X} \times \R$, we have
\begin{align*}
&\left| \Phi_{\what \MC_\batch, \what \qfunc_\batch, \rho}(x,s) - \Phi_{\MC,  \eta, \rho}(x,s)  \right| \\
&\le e^\rho \left( \indic{ s > q} \left| h_{\what \MC_\batch, \what \qfunc_\batch, \rho}(x) - h_{\MC, \qfunc, \rho}(x) \right| 
+ \left|\what \qfunc_\batch(\rho)h_{\what \MC_\batch, \what \qfunc_\batch, \rho}(x) - \qfunc(\rho) h_{\MC, \qfunc, \rho}(x)\right| + \left| \what \qfunc_\batch(\rho) - \qfunc(\rho)\right| \right) \\
&\le 2e^\rho \left( \left| h_{\what \MC_\batch, \what \qfunc_\batch, \rho}(x) - h_{\MC, \qfunc, \rho}(x) \right|  + \left| \what \qfunc_\batch(\rho) - \qfunc(\rho) \right|  \right),
\end{align*}
where we used the fact that $\left| \qfunc(\rho) \right| \le 1$ in the last line.

Since $h_{\what \MC_\batch, \what \qfunc_\batch, \rho} - h_{\MC, \qfunc, \rho} \in \{-1, 0,1\}$, it is immediate that $\norm{ h_{\what \MC_\batch, \what \qfunc_\batch, \rho} - h_{\MC, \qfunc, \rho}}_{L^2(P_{0,I})}^2= \norm{ h_{\what \MC_\batch, \what \qfunc_\batch, \rho} - h_{\MC, \qfunc, \rho}}_{L^1(P_{0,I})}$ and hence that
\begin{align*}
\what \sigma_\batch^2(\rho) &\lesssim e^{2\rho} \left( \norm{ h_{\what \MC_\batch, \what \qfunc_\batch, \rho} - h_{\MC, \qfunc, \rho}}_{L^1(P_{0,I})} + \left| \what \qfunc_\batch(\rho) - \qfunc(\rho) \right|^2 
\right) \\
&\lesssim e^{2\rho} \left( \left| \what \qfunc_\batch(\rho) - \qfunc(\rho) \right|^2  + \left| \what \qfunc_\batch(\rho) - \qfunc_0(\what{\MC}_\batch, \rho) \right| + \norm{\what{\MC}_\batch - \MC}_{L^\infty(P_{0,I})} \right) \\
&\lesssim e^{2\rho} \biggr\{ \psi\left( \norm{\what{\MC}_\batch - \MC}_{L^\infty(P_{0,I})} \right) +  \psi\left( \left|\what \qfunc_\batch(\rho) - \qfunc_0(\what{\MC}_\batch, \rho) \right| \right) \biggr\},
\end{align*}
where $\psi(t) \defeq \max(t,t^2)$ for all $t \in \R$.
By Assumptions~\ref{ass:miscov-estimator-consistent} and~\ref{ass:miscov-quantile-estimator-consistent}, we can conclude that
\begin{align*}
\sup_{\rho \in K} \what \sigma_\batch^2(\rho) = o_p(n^{-1/4}).
\end{align*}
\end{proof}
The proof of Lemma~\ref{lem:uniform-empiricalprocess-remainder} then follows from an application of~\citet[Lemma 6.2]{ChernozhukovChDeDuHaNeRo16}, which we recall below,  to the family of functions
\begin{align*}
\what{\mc{F}}_\batch^K \defeq \left\{ \Phi_{\what \MC_\batch, \what \qfunc_\batch, \rho} - \Phi_{\MC,  \qfunc, \rho} \mid \rho \in K \right\},
\end{align*}
using as envelope function the constant function $(x,s) \mapsto 4\sup_{\rho\in K} e^\rho$, since $\max\left(\norm{\what \qfunc_\batch}_\infty, \norm{\qfunc}_\infty \right) \le 1$, and bounding the uniform covering number of $\what{\mc{F}}_\batch^K$ thanks to Lemma~\ref{lem:uniform-covering-number}.

\begin{lemma}[\citet{ChernozhukovChDeDuHaNeRo16}]
\label{lem:chernozhukov-lemma}
Let $\mc{F} \subset \{ \mc{X} \to \R \}$ be a collection of measurable functions with envelope function $F \ge \sup_{f \in \mc{F}} |f|$ satisfying $\norm{F}_{L^2(P)} < \infty$. Let $\sigma^2 >0$ be any positive constant such that $\sup_{f \in \mc{F}} Pf^2  \le \sigma^2 \le PF^2$, and $M \defeq \max_{1\le i \le n} F(X_{I,i})$.
If there exists constants $a \ge e$ and $v \ge 1$ such that for all $0<\varepsilon \le 1$,
\begin{align*}
\sup_Q \log N \left( \varepsilon \norm{F}_{L^2(Q)},  \mc{F},  \norm{\cdot}_{L^2(Q)} \right) \le v \log(a/\varepsilon),
\end{align*}
then we have
\begin{align*}
\E_P\left[ \sup_{f \in \mc{F}} \mathbb{G}_n f \right] \le O(1)\left( \sqrt{v \sigma^2 \log \left( \frac{a \norm{F}_{L^2(P)}}{\sigma} \right)} + \frac{v \norm{M}_{2}}{\sqrt{n}}\log \left( \frac{a \norm{F}_{L^2(P)}}{\sigma} \right)   \right) .
\end{align*}
\end{lemma}
By Lemma~\ref{lem:uniform-covering-number},  for all $0< \varepsilon \le 1$, we have 
\begin{align*}
\sup_Q \log N \left(4 \varepsilon \sup_{\rho \in K} e^{\rho},  \what{\mc{F}}_\batch^K,  \norm{\cdot}_{L^2(Q)} \right) \le 2 \log(2c_K/\varepsilon),
\end{align*}
since both pairs $(\what \MC_\batch, \what \qfunc_\batch)$ and $(\MC, \qfunc)$ satisfy its conditions of application, allowing us to construct an $\varepsilon$-cover from respective $\varepsilon/2$-covers for $\mc{F}^K_{\what \MC_\batch, \what \qfunc_\batch}$ and $\mc{F}^K_{\MC, \qfunc}$.
Applying Lemma~\ref{lem:chernozhukov-lemma} conditionally on $\mc{F}_{n,-\batch}$,  and letting $a_K \defeq \sup_{\rho\in K} e^\rho$,  we therefore have
\begin{align*}
\E\left[ \sup_{\rho \in K} \left|  \mathbb{G}_{n,\batch}\left( \Phi_{\what \MC_\batch, \what \qfunc_\batch, \rho} - \Phi_{\MC,  \qfunc, \rho} \right) \right| \mid \mc{F}_{n,-\batch} \right] \lesssim \what \sigma_{\batch,K}\sqrt{\log \frac{a_K}{\what \sigma_{\batch,K}}} + \frac{a_K}{\sqrt{n/\nBatch}}\log \frac{a_K}{\what \sigma_{\batch,K}},
\end{align*}
which is $o_p(n^{-1/8}\log(n)) = o_p(1)$ by Lemma~\ref{lem:uniform-second-moment}.

\subsubsection{Proof of Lemma~\ref{lem:uniform-expectation-remainder}}
Only in the proof of Lemma~\ref{lem:uniform-expectation-remainder} does the benefit of augmenting the estimator finally appear, as we shall see that the difference between the population averages of $\Phi_{\what \MC_\batch, \what \qfunc_\batch, \rho}$  and $\Phi_{\MC,  \qfunc, \rho}$ is actually smaller than $n^{-1/2}$ instead of the more naive $n^{-1/4}$.

For any measurable function $m \in L^\infty(Q_{0,I})$,  $\eta \in \R$ and $\rho > 0$, define
\begin{align*}
\Psi_\rho(m, \eta) \defeq e^\rho \E_{X_I \sim Q_{0,I}} \left[ \left( m(X_I) - \eta \right)_+ \right] + \eta.
\end{align*}
First observe that for all $x,y \in \R$,   the function $t \mapsto (x + t(y-x))_+$ is absolutely continuous on $[0,1]$, hence
\begin{align*}
y_+ - x_+ = (y-x) \int_0^1 \indic{ r y + (1-r) x > 0} dr.
\end{align*}
By Fubini's theorem (valid here since every variable is bounded), this implies  that
\begin{align*}
\begin{split}
 \Psi_\rho \left( \what \MC_\batch, \what \qfunc_\batch(\rho) \right) &- \Psi_\rho \left( \MC,  \qfunc(\rho) \right) 
 = e^\rho \int_{0}^1 P\biggr[ h_{r \what \MC_\batch + (1-r) \MC,  r \what \qfunc_\batch + (1-r) \qfunc, \rho} \big(\what \MC_\batch - \MC \big) \biggr] dr \\
 & + \left(\what \qfunc_\batch(\rho) - \qfunc(\rho)\right)\left\{ 1 -e^\rho \int_{0}^1  P\biggr[ h_{r \what \MC_\batch + (1-r) \MC,  r \what \qfunc_\batch + (1-r) \qfunc, \rho} \biggr] dr \right\}.
 \end{split}
\end{align*}
%
Additionally,  using the fact that $\E \left[ \indic{\scorerv > \predsetthresh} \mid X_I \right] = \MC(X_I)$,  and that $P h_{\MC, \qfunc, \rho} = e^{-rho}$ (since $\qfunc(\rho)$ is the $1-e^{\rho}$ quantile of $\MC(X)$,  whose distribution is continuous), we have
\begin{align*}
P \Phi_{\what \MC_\batch, \what \qfunc_\batch, \rho} &- P\Phi_{\MC,  \qfunc, \rho} \\
&= \Psi_\rho \left( \what \MC_\batch, \what \qfunc_\batch(\rho) \right) - \Psi_\rho \left( \MC,  \qfunc(\rho) \right) - P \left[ h_{\what \MC_\batch,  \what \qfunc_\batch, \rho}\left( \what \MC_\batch - \MC \right) \right] 
\\
\begin{split}
 &= e^\rho \int_{0}^1 P\biggr[ \left( h_{r \what \MC_\batch + (1-r) \MC,  r \what \qfunc_\batch + (1-r) \qfunc, \rho} -  h_{\what \MC_\batch,  \what \qfunc_\batch, \rho} \right) \big(\what \MC_\batch - \MC \big) \biggr] dr \\
 & \qquad - e^\rho \left(\what \qfunc_\batch(\rho) - \qfunc(\rho)\right) \int_{0}^1  P\biggr[ \left( h_{r \what \MC_\batch + (1-r) \MC,  r \what \qfunc_\batch + (1-r) \qfunc, \rho} - h_{\MC, \qfunc, \rho} \right) \biggr] dr.
 \end{split}
\end{align*}
Observing that $\left|h_{r \what \MC_\batch + (1-r) \MC,  r \what \qfunc_\batch + (1-r) \qfunc, \rho} - h_{\MC, \qfunc, \rho}\right| \le \left| h_{\what \MC_\batch, \what \eta_\batch, \rho} - h_{\MC, \qfunc, \rho}\right|$, this equality implies that
\begin{align*}
\left| P \Phi_{\what \MC_\batch, \what \qfunc_\batch, \rho} - P\Phi_{\MC,  \qfunc, \rho} \right| 
\le  e^\rho \norm{h_{ \what \MC_\batch,  \qfunc_\batch, \rho} - h_{\MC,  \qfunc, \rho}}_{L^1(P_{0,I})} \left( \left| \what \qfunc_\batch(\rho)  - \qfunc(\rho) \right| + \norm{\what \MC_\batch - \MC}_{\infty} \right).
\end{align*}
As a result, we can conclude that
\begin{align*}
\sup_{\rho \in K} \left|P \Phi_{\what \MC_\batch, \what \qfunc_\batch, \rho} - P\Phi_{\MC,  \qfunc, \rho} \right| &\le \sup_{\rho \in K} \biggr\{ e^\rho \norm{h_{ \what \MC_\batch, \what \qfunc_\batch, \rho} - h_{\MC,  \qfunc, \rho}}_{L^1(P_{0,I})} \left( \left| \what \qfunc_\batch(\rho)  - \qfunc(\rho) \right| + \norm{\what \MC_\batch - \MC}_{\infty} \right) \biggr\}  \\
&\le O(a_K \norm{f_{\MC}}_\infty) \left( \sup_{\rho \in K} \left| \what \qfunc_\batch(\rho)  - \qfunc(\what \MC_\batch, \rho) \right| + \norm{\what \MC_\batch - \MC}_{\infty} \right)^2,
\end{align*}
which is $o_p(n^{-1/2})$ by Assumptions~\ref{ass:miscov-estimator-consistent} and~\ref{ass:miscov-quantile-estimator-consistent}.

\begingroup
\bibliographystyle{apalike}
\bibliography{bib}
\endgroup

\end{document}